\documentclass[11pt]{article}

\usepackage{fullpage}
\usepackage{natbib}
\usepackage{algorithm}
\usepackage[noend]{algpseudocode}
\usepackage{amsmath,amsthm,amsfonts,amssymb}
\usepackage{hyperref}
\usepackage{color}
\usepackage{mathrsfs}
\usepackage{mathtools}
\usepackage{enumitem}
\usepackage{bm}
\usepackage{multirow}
\usepackage{booktabs}
\usepackage{makecell}
\usepackage{graphicx}
\usepackage{subfigure}
\usepackage{caption}

\newcommand{\defeq}{\mathrel{\mathop:}=}

\newcommand{\gd}{GD}
\newcommand{\pgd}{PGD}
\newcommand{\pagd}{PAGD}
\newcommand{\nag}{AGD}
\newcommand{\fstar}{f^*}

\newcommand{\nce}{NCE}

\newcommand{\vect}[1]{\ensuremath{\mathbf{#1}}}
\newcommand{\mat}[1]{\ensuremath{\mathbf{#1}}}

\newcommand{\grad}{\nabla}
\newcommand{\hess}{\nabla^2}
\newcommand{\argmin}{\mathop{\rm argmin}}

\newcommand{\abs}[1]{\left|{#1}\right|}
\newcommand{\norm}[1]{\left\|{#1}\right\|}

\newcommand{\otilde}[1]{\widetilde{O}\left({#1}\right)}

\renewcommand{\det}{\text{det}}

\newcommand{\trans}{^{\top}}
\newcommand{\poly}{\text{poly}}

\newcommand{\proj}{\mathcal{P}}

\newcommand{\N}{\mathbb{N}}
\newcommand{\R}{\mathbb{R}}

\newcommand{\iol}{improve or localize}

\newcommand{\A}{\mat{A}}
\newcommand{\B}{\mat{B}}

\newcommand{\I}{\mat{I}}

\newcommand{\e}{\vect{e}}

\renewcommand{\v}{\vect{v}}
\newcommand{\w}{\vect{w}}
\newcommand{\x}{\vect{x}}
\newcommand{\y}{\vect{y}}

\renewcommand{\H}{\mathcal{H}}

\newcommand{\cn}{\kappa}
\newcommand{\nn}{\nonumber}

\newtheorem{theorem}{Theorem}
\newtheorem{lemma}[theorem]{Lemma}
\newtheorem{corollary}[theorem]{Corollary}

\newtheorem*{fact}{Fact}
\newtheorem{proposition}[theorem]{Proposition}
\theoremstyle{definition}
\newtheorem{definition}{Definition}

\newcommand{\g}{\bm{g}}
\newcommand{\m}{\bm{m}}

\newcommand{\ufun}{\mathscr{E}}
\newcommand{\uspace}{\mathscr{S}}
\newcommand{\utime}{\mathscr{T}}
\newcommand{\umom}{\mathscr{M}}

\renewcommand{\S}{\mathcal{S}}

\newcommand{\zero}{\mathbf{0}}

\newcommand{\la}{\langle}
\newcommand{\ra}{\rangle}

\newcommand{\cXs}{\mathcal{X}_{\text{stuck}}}

\newcommand{\ball}{\mathbb{B}}
\newcommand{\EFSP}{$\epsilon$-first-order stationary point}
\newcommand{\ESSP}{$\epsilon$-second-order stationary point}
\newcommand{\ESP}{$\epsilon$-suboptimal point}

\renewcommand{\Im}{\mathrm{Im}}
\newcommand{\pmat}[1]{\begin{pmatrix} #1 \end{pmatrix}}
\newcommand{\modify}[1]{#1 '}

\begin{document}

\title{\textbf{Accelerated Gradient Descent Escapes Saddle Points Faster than Gradient Descent}}

\author{Chi Jin \\ University of California, Berkeley \\ \texttt{chijin@cs.berkeley.edu}
\and 
Praneeth Netrapalli \\ Microsoft Research, India \\ \texttt{praneeth@microsoft.com}
\and
Michael I. Jordan \\ University of California, Berkeley \\ \texttt{jordan@cs.berkeley.edu}}

\maketitle

\begin{abstract}
Nesterov's accelerated gradient descent (\nag), an instance of the general family of ``momentum methods,'' provably achieves faster convergence rate than gradient descent (\gd) in the convex setting. However, whether these methods are superior to~\gd~in the nonconvex setting remains open. This paper studies a simple variant of~\nag, and shows that it escapes saddle points and finds a second-order stationary point in $\tilde{O}(1/\epsilon^{7/4})$ iterations, faster than the $\tilde{O}(1/\epsilon^{2})$ iterations required by~\gd. To the best of our knowledge, this is the first Hessian-free algorithm to find a second-order stationary point faster than~\gd, and also the first single-loop algorithm with a faster rate than~\gd~even in the setting of finding a first-order stationary point. Our analysis is based on two key ideas: (1) the use of a simple Hamiltonian function, inspired by a continuous-time perspective, which~\nag~monotonically decreases per step even for nonconvex functions, and (2) a novel framework called~\emph{\iol}, which is useful for tracking the long-term behavior of gradient-based optimization algorithms. We believe that these techniques may deepen our understanding of both acceleration algorithms and nonconvex optimization.
%  key techinics relies on keep track of energy function of AGD which comes from ODE perspective, and we believe it may be of interests in general nonconvex optimization community.
% Linking movement with progress, locality.
\end{abstract}

\clearpage
% \newpage
% \setcounter{page}{1}

% \cnote{Think name of title, algorithm, improve or localize, negative curvature exploration}
%!TEX root = main.tex

\section{Introduction}\label{sec:intro}

Nonconvex optimization problems are ubiquitous in modern machine learning. 
While it is NP-hard to find global minima of a nonconvex function in the worst case, 
in the setting of machine learning it has proved useful to consider a less stringent
notion of success, namely that of convergence to a \emph{first-order stationary point} 
(where $\grad f(\x) = 0$).  Gradient descent (GD), a simple and fundamental 
optimization algorithm that has proved its value in large-scale machine learning, is 
known to find an~\EFSP ~(where $\norm{\grad f(\x)} \le \epsilon$) in $O(1/\epsilon^2)$ 
iterations \citep{nesterov1998introductory}, and this rate is sharp~\citep{cartis2010complexity}. 
Such results, however, do not seem to address the practical success of gradient descent;
first-order stationarity includes local minima, saddle points or even local maxima, and a
mere guarantee of convergence to such points seems unsatisfying.  Indeed, architectures
such as deep neural networks induce optimization surfaces that can be teeming with such 
highly suboptimal saddle points~\citep{dauphin2014identifying}.  It is important to 
study to what extent gradient descent avoids such points, particular in the high-dimensional
setting in which the directions of escape from saddle points may be few.

%  First-order stationary point
% It is well known that~\gd~obtains an $\epsilon$-first order stationary point i.e., where $\norm{\nabla f} \leq \epsilon$ in $\order{\frac{1}{\epsilon^2}}$ steps~\cite{nesterov1998introductory}. 
% Unfortunately however, it turns out that~\gd~can get stuck in saddle points (i.e., first order stationary points which are not second order stationary points)~\cite[Section~1.2.3]{nesterov1998introductory} and the rate of~\gd~cannot be improved beyond $\order{\frac{1}{\epsilon^2}}$ mentioned above~\cite{cartis2010complexity}.

This paper focuses on convergence to a \emph{second-order stationary point} 
(where $\grad f(\x) = 0$ and $\hess f(\x) \succeq 0$). Second-order stationarity 
rules out many common types of saddle points (\emph{strict} saddle points where 
$\lambda_{\min}(\hess f(\x)) <0$), allowing only local minima and higher-order saddle 
points. A significant body of recent work, some theoretical and some empirical, shows
that for a large class of well-studied machine learning problems, neither higher-order 
saddle points nor spurious local minima exist. That is, \emph{all second-order stationary 
points are (approximate) global minima} for these problems. \citet{choromanska2014loss,kawaguchi2016deep} 
present such a result for learning multi-layer neural networks, \citet{bandeira2016low,mei2017solving} 
for synchronization and MaxCut, \citet{boumal2016non} for smooth semidefinite programs, 
\citet{bhojanapalli2016global} for matrix sensing, \citet{ge2016matrix} for matrix completion, and
\citet{ge2017no} for robust PCA.
% More concretely, these works show that while high dimensional nonconvex optimization problems are replete with highly suboptimal saddle points, \emph{all} of these saddle points are \emph{first order stationary points} (i.e., where gradient is small). On the other hand, they also show that \emph{all second order stationary points} (i.e., where gradient is small and Hessian is almost positive semidefinite) are not only local minima but in fact are also (approximate) global minima.
%The remarkable aspect of these works is that while there are plenty of highly suboptimal saddle points in these nonconvex optimization problems, they manage to show that they are all \emph{first order stationary points} (i.e., where the gradient is small) but \emph{not second order stationary points}.
These results strongly motivate the quest for \emph{efficient algorithms} to find second-order stationary points.
 % can find local minima, or equivalently second order stationary points, in nonconvex optimization problems.

Hessian-based algorithms can explicitly compute curvatures and thereby avoid saddle points
(e.g.,~\citep{nesterov2006cubic, curtis2014trust}), but these algorithms are computationally
infeasible in the high-dimensional regime.  ~\gd, by contrast, is known to get stuck at 
strict saddle points \cite[Section~1.2.3]{nesterov1998introductory}.  Recent work has
reconciled this conundrum in favor of \gd; ~\cite{jin2017escape}, building on earlier work 
of~\cite{ge2015escaping}, show that a perturbed version of~\gd~converges to an $\epsilon$-relaxed 
version of a second-order stationary point (see Definition \ref{def:SOSP}) in $\tilde{O}(1/\epsilon^2)$ 
iterations.  That is, perturbed \gd~in fact finds second-order stationary points as fast as
standard \gd~finds first-order stationary point, up to logarithmic factors in dimension.

On the other hand,~\gd~is known to be suboptimal in the convex case. 
In a celebrated paper, \citet{nesterov1983method} showed that an accelerated version of
gradient descent (\nag) finds an~\ESP~(see Section \ref{sec:prelim_convex}) in  $O(1/\sqrt{\epsilon})$ 
steps, while gradient descent takes $O(1/\epsilon)$ steps.  The basic idea of acceleration
has been used to design faster algorithms for a range of other convex optimization 
problems~\citep{beck2009fast,nesterov2012efficiency,lee2013efficient,shalev2014accelerated}.
We will refer to this general family as ``momentum-based methods.''

Such results have focused on the convex setting.  It is open as to whether 
momentum-based methods yield faster rates in the \emph{nonconvex setting}, specifically when
we consider the convergence criterion of second-order stationarity.  We are thus led to ask 
the following question: \textbf{Do momentum-based methods yield faster convergence than~\gd~in 
the presence of saddle points?}

~

\begin{algorithm}[t]
\caption{Nesterov's Accelerated Gradient Descent ($\x_0, \eta, \theta$)}\label{algo:AGD}
\begin{algorithmic}[1]
\State $\v_0 \leftarrow 0$
\For{$t = 0, 1, \ldots, $}
\State $\y_{t} \leftarrow \x_{t} + (1-\theta) \v_t$
\State $\x_{t+1} \leftarrow \y_t - \eta \grad f (\y_t)$
\State $\v_{t+1} \leftarrow \x_{t+1} - \x_t $
%   \State $\y_{t} \leftarrow \x_{t} + (1-\theta) (\x_{t+1} - \x_t)$
%   \State $\x_{t+1} \leftarrow \y_t - \eta \grad f (\y_t)$
\EndFor
% \State \textbf{return} $\x_T$
\end{algorithmic}
\end{algorithm}

% \begin{algorithm}[t]
% \caption{Perturbed Accelerated Gradient Descent 
% ($\x_0, \eta, \theta, \gamma, s, r, \utime$)}\label{algo:PAGD}
% \begin{algorithmic}[1]
% \State $\v_0 \leftarrow 0$
% \For{$t = 0, 1, \ldots, $}
% \If{$\norm{\grad f(\x_t)} \le \epsilon$ and \emph{no perturbation in last $\utime$ steps}}
% \hspace{4em}\rlap{\smash{$\left.\begin{array}{@{}c@{}}\\{}\end{array} \color{blue} \right\}%
% 		\color{blue}\begin{tabular}{l}{Perturbation}\end{tabular}$}}
% \State $\x_t \leftarrow \x_t + \xi_t \qquad \xi_t \sim \text{Unif}\left(\mathbb{B}_0(r)\right)$
% \EndIf
% \State $\y_{t} \leftarrow \x_{t} + (1-\theta) \v_t$
% \State $\x_{t+1} \leftarrow \y_t - \eta \grad f (\y_t)$
% \hspace{20em}\rlap{\smash{$\left.\begin{array}{@{}c@{}}\\{}\\{}\end{array} \color{blue} \right\}%
% 		\color{blue}\begin{tabular}{l}{\nag}\end{tabular}$}}
% \State $\v_{t+1} \leftarrow \x_{t+1} - \x_t $
% % \If{along $\y_t$ to $\x_t$ is too nonconvex}
% \If{$f(\x_t) \le  f(\y_t) + \la \grad f(\y_t), \x_t - \y_t \ra - \frac{\gamma}{2} \norm{\x_t - \y_t}^2$}
% \State $(\x_{t+1}, \v_{t+1}) \leftarrow $ Negative-Curvature-Exploitation($\x_t, \v_t, s$)
% \hspace{0.5em}\rlap{\smash{$\left.\begin{array}{@{}c@{}}\\{}\end{array} \color{blue} \right\}%
% 		$}{\color{blue}\begin{tabular}{l}Negative curvature\\exploitation\end{tabular}}}
% \EndIf
% \EndFor
% % \State \textbf{return} $\x_T$
% \end{algorithmic}
% \end{algorithm}

\begin{algorithm}[t]
\caption{Perturbed Accelerated Gradient Descent 
($\x_0, \eta, \theta, \gamma, s, r, \utime$)}\label{algo:PAGD}
\begin{algorithmic}[1]
\State $\v_0 \leftarrow 0$
\For{$t = 0, 1, \ldots, $}
\If{$\norm{\grad f(\x_t)} \le \epsilon$ and \emph{no perturbation in last $\utime$ steps}}

% \hspace{4em}\rlap{\smash{\raisebox{\dimexpr.5\normalbaselineskip+.5\jot} $\left.\begin{array}{@{}c@{}}\\{}\end{array}  \right\}%
%         \begin{tabular}{l}{Perturbation}\end{tabular}$}}
\State $\x_t \leftarrow \x_t + \xi_t \qquad \xi_t \sim \text{Unif}\left(\mathbb{B}_0(r)\right)$
\hspace{12.2em}\smash{\raisebox{\dimexpr.4\normalbaselineskip+.5\jot}{\begin{scriptsize}
$%
            \left.\begin{array}{@{}c@{}}\\[\jot]\\[\jot]\end{array}\right\}$ \end{scriptsize} $\text{Perturbation}$}}
\EndIf
\State $\y_{t} \leftarrow \x_{t} + (1-\theta) \v_t$
\State $\x_{t+1} \leftarrow \y_t - \eta \grad f (\y_t)$
\hspace{19.3em}\rlap{\smash{$\left.\begin{array}{@{}c@{}}\\{}\\{}\end{array}  \right\}%
        \begin{tabular}{l}{\nag}\end{tabular}$}}
\State $\v_{t+1} \leftarrow \x_{t+1} - \x_t $
% \If{along $\y_t$ to $\x_t$ is too nonconvex}
\If{$f(\x_t) \le  f(\y_t) + \la \grad f(\y_t), \x_t - \y_t \ra - \frac{\gamma}{2} \norm{\x_t - \y_t}^2$}
\State $(\x_{t+1}, \v_{t+1}) \leftarrow $ Negative-Curvature-Exploitation($\x_t, \v_t, s$)
% \hspace{0.5em}\rlap{\smash{$\left.\begin{array}{@{}c@{}}\\{}\end{array}  \right\}%
%         $}{\begin{tabular}{l}Negative curvature\\exploitation\end{tabular}}}
\hspace{2em}\smash{\raisebox{\dimexpr.4\normalbaselineskip+.5\jot}{\begin{scriptsize}
			$\left.\begin{array}{@{}c@{}}\\[\jot]\\[\jot]\end{array}\right\}$
		\end{scriptsize}%
            $\begin{tabular}{l}Negative curvature\\exploitation\end{tabular}$}}
\EndIf
\EndFor
\end{algorithmic}
\end{algorithm}
\vspace{-0.5cm}
%\pn{Fix algorithm boxes. bracket too large}

This paper answers this question in the affirmative. We present a simple momentum-based algorithm (\pagd~for ``perturbed~\nag'') that finds an $\epsilon$-second order stationary point in $\tilde{O}(1/\epsilon^{7/4})$ iterations, faster than the $\tilde{O}(1/\epsilon^{2})$ iterations required by~\gd.
%gradient descent \citep{jin2017escape}.
The pseudocode of our algorithm is presented in Algorithm~\ref{algo:PAGD}.\footnote{See Section~\ref{sec:results} for values of various parameters.}
%Our algorithm as highlighted in Algorithm \ref{algo:PAGD}.
\pagd~adds two algorithmic features to~\nag~(Algorithm \ref{algo:AGD}):
\begin{itemize}
\item Perturbation (Lines 3-4): when the gradient is small, we add a small perturbation sampled uniformly from a $d$-dimensional ball with radius $r$.
The homogeneous nature of this perturbation mitigates our lack of knowledge of the curvature tensor
at or near saddle points.
\item Negative Curvature Exploitation (\nce, Lines 8-9; pseudocode in Algorithm~\ref{algo:NCE}): 
when the function becomes ``too nonconvex'' along $\y_t$ to $\x_t$, we reset the momentum and 
decide whether to exploit negative curvature depending on the magnitude of the current momentum $\v_t$. 
%This guarantees a decrease of our Hamiltonian criterion function even in the highly nonconvex 
%setting (see Section \ref{sec:tech} for details).
\end{itemize}
We note that both components are straightforward to implement and increase computation by a constant factor. 
The perturbation idea follows from~\cite{ge2015escaping} and~\cite{jin2017escape}, while~\nce~is 
inspired by~\citep{carmon2017convex}.
To the best of our knowledge, \pagd~is the first Hessian-free algorithm to find a second-order stationary point in $\tilde{O}(1/\epsilon^{7/4})$ steps.  Note also that \pagd~is a ``single-loop algorithm,'' 
meaning that it does not require an inner loop of optimization of a surrogate function.  It is the
first single-loop algorithm to achieve a $\tilde{O}(1/\epsilon^{7/4})$ rate even in the setting of 
finding a first-order stationary point. 

% \pn{Informal theorem here? Might have to introduce too much notation.}

% \pn{Single loop vs two loops etc.
% 	Compare to Carmon et al. papers and others you will compare in the table.}
% \cnote{See section ?? for overview of our techniques}

% \begin{enumerate}
% \item Nonconvex first order GD rate
% \item Nonconvex why second order
% \item Nonconvex second order GD rate
% \item Acceleration and convex
% \item Question
% \item Our algorithm and result
% \end{enumerate}

%!TEX root = main.tex

\subsection{Related Work}

{\small
\begin{table}[t]
  \begin{center}
    {\renewcommand{\arraystretch}{1.3}
    \begin{tabular}  {l | l | l | l | l}
       \toprule
\textbf{Guarantees}  & \textbf{Oracle}  & \textbf{Algorithm} & \textbf{Iterations} & \textbf{Simplicity}\\ 
\midrule 
\multirow{3}{*}{\makecell[l]{First-order\\Stationary\\Point}} & \multirow{3}{*}{Gradient}  & GD \citep{nesterov1998introductory} & $O(1/\epsilon^2)$  & Single-loop\\ 
& & AGD \citep{ghadimi2016accelerated} & $O(1/\epsilon^2)$ &  Single-loop\\ 
 &  & \cite{carmon2017convex} & $\otilde{1/\epsilon^{7/4}}$ &  Nested-loop\\
      \midrule
% \multirow{6}{*}{\makecell[l]{Second-order\\Stationary\\Point}} & \multirow{3}{*}{\makecell{Hessian\\-vector}} & \citet{carmon2016gradient} & $\widetilde{O}(1/\epsilon^2)$ &   Nested-loop\\ 
%  & & \citet{agarwal2016finding}
%  & $\widetilde{O}(1/\epsilon^{7/4})$ & Nested-loop\\
\multirow{5}{*}{\makecell[l]{Second-order\\Stationary\\Point}} & \multirow{2}{*}{\makecell{Hessian\\-vector}}  & \citet{carmon2016accelerated}
 & $\widetilde{O}(1/\epsilon^{7/4})$ & Nested-loop\\
  & & \citet{agarwal2017finding}
 & $\widetilde{O}(1/\epsilon^{7/4})$ & Nested-loop\\ 
      \cmidrule {2-5}
& \multirow{3}{*}{Gradient} &  Noisy GD \citep{ge2015escaping} & $O(\poly(d/\epsilon))$ &  Single-loop\\ 
 % & $O(\poly(d/\epsilon))$ & Gradient &  SOSP &Single-loop\\
 &  &Perturbed GD \citep{jin2017escape} & $\widetilde{O}(1/\epsilon^{2})$ & Single-loop \\ 
 &  &\textbf{Perturbed AGD [This Work]} & $\widetilde{O}(1/\epsilon^{7/4})$ &  Single-loop\\ 

    \bottomrule
    \end{tabular}
      \caption{Complexity of finding stationary points. $\widetilde{O}(\cdot)$ ignores polylog factors in $d$ and $\epsilon$.}
      \label{tab:main}
    }
  \end{center}
  % \praneeth{I think it will be cleaner to make the dependence on smoothness paramters explicit here.}
  \vspace{-4ex}
\end{table}
}
In this section, we review related work from the perspective of both nonconvex optimization and momentum/acceleration. For clarity of presentation, when discussing rates, we focus on the dependence on the accuracy $\epsilon$ and the dimension $d$ while assuming all other problem parameters are constant. Table~\ref{tab:main} presents a comparison of the current work with previous work.

\textbf{Convergence to first-order stationary points:} Traditional analyses in this case assume only Lipschitz gradients (see Definition \ref{def:smooth}).~\cite{nesterov1998introductory} shows that~\gd~finds an~\EFSP~in $O(1/\epsilon^2)$ steps.~\cite{ghadimi2016accelerated}~guarantee that~\nag~also converges in $\otilde{1/\epsilon^2}$ steps. Under the additional assumption of Lipschitz Hessians (see Definition \ref{def:HessianLip}), \citet{carmon2017convex} develop a new algorithm that converges in $O(1/\epsilon^{7/4})$ steps. Their algorithm is a nested-loop algorithm, where the outer loop adds a proximal term to reduce the nonconvex problem to a convex subproblem. A key novelty in their algorithm is the idea of ``negative curvature exploitation," which inspired a similar step in our algorithm.
%By adding proxy, their algorithm reduces nonconvex optimization to a series of subproblem, which is either ``strongly convex'' so AGD is applicable or ``guilty'' so exploiting negative curvature. Our negative curvature exploitation component is mainly inspired by this work.
In addition to the qualitative and quantitative differences between~\cite{carmon2017convex} and the
current work, as summarized in Table~\ref{tab:main}, we note that while~\cite{carmon2017convex} 
analyze~\nag~applied to convex subproblems, we analyze~\nag~applied directly to nonconvex functions 
through a novel Hamiltonian framework.
%We also note that differences from our work: other than we guarantees second-order stationary point while \citet{carmon2017convex} only guarantees first-order version, we also directly apply AGD to nonconvex function while \citet{carmon2017convex} applies AGD to strongly convex subproblem, which results in significant difference in analysis, also they are running algorithm with nested-loop while our algorithm is single-loop.

% \cnote{Did I say too much about Carmon et al 2017? I don't want Carmon et al feel we are trolling their work...}

\textbf{Convergence to second-order stationary points:} All results in this setting assume Lipschitz
conditions for both the gradient and Hessian.
%Since second-order stationary point requires Hessian to be almost positive definite. A natural way to guarantee Hessian property is by directly using Hessian information per step.
Classical approaches, such as cubic regularization~\citep{nesterov2006cubic} and trust region algorithms~\citep{curtis2014trust}, require access to Hessians, and are known to find $\epsilon$-second-order stationary points in $O(1/\epsilon^{1.5})$ steps. However, the requirement of these algorithms to form the Hessian makes them infeasible for 
high-dimensional problems. A second set of algorithms utilize only Hessian-vector products instead 
of the explicit Hessian; in many applications such products can be computed efficiently. 
Rates of $\widetilde{O}(1/\epsilon^{7/4})$ have been established for such algorithms~\citep{carmon2016accelerated,agarwal2017finding,royer2017complexity}.
Finally, in the realm of purely gradient-based algorithms,
%(which is typically simple, easy to implement and efficient in high dimensional setting),
%(simple to implement and efficient in high-dimensional settings)
\citet{ge2015escaping} present the first polynomial guarantees for a perturbed version of~\gd, and~\citet{jin2017escape} sharpen it to $\widetilde{O}(1/\epsilon^2)$. For the special case of quadratic functions,~\cite{o2017behavior} analyze the behavior of~\nag~around critical points and show that it escapes saddle points faster than~\gd. We note that the current work is the first achieving a rate of $\widetilde{O}(1/\epsilon^{7/4})$ for general nonconvex functions. %\cnote{add Wright paper, quadratic case.}

\textbf{Acceleration:} There is also a rich literature that aims to understand momentum methods; 
e.g., \citet{allen2014linear} view~\nag~as a linear coupling of~\gd~and mirror descent, \citet{su2016differential} and \citet{wibisono2016} view~\nag~as a second-order differential equation, and \citet{bubeck2015geometric} view~\nag~from a geometric perspective. Most of this work is tailored to the convex setting, and it is unclear and nontrivial to generalize the results to a nonconvex setting. There are also several papers that study~\nag~with relaxed versions of convexity---see~\citet{necoara2015linear,li2017provable} and references therein for 
overviews of these results.

\subsection{Main Techniques}\label{sec:tech}
Our results rely on the following three key ideas. To the best of our knowledge, the first two are novel, while the third one was delineated in~\citet{jin2017escape}.

\textbf{Hamiltonian:}
A major challenge in analyzing momentum-based algorithms is that the objective function does not 
decrease monotonically as is the case for~\gd. To overcome this in the convex setting, several Lyapunov functions have been proposed~\citep{wilson2016lyapunov}. However these Lyapunov functions involve the 
global minimum $\x^\star$, which cannot be computed by the algorithm, and is thus of limited value in
the nonconvex setting. A key technical contribution of this paper is the design of a function 
which is both computable and tracks the progress of~\nag. The function takes the form of a
Hamiltonian:
\begin{equation}\label{eqn:hamiltonian}
	E_t \defeq f(\x_t) + \frac{1}{2\eta} \norm{\v_t}^2;
\end{equation}
i.e., a sum of potential energy and kinetic energy terms.  It is monotonically decreasing
in the continuous-time setting.  This is \emph{not} the case in general in the discrete-time setting,
a fact which requires us to incorporate the~\nce~step.

\textbf{Improve or localize:}
Another key technical contribution of this paper is in formalizing a simple but powerful framework 
for analyzing nonconvex optimization algorithms.  This framework requires us to show that for a 
given algorithm, \emph{either the algorithm makes significant progress or the iterates do not move much}.
% lie in a small ball around the starting point}. 
We call this the~\emph{improve-or-localize}~phenomenon. For instance, when progress is measured 
by function value, it is easy to show that for~\gd, with proper choice of learning rate, we have: 
$$ \frac{1}{2\eta} \sum_{\tau=0}^{t-1} \norm{\x_{\tau+1} - \x_\tau}^2 \le f(\x_0) - f(\x_t).$$
% which means either function decreases a lot, or the sum of distance square is upper-bounded. 
For~\nag, a similar lemma can be shown by replacing the objective function with the Hamiltonian 
(see Lemma \ref{lem:energy_nonconvex}).  Once this phenomenon is established, we can conclude 
that if an algorithm does not make much progress, it is localized to a small ball, and we can then 
approximate the objective function by either a linear or a quadratic function (depending on smoothness 
assumptions) in this small local region. Moreover, an upper bound on 
$\sum_{\tau=0}^{t-1} \norm{\x_{\tau+1} - \x_\tau}^2$ lets us conclude that iterates do not 
oscillate much in this local region (oscillation is a unique phenomenon of momentum algorithms 
as can be seen even in the convex setting).  This gives us better control of approximation error.

\textbf{Coupling sequences for escaping saddle points:} When an algorithm arrives in the neighborhood of a strict saddle point, where $\lambda_{\min}(\hess f(\x)) <0$, all we know is that there exists a direction of
escape (the direction of the minimum eigenvector of $\hess f(\x)$); denote it by $\e_{\text{esc}}$. To avoid such points, the algorithm randomly perturbs the current iterate uniformly in a small ball, and runs~\nag~starting from this point $\tilde{\x}_0$. %sampling uniform from a ball.
As in~\cite{jin2017escape}, we can divide this ball into a ``stuck region,'' 
$\mathcal{X}_{\text{stuck}}$, starting from which~\nag~does not escape the saddle 
quickly, and its complement from which~\nag~escapes quickly. In order to show 
quick escape from a saddle point, we must show that the volume of $\mathcal{X}_{\text{stuck}}$ 
is very small compared to that of the ball. Though $\mathcal{X}_{\text{stuck}}$ may 
be without an analytical form, one can control the rate of escape by studying 
two~\nag~sequences that start from two realizations of perturbation, $\tilde{\x}_0$ and $\tilde{\x}'_0$, 
which are separated along $\e_{\text{esc}}$ by a small distance $r_0$. In this case, 
at least one of the sequences escapes the saddle point quickly, which proves that the 
width of $\mathcal{X}_{\text{stuck}}$ along $\e_{\text{esc}}$ can not be greater than 
$r_0$, and hence $\mathcal{X}_{\text{stuck}}$ has small volume.

%  and would like to show that with high probability, the perturbation gives us a better perturbed point $\tilde{\x}_0$, which leads to faster escaping saddle point by running AGD starting at $\tilde{\x}_0$.

% The key requirement for implementing this framework is establishing~\iol~phenomenon for the given algorithm -- in our case,~\nag. One of the key technical contributions of this paper is in designing an energy function for~\nag~and using it to establish~\iol~phenomenon.

% \pn{Put lemma here on energy function and~\iol}

% \pn{Put some stress on novelty of energy function?}
% Once~\iol~phenomenon is established, we approximate the function locally by a quadratic, and analyze~\nag's performance on this quadratic. Bounding approximation errors introduced here turns out to be technically quite challenging. Bulk of this paper is devoted to doing just that.

% We would like to reiterate that while the~\iol~framework as well as establishing this phenomenon for~\nag~are quite simple, they are quite powerful and yield short and simple proofs of some well-known existing results e.g., convergence rate of~\nag~to first order stationary points originally proved in~\cite{ghadimi2016accelerated}. See Section~\ref{sec} for details. We believe that~\iol~is the right framework for nonconvex optimization and will help push its boundaries.\pn{Something better in the last sentence?}

%!TEX root = main.tex

\section{Preliminaries}
In this section, we will review some well-known results on~\gd~and~\nag~in the strongly convex setting, 
and existing results on convergence of~\gd~to second-order stationary points. 
% The pseudocode for these algorithms is given in Algorithms~\ref{algo:gd} and~\ref{algo:AGD} respectively.

% \cnote{Show gradient descent in equation}

\subsection{Notation}
Bold upper-case letters ($\A, \B$) denote matrices and bold lower-case letters ($\x, \y$) denote vectors. 
For vectors $\norm{\cdot}$ denotes the $\ell_2$-norm. For matrices, $\norm{\cdot}$ denotes the spectral norm and $\lambda_{\min}(\cdot)$ denotes the minimum eigenvalue.
For $f: \R^d \rightarrow \R$, $\grad f(\cdot)$ and  $\hess f(\cdot)$ denote its gradient and Hessian respectively, and $f^\star$ denotes its global minimum.
% Other than Section \ref{sec:related}, 
We use $O(\cdot), \Theta(\cdot), \Omega(\cdot)$ to hide absolute constants, and $\tilde{O}(\cdot), \tilde{\Theta}(\cdot), \tilde{\Omega}(\cdot)$ to hide absolute constants and polylog factors for all problem parameters. 
% \praneeth{I think it will be cleaner to make the dependence on smoothness parameters in Table~\ref{tab:main} and edit this statement} \jccomment{Then I also need to add function value dependence, maybe too complicated to compare}.\praneeth{The issue with this is that $O()$ is not just hiding constants but also problem dependent parameters. May be mention this explicitly in the caption to the table.} 
% We let $\ball^{(d)}_\x(r)$ denote the d-dimensional ball centered at $\x$ with radius $r$; when it is clear from context, we simply denote it as $\ball_\x(r)$. We use $\proj_{\mathcal{X}}(\cdot)$ to denote projection onto the set $\mathcal{X}$. Distance and projection are always defined in a Euclidean sense.

% \pn{Talk about ignoring $\log d$ factors in notation.}

\subsection{Convex Setting}\label{sec:prelim_convex}
% \begin{figure}[t]
% \begin{minipage}{0.5\textwidth}
% 	\begin{algorithm}[H]
% 	\caption{\gd($\x_0, \eta$)}\label{algo:gd}
% 	\begin{algorithmic}[1]
% 		\For{$t = 0, 1, \ldots, T $}
% 		\State $\x_{t+1} \leftarrow \x_t - \eta \grad f (\x_t)$
% 		\EndFor
% 		\State \textbf{return} $\x_T$
% 	\end{algorithmic}
% 	\end{algorithm}
% 	\vspace{0.5cm}
% \end{minipage}
% \begin{minipage}{.5\textwidth}

% \end{minipage}
% \end{figure}
To minimize a function $f(\cdot)$,~\gd ~performs the following sequence of steps:
\begin{equation*}
\x_{t+1} = \x_{t}- \eta \grad f(\x_t).
\end{equation*}
The suboptimality of~\gd~and the improvement achieved by~\nag~can be clearly illustrated for the case of smooth and strongly convex functions. %The definitions of smoothness and strong convexity are as follows.
\begin{definition}\label{def:smooth}
A differentiable function $f(\cdot)$ is \textbf{$\ell$-smooth (or $\ell$-gradient Lipschitz)} if:
\begin{equation*}
\norm{\grad f(\x_1) - \grad f(\x_2)} \le \ell \norm{\x_1 - \x_2} \quad \forall \; \x_1, \x_2.
\end{equation*}
\end{definition}
\noindent
The gradient Lipschitz property asserts that the gradient can not change too rapidly in a small local region.
\begin{definition}\label{def:convex}
A twice-differentiable function $f(\cdot)$ is \textbf{$\alpha$-strongly convex} if
$\lambda_{\min}(\hess f(\x)) \ge \alpha, \;  \forall \; \x$.
% $f(\x_2) \ge f(\x_1) + \la \grad f(\x_1), \x_2 - \x_1 \ra + \frac{\alpha}{2}\norm{\x_2 - \x_1}^2, \quad \forall \; \x_1, \x_2.$
\end{definition}
Let $\fstar \defeq \min_{\y}f(\y)$. A point $\x$ is said to be \textbf{$\epsilon$-suboptimal} if $f(\x)  \le  \fstar + \epsilon$. The following theorem gives the convergence rate of GD and AGD for smooth and strongly convex functions.
\begin{theorem}[\cite{nesterov2004introductory}]\label{thm:gd_convex}
Assume that the function $f(\cdot)$ is $\ell$-smooth and $\alpha$-strongly convex. Then, for any $\epsilon>0$,
the iteration complexities to find an $\epsilon$-suboptimal point are as follows:
\begin{itemize}
\item GD with $\eta  = 1/\ell$: \quad $O((\ell/\alpha) \cdot \log ((f(\x_0) - \fstar)/\epsilon))$
\item AGD (Algorithm~\ref{algo:AGD}) with $\eta = 1/\ell$ and $\theta = \sqrt{\alpha/\ell}$:
\quad$O(\sqrt{\ell/\alpha} \cdot \log ((f(\x_0) - \fstar)/\epsilon))$.
\end{itemize}
% ~\gd~with $\eta = \frac{1}{\ell}$ will output an \ESP ~in iterations:
% \begin{equation*}
% O\left(\frac{\ell}{\alpha}\log \frac{f(\x_0) - \fstar}{\epsilon}\right).
% \end{equation*}
\end{theorem}

The number of iterations of GD depends linearly on the ratio $\ell/\alpha$, which is called the condition number of $f(\cdot)$ since $\alpha \I \preceq\hess f(\x) \preceq \ell \I $. Clearly $\ell \geq \alpha$ and hence condition number is always at least one. Denoting the condition number by ${\cn}$, we highlight two important aspects of~\nag: (1) the momentum parameter satisfies $\theta = 1/\sqrt{\cn}$ and (2) \nag~improves upon GD by a factor of $\sqrt{\cn}$. 
% The following theorem gives the convergence rate of~\nag~for these problems.
% \begin{theorem}[\cite{nesterov2004introductory}]\label{thm:agd_convex}
% Assume that the function $f(\cdot)$ is $\ell$-smooth and convex. Then, for any $\epsilon>0$,~\nag~with $\eta = \frac{1}{\ell}$ and $\theta = \Theta(\sqrt{\frac{\alpha}{\ell}}) $ will output an~\ESP~in iterations:
% \begin{equation*}
% O\left(\sqrt{\frac{\ell}{\alpha}}\log \frac{f(\x_0)-\fstar }{\epsilon}\right).
% \end{equation*}
% \end{theorem}
% \noindent
% Note that the rate here improves upon that of~\gd~by a factor of $\sqrt{\frac{\ell}{\alpha}}$ i.e., squareroot of the condition number.
%say something about condition number.

\subsection{Nonconvex Setting}
For nonconvex functions finding global minima is NP-hard in the worst case. The best one can hope for in this setting is convergence to stationary points. There are various levels of stationarity.
\begin{definition}
$\x$ is an \textbf{\EFSP} of function $f(\cdot)$ if $\norm{\grad f(\x)} \le \epsilon$.
\end{definition}
\noindent
As mentioned in Section~\ref{sec:intro}, for most nonconvex problems encountered in practice, a majority of first-order stationary points turn out to be saddle points. Second-order stationary points require not only zero gradient, but also positive semidefinite Hessian, ruling out most saddle points.
%Therefore, this paper focus on finding second-order stationary point.
%In order to discuss Hessian-related properties meaningfully, we first need to assert Hessian smoothness condition.
Second-order stationary points are meaningful, however, only when the Hessian is continuous.
% second order stationary points which means that in addition to being first order stationary points, the Hessian at these points is almost positive semidefinite. This is meaningful only if the Hessian does not change arbitrarily (and perhaps have large negative eigenvalues) in a small neighborhood around this point. In other words, finding second order stationary points is meaningful only if the Hessian is continuous.
%\cnote{Should we talk the case where gradient is Lipschitz but Hessian is not?}
% \begin{theorem}[\citep{nesterov1998introductory}]\label{thm:grad_smooth}
% Assume that the function $f(\cdot)$ is $\ell$-smooth. Then, for any $\epsilon>0$, gradient descent will output an \EFSP ~in iterations:
% \begin{equation*}
% \frac{\ell(f(\x_0) - f^\star)}{\epsilon^2}.
% \end{equation*}
% \end{theorem}
\begin{definition}\label{def:HessianLip}
A twice-differentiable function $f(\cdot)$ is \textbf{$\rho$-Hessian Lipschitz} if:
\begin{equation*}
\norm{\hess f(\x_1) - \hess f(\x_2)} \le \rho \norm{\x_1 - \x_2} \quad \forall \; \x_1, \x_2.
\end{equation*}
\end{definition}
\noindent
% For Hessian Lipschitz functions, we recall the definition of second order stationary points from~\cite{nesterov2006cubic}.
\begin{definition}[\cite{nesterov2006cubic}]\label{def:SOSP}
For a $\rho$-Hessian Lipschitz function $f(\cdot)$, $\x$ is an \textbf{\ESSP} if:
% $\norm{\grad f(\x)} \le \epsilon$ and $\lambda_{\min}(\hess f(\x)) \ge - \sqrt{\rho \epsilon}$.
\begin{equation*}
\norm{\grad f(\x)} \le \epsilon \quad\text{and}\quad \lambda_{\min}(\hess f(\x)) \ge - \sqrt{\rho \epsilon}.
\end{equation*}
\end{definition}
\noindent
The following theorem gives the convergence rate of a perturbed version of~\gd~to second-order stationary points. See~\cite{jin2017escape} for a detailed description of the algorithm.
\begin{theorem}[\citep{jin2017escape}]\label{thm:perturbed_GD}
Assume that the function $f(\cdot)$ is $\ell$-smooth and $\rho$-Hessian Lipschitz. Then, for any $\epsilon>0$, perturbed GD outputs an \ESSP ~w.h.p in iterations:
\begin{equation*}
\otilde{\frac{\ell(f(\x_0) - \fstar)}{\epsilon^2}}.
\end{equation*}
\end{theorem}
\noindent
Note that this rate is essentially the same as that of~\gd~for convergence to first-order stationary points. In particular, it only has polylogarithmic dependence on the dimension.

%!TEX root = main.tex

\section{Main Result}\label{sec:results}

% \begin{algorithm}[t]
% \caption{Perturbed AGD with Negative Curvature Exploration
% ($\x_0, \eta, \theta, r, s$)}\label{algo:PAGD}
% \begin{algorithmic}[1]
% \State $\v_0 \leftarrow 0$
% \For{$t = 0, 1, \ldots, T $}
% \If{perturbation condition holds}
% \State $\x_t \leftarrow \x_t + \xi_t \qquad \xi_t \sim \text{Unif}\left(B_0(r)\right)$
% \EndIf
% \State $\y_{t} \leftarrow \x_{t} + (1-\theta) \v_t$
% \State $(\x_{t+1}, ~\v_{t+1}) \leftarrow (\y_t - \eta \grad f (\y_t), ~\x_{t+1} - \x_t)$
% \If{certificate Eq.\eqref{eq:certificate} is false}
% \State $(\x_{t+1}, \v_{t+1}) \leftarrow $ Negative-Curvature-Exploitation($\x_t, \v_t, s$)
% \EndIf
% \EndFor
% \State \textbf{return} $\x_T$
% \end{algorithmic}
% \end{algorithm}

\begin{algorithm}[t]
\caption{Negative Curvature Exploitation$\left(\x_t, \v_t, s\right)$}\label{algo:NCE}
\begin{algorithmic}[1]
%	\Procedure{Negative-Curvature-Exploration}{$\x_t, \v_t, s$}
\If{$\norm{\v_t} \ge s$}
\State $\x_{t+1} \leftarrow \x_t$;
\Else
\State $\delta = s\cdot \v_t/\norm{\v_t}$
\State $\x_{t+1} \leftarrow \argmin_{\x \in \{\x_t + \delta, \x_t - \delta\}} f(\x)$
\EndIf
\State \textbf{return} $(\x_{t+1}, 0)$
%	\EndProcedure
\end{algorithmic}
\end{algorithm}

% \begin{figure}[t]
% 	\begin{minipage}{0.48\textwidth}
% 	\begin{algorithm}[H]
% 	\caption{Perturbed AGD with Negative Curvature Exploration
% 	($\x_0, \eta, \theta, r, s$)}\label{algo:PAGD}
% 	\begin{algorithmic}[1]
% 	\State $\v_0 \leftarrow 0$
% 	\For{$t = 0, 1, \ldots, T $}
% 	\If{perturbation condition holds}
% 	\State $\x_t \leftarrow \x_t + \xi_t \qquad \xi_t \sim \text{Unif}\left(B_0(r)\right)$
% 	\EndIf
% 	\State $\y_{t} \leftarrow \x_{t} + (1-\theta) \v_t$
% 	\State $(\x_{t+1}, ~\v_{t+1}) \leftarrow (\y_t - \eta \grad f (\y_t), ~\x_{t+1} - \x_t)$
% 	\If{certificate Eq.\eqref{eq:certificate} is false}
% 	\State $(\x_{t+1}, \v_{t+1}) \leftarrow $ Negative-Curvature-Exploration($\x_t, \v_t, s$)
% 	\EndIf
% 	\EndFor
% 	\State \textbf{return} $\x_T$
% 	\end{algorithmic}
% 	\end{algorithm}
% 	\end{minipage}
% 	\hfill
% 	\begin{minipage}{0.48\textwidth}
% 	\begin{algorithm}[H]
% 	\caption{Negative Curvature Exploitation$\left(\x_t, \v_t, s\right)$}\label{algo:NCE}
% 	\begin{algorithmic}[1]
% %	\Procedure{Negative-Curvature-Exploration}{$\x_t, \v_t, s$}
% 	\If{$\norm{\v_t} \ge s$}
% 	\State $\x_{t+1} \leftarrow \x_t$;
% 	\Else
% 	\State $\delta = s\cdot \v_t/\norm{\v_t}$
% 	\State $\x_{t+1} \leftarrow \argmin_{\x \in \{\x_t + \delta, \x_t - \delta\}} f(\x)$
% 	\EndIf
% 	\State \textbf{return} $(\x_{t+1}, 0)$
% %	\EndProcedure
% 	\end{algorithmic}
% 	\end{algorithm}
% 	\vspace{1cm}
% 	\end{minipage}
% \end{figure}

In this section, we present our algorithm and main result. As mentioned in Section~\ref{sec:intro}, the algorithm we propose is essentially~\nag~with two key differences (see Algorithm~\ref{algo:PAGD}): perturbation and negative curvature exploitation (\nce). A perturbation is added when the gradient is small (to escape saddle points), and no more frequently than once in $\utime$ steps. The perturbation $\xi_t$ is sampled uniformly from a $d$-dimensional ball with radius $r$. The specific choices of gap and uniform distribution are for
technical convenience (they are sufficient for our theoretical result but not necessary).

\nce~(Algorithm~\ref{algo:NCE}) is explicitly designed to guarantee decrease of the Hamiltonian~\eqref{eqn:hamiltonian}.
When it is triggered, i.e., when
\begin{equation}\label{eq:certificate}
f(\x_t) \le  f(\y_t) + \la \grad f(\y_t), \x_t - \y_t \ra - \frac{\gamma}{2} \norm{\x_t - \y_t}^2
\end{equation}
the function has a large negative curvature between the current iterates $\x_{t}$ and $\y_t$. In this case, if the momentum $\v_t$ is small, then $\y_t$ and $\x_t$ are close, so the large negative curvature also carries over to the Hessian at $\x_t$ due to the Lipschitz property. Assaying two points along $\pm(\y_t - \x_t)$ around $\x_t$ gives one point that is negatively aligned with $\grad f(\x_t)$ and yields a decreasing function value and Hamiltonian.
If the momentum $\v_t$ is large, negative curvature can no longer be exploited, but fortunately resetting the momentum to zero kills the second term in~\eqref{eqn:hamiltonian}, significantly decreasing the Hamiltonian.
\vspace{0.2cm}

\noindent\textbf{Setting of hyperparameters:} Let $\epsilon$ be the target accuracy for a second-order stationary point, let $\ell$ and $\rho$ be gradient/Hessian-Lipschitz parameters, and let $c, \chi$ be absolute constant and log factor to be specified later.
Let $\cn \defeq \ell/\sqrt{\rho\epsilon}$, and set
\begin{equation}
\eta = \frac{1}{4\ell}, \quad
\theta = \frac{1}{4\sqrt{\cn}},
\quad \gamma = \frac{\theta^2}{\eta} ,
\quad s = \frac{\gamma}{4\rho}, 
\quad \utime = \sqrt{\cn}\cdot \chi c,
\quad r = \eta\epsilon\cdot \chi^{-5}c^{-8}.
\label{eq:parameter}
\end{equation}

\noindent
%Now we are ready to present our main theorem:
The following theorem is the main result of this paper.

% parameter setting

% theorem

% More formally, negative curvature exploitation is triggered if the following \emph{does not hold}: \pn{Do you want to state the certificate this way or opposite?}
% %\begin{enumerate}
% %	\item \textbf{Perturbation} (lines $3$--$4$ of Algorithm~\ref{algo:PAGD}): 
% %	\item \textbf{Negative curvature exploitation} (lines $7$--$8$ of Algorithm~\ref{algo:PAGD}): 
% %\end{enumerate}
% % \begin{equation}
% % .
% % \label{eq:certificate}
% % \end{equation}
% \pn{Give the value of $\gamma=\sqrt{\rho \epsilon}$.}
% In this case, negative curvature exploitation (pseudocode in Algorithm~\ref{algo:NCE}) decides to either move in the direction of momentum or stay as is (depending on the magnitude of momentum), and resets momentum to $0$.
% %where $s$ in algorithm should be of $O(\sqrt{\frac{\epsilon}{\rho}})$.
% At a high level, the algorithm we are considering is essentially~\nag~with 
% occasional perturbations and in the presence of highly negative curvature, exploits it. The following theorem is the main result of this paper.
% \cnote{explain what is step 5 doing, and why this decrease Hamiltonian, why it depends on the magnitude of momentum.}

% \cnote{Also say a bit more words on the perturbation component, what's the perturbation condition}\

% \cnote{Say something about termination condition.}

% \noindent \textbf{Parameter Setting:} let $\chi = c \log \frac{d \ell\Delta_f}{\rho \epsilon\delta}$.
% \begin{equation*}
% \eta = \frac{1}{2\ell}
% \end{equation*}

\begin{theorem}\label{thm:main}
Assume that the function $f(\cdot)$ is $\ell$-smooth and $\rho$-Hessian Lipschitz.  There exists an absolute constant $c_{\max}$ such that for any $\delta >0$, $\epsilon \le \frac{\ell^2}{\rho}$, $\Delta_f \ge f(\x_0) - f^\star$, if $\chi =\max\{1, \log \frac{d \ell\Delta_f}{\rho \epsilon\delta}\}$, $c\ge c_{\max}$ and such that if we run~\pagd~(Algorithm~\ref{algo:PAGD}) with choice of parameters according to~\eqref{eq:parameter}, then with probability at least $1-\delta$, one of the iterates $\x_t$ will be an $\epsilon$-second order stationary point
in the following number of iterations:
\begin{equation*}
O\left(\frac{\ell^{1/2}\rho^{1/4}(f(\x_0) - f^*)}{\epsilon^{7/4}} \log^6 \left(\frac{d \ell\Delta_f}{\rho \epsilon\delta}\right)\right)
\end{equation*}
\end{theorem}
\noindent
Theorem~\ref{thm:main} says that when~\pagd~is run for the designated number of steps (which is poly-logarithmic in dimension), at least one of the iterates is an~\ESSP. We focus on the case of small $\epsilon$ (i.e., $\epsilon \le \ell^2/\rho$) so that the Hessian requirement for the~\ESSP~($\lambda_{\min}(\hess f(\x)) \ge -\sqrt{\rho \epsilon}$) is nontrivial.
Note that $\norm{\hess f(\x)} \le \ell$ implies $\cn = \ell/\sqrt{\rho\epsilon}$,
which can be viewed as a condition number, akin to that in convex setting.
%We denote $\ell/\sqrt{\rho\epsilon}$ as $\cn$.
Comparing Theorem \ref{thm:main} with Theorem~\ref{thm:perturbed_GD},~\pagd, with a momentum parameter $\theta = \Theta(1/\sqrt{\cn})$, achieves $\tilde{\Theta}(\sqrt{\cn})$ better iteration complexity compared to~\pgd. 
% The dimension dependence in iteration complexity is poly-logarithmic ($\log^8 d$), slightly worse than~\pgd.

\vspace{0.2cm}

\noindent\textbf{Output $\epsilon$-second order stationary point:}
Although Theorem~\ref{thm:main} only guarantees that one of the iterates is an $\epsilon$-second order stationary point, it is straightforward to identify one of them by adding a proper termination condition: once the gradient is small and satisfies the pre-condition to add a perturbation, we can keep track of the point $\x_{t_0}$ prior to adding perturbation, and compare the Hamiltonian at $t_0$ with the one $\utime$ steps after. If the Hamiltonian decreases by $\ufun = \tilde{\Theta}(\sqrt{\epsilon^3/\rho})$, then the algorithm has made progress, otherwise $\x_{t_0}$ is an~\ESSP~according to Lemma~\ref{lem:negHess}. Doing so will add a hyperparameter (threshold $\ufun$) but does not increase complexity.

\section{Overview of Analysis}
In this section, we will present an overview of the proof of Theorem~\ref{thm:main}. Section~\ref{sec:hamiltonian} presents the Hamiltonian for~\nag~and its key property of monotonic decrease. This leads to Section~\ref{sec:imp_local} where the \emph{improve-or-localize} lemma is stated, as well as the main intuition behind acceleration. Section~\ref{sec:framework} demonstrates how to apply these tools to prove Theorem~\ref{thm:main}.
% in the proof presents the main framework of the proof. 
Complete details can be found in the appendix.

\subsection{Hamiltonian}\label{sec:hamiltonian}
While~\gd~guarantees decrease of function value in every step (even for nonconvex problems), the biggest stumbling block to analyzing~\nag~is that it is less clear
how to keep track of ``progress.'' Known Lyapunov functions 
for~\nag~\citep{wilson2016lyapunov} are restricted to the convex setting and furthermore are not computable by the algorithm (as they depend on $\x^\star$). %, and also is only limited to convex setting.

To deepen the understanding of AGD in a nonconvex setting, we inspect it from a dynamical
systems perspective, where we fix the ratio $\tilde{\theta} = \theta / \sqrt{\eta}$ to be a 
constant, while letting $\eta \rightarrow 0$. This leads to an ODE which is the continuous 
limit of AGD~\citep{su2016differential}:
\begin{equation}\label{eq:ODE}
\ddot{\x} + \tilde{\theta}\dot{\x} + \grad f(\x) = 0,
\end{equation}
where $\ddot{\x}$ and $\dot{\x}$ are derivatives with respect to time $t$. 
This equation is a second-order dynamical equation with \emph{dissipative forces} 
$-\tilde{\theta}\dot{\x}$. Integrating both sides, we obtain:
\begin{equation}
f(\x(t_2)) + \frac{1}{2}\dot{\x}(t_2)^2 = f(\x(t_1)) + \frac{1}{2}\dot{\x}(t_1)^2 - \tilde{\theta}\int_{t_1}^{t_2}\dot{\x}(t)^2\mathrm{d}t.
\label{eq:energy_ODE}
\end{equation}

Using physical language, $f(\x)$ is a \emph{potential energy} while $\dot{\x}^2/2$ is a
\emph{kinetic energy}, and the sum is a \emph{Hamiltonian}.  The integral shows that
the Hamiltonian decreases monotonically with time $t$, and the decrease is given by 
the \emph{dissipation} term $\tilde{\theta}\int_{t_1}^{t_2}\dot{\x}(t)^2\mathrm{d}t$. 
Note that~\eqref{eq:energy_ODE} holds regardless of the convexity of $f(\cdot)$. 
This monotonic decrease of the Hamiltonian can in fact be extended to the discretized 
version of AGD when the function is convex, or mildly nonconvex:
% As AGD can be viewed as a discretization of Eq.\eqref{eq:ODE}, we can indeed also show a discretized version of its integrated form Eq.\eqref{eq:energy_ODE}:

% \begin{equation}
%     f(\x_{t+1}) + \frac{1}{2\eta}\norm{\x_{t+1} - \x_t}^2 \le f(\x_{t}) + \frac{1}{2\eta}\norm{\x_{t}- \x_{t-1}}^2 - \frac{\theta}{2\eta}\norm{\x_{t}- \x_{t-1}}^2.
% \label{eq:energy_discrete}
% \end{equation}

% \begin{lemma}\label{lem:energy_convex}
%   Assume that the function $f(\cdot)$ is $\ell$-smooth and at most $\gamma$-nonconvex i.e.,
%   \begin{align*}
%       f(\y) \geq f(\x) + \iprod{\nabla f(\x)}{\y-\x} - \frac{\gamma}{2} \norm{\y - \x}^2.
%   \end{align*}
%   then, ~\nag~(algorithm \ref{algo:AGD}) with learning rate $\eta \le \frac{1}{2\ell}$, $\theta\in [2\eta \gamma,1]$ satisfies following:
%   \begin{equation*}
%   E_{t+1} \le E_t - \frac{\theta}{2\eta}\norm{\v_t}^2 - \frac{\eta}{4}\norm{\grad f(\y_{t})}^2.
%   \end{equation*}
% \end{lemma}

\begin{lemma}[Hamiltonian decreases monotonically]\label{lem:energy_nonconvex}
  Assume that the function $f(\cdot)$ is $\ell$-smooth, the learning rate $\eta \le \frac{1}{2\ell}$, and $\theta\in [2\eta \gamma,\frac{1}{2}]$ in ~\nag~(Algorithm \ref{algo:AGD}). Then, for every iteration $t$ where~\eqref{eq:certificate} does not hold, we have:
  \begin{equation}\label{eq:energy_discrete}
  f(\x_{t+1}) + \frac{1}{2\eta}\norm{\v_{t+1}}^2 \le f(\x_{t}) + \frac{1}{2\eta}\norm{\v_{t}}^2- \frac{\theta}{2\eta}\norm{\v_t}^2 - \frac{\eta}{4}\norm{\grad f(\y_{t})}^2.
  \end{equation}
  % \begin{equation*}
  % E_{t+1} \le E_t - \frac{\theta}{2\eta}\norm{\v_t}^2 - \frac{\eta}{4}\norm{\grad f(\y_{t})}^2.
  % \end{equation*}
\end{lemma}
Denote the discrete Hamiltonian as $E_t \defeq f(\x_{t}) + \frac{1}{2\eta}\norm{\v_{t}}^2$, and note that in AGD, $\v_t = \x_{t} - \x_{t-1}$. Lemma~\ref{lem:energy_nonconvex} tolerates nonconvexity with curvature at most $\gamma = \Theta(\theta/\eta)$. Unfortunately, when the function becomes too nonconvex in certain regions (so that~\eqref{eq:certificate} holds), the analogy between the continuous and discretized versions breaks and~\eqref{eq:energy_discrete} no longer holds. In fact, standard AGD can even increase the Hamiltonian in this regime (see Appendix \ref{sec:counterex} for more details). This motivates us to modify the algorithm by adding the~\nce~step, which addresses this issue. We have the following result:

% \begin{lemma}[Hamiltonian decreases monotonically]\label{lem:energy_nonconvex}
% %Assume that $f(\cdot)$ is $\ell$-smooth and $\rho$-Hessian Lipschitz. Then, we have for Algorithm \ref{algo:PAGD} with learning rate $\eta \le \frac{1}{2\ell}$, $\theta\in [2\eta \gamma,1]$ satisfies energy function monotonically decreasing. More specifically:
% Assume that $f(\cdot)$ is $\ell$-smooth and $\rho$-Hessian Lipschitz. Then, Algorithm~\ref{algo:PAGD} without perturbation, with $\eta \le \frac{1}{2\ell}$, and $\theta\in [2\eta \gamma, \frac{1}{2}]$ 
% % monotonically decreases Hamiltonian~\eqref{eqn:hamiltonian}. More specifically:
% satisfies:
% \begin{equation*}
% E_{t+1}
% \le 
% \begin{dcases}
% E_t -\min\{\frac{s^2}{2\eta},  \frac{1}{2}(\gamma - 2\rho s) s^2\}
% & \mbox{\quad~if NCE is triggered}\\
% E_t - \frac{\theta}{2\eta}\norm{\v_t}^2 - \frac{\eta}{4}\norm{\grad f(\y_{t})}^2
% & \mbox{\quad~otherwise}.
% \end{dcases}
% \end{equation*}
% % \cnote{maybe put the exact formulation in appendix, here only says $\sqrt{\frac{\epsilon^3}{\rho}}$}
% \end{lemma}

\begin{lemma}\label{lem:energy_NCE}
Assume that $f(\cdot)$ is $\ell$-smooth and $\rho$-Hessian Lipschitz. For every iteration $t$ of Algorithm~\ref{algo:PAGD} where~\eqref{eq:certificate} holds (thus running NCE), we have:
\begin{equation*}
E_{t+1}\le E_t -\min\{\frac{s^2}{2\eta},  \frac{1}{2}(\gamma - 2\rho s) s^2\}.
\end{equation*}
% \cnote{maybe put the exact formulation in appendix, here only says $\sqrt{\frac{\epsilon^3}{\rho}}$}
\end{lemma}

Lemmas~\ref{lem:energy_nonconvex} and~\ref{lem:energy_NCE} jointly assert that the Hamiltonian decreases monotonically in all situations, and are the main tools in the proof of Theorem~\ref{thm:main}. They not only give us a way of tracking progress, but also quantitatively measure the amount of progress.

\subsection{Improve or Localize} \label{sec:imp_local}
One significant challenge in the analysis of gradient-based algorithms for nonconvex optimation is that many phenomena---for instance the accumulation of momentum and the escape from saddle points via perturbation---are multiple-step behaviors; they do not happen in each step. We address this issue by developing a general technique for analyzing the long-term behavior of such algorithms.

% We also stress that under both of these settings,~\pagd~can not achieve $\Omega(\ufun/\utime)$ decrease in each step---it has to accumulate momentum over time to achieve $\Omega(\ufun/\utime)$ amortized decrease, requiring an understanding of multiple-step behavior of~\nag. The \emph{\iol} framework (Corollary \ref{cor:localball}) is crucial here, telling us that either the Hamiltonion decreases by $\ufun$, or $\{\x_\tau\}_{\tau = t}^{t+\utime}$ is localized in a {ball} of radius $\uspace = \tilde{\Theta}(\sqrt{\epsilon/\rho})$. In this ball, we can approximate $f(\cdot)$ by a quadratic taking advantage of the Lipschitz property, and analyze the remainder as approximation error.

% When tracking the Hamiltonian over the iterates of Algorithm \ref{algo:PAGD}, previous section shows either \eqref{eq:certificate} holds and the Hamiltonian is decreased by a fixed amount due to NCE step (Lemma \ref{lem:energy_NCE}), or AGD steps are running and the progress is discribed in Lemma \ref{lem:energy_nonconvex}. 

In our case, to track the long-term behavior of AGD, one key observation from Lemma~\ref{lem:energy_nonconvex} is that the amount of progress actually relates to movement of the iterates, which leads to the following \emph{improve-or-localize} lemma:
\begin{corollary}[Improve or localize] \label{cor:localball}
Under the same setting as in Lemma \ref{lem:energy_nonconvex}, if \eqref{eq:certificate} does not hold for all steps in $[t, t+T]$, we have:
\begin{equation*}
\sum_{\tau = t+1}^{t+T}\norm{\x_\tau - \x_{\tau-1}}^2
\le \frac{2\eta}{\theta} (E_t - E_{t+T}).
% \text{~~~and~~~} \norm{\x_{t+T} -\x_t}^2 \le \frac{2\eta T}{\theta}(E_t - E_{t+T}).
\end{equation*}
\end{corollary}
% \pn{Clarify that this is the unperturbed version.}
Corollary~\ref{cor:localball} says that the algorithm either makes progress in terms of the 
Hamiltonian, or the iterates do not move much. In the second case, Corollary~\ref{cor:localball} allows us to approximate the dynamics of $\{\x_\tau\}_{\tau = t}^{t+T}$ with a \emph{quadratic approximation} of $f(\cdot)$.
% function, which is second-order Taylor expansion around $\x_t$:
% We defer the actual usage of this 
% in case there is no progress in function value (i.e., does not decrease much), also helps restrict attention to a small local region, where the function is approximately quadratic.

The acceleration phenomenon is rooted in and can be seen clearly for a quadratic, where the function can be decomposed into eigen-directions. Consider an eigen-direction with eigenvalue $\lambda$, and linear term $g$ (i.e., in this direction $f(x) =  \frac{\lambda}{2} x^2 + gx$).
The GD update becomes $x_{\tau+1} = (1-\eta\lambda)x_{\tau} - \eta g$, with $\mu_{\text{GD}}(\lambda) \defeq 1-\eta\lambda$ determining the rate of~\gd. The update of AGD is $(x_{\tau+1}, x_\tau) = (x_{\tau}, x_{\tau-1})\A\trans - (\eta g, 0)$ with matrix $\A$ defined as follows:
\begin{equation*}
\A \defeq \pmat{(2-\theta) (1 - \eta\lambda)&  -(1-\theta) (1 - \eta\lambda) \\ 1 & 0}.
\end{equation*}
The rate of~\nag~is determined by largest eigenvalue of matrix $\A$, which is denoted by $\mu_{\text{AGD}}(\lambda)$. Recall the choice of parameter~\eqref{eq:parameter}, and divide the eigen-directions into the following three categories.
\begin{itemize}
\item \textbf{Strongly convex directions $\lambda \in [\sqrt{\rho\epsilon}, \ell]$:} the slowest case is $\lambda = \sqrt{\rho\epsilon}$, where
$\mu_{\text{GD}}(\lambda) = 1- \Theta(1/\cn)$ while $\mu_{\text{AGD}}(\lambda) = 1-\Theta(1/\sqrt{\cn})$, which results in AGD converging faster than GD.
\item \textbf{Flat directions $\lambda \in [-\sqrt{\rho\epsilon}, \sqrt{\rho\epsilon}]$:}  the representative case is $\lambda = 0$ where AGD update becomes $x_{\tau+1} - x_\tau = (1-\theta) (x_\tau - x_{\tau - 1}) - \eta g$. For $\tau \le 1/\theta$, we have $|x_{t+\tau} - x_t| =  \Theta(\tau)$ for GD while  $|x_{t+\tau} - x_t| =  \Theta(\tau^2)$ for AGD, which results in AGD moving along negative gradient directions faster than GD.
\item \textbf{Strongly nonconvex directions $\lambda \in [-\ell, -\sqrt{\rho\epsilon}]$:} similar to the strongly convex case, the slowest rate is for $\lambda = -\sqrt{\rho\epsilon}$ where
$\mu_{\text{GD}}(\lambda) = 1 + \Theta(1/\cn)$ while $\mu_{\text{AGD}}(\lambda) = 1 + \Theta(1/\sqrt{\cn})$, which results in AGD escaping saddle point faster than GD.
\end{itemize}

% If negative curvature $\lambda = -\sqrt{\rho\epsilon}$, then we have $\mu_{\text{GD}} = 1+ \Theta(1/\cn)$ while $\mu_{\text{AGD}} = 1+\Theta(1/\sqrt{\cn})$, which results in~\nag~escaping saddle points faster than~\gd. 

Finally, the approximation error (from a quadratic) is also under control in this framework. With appropriate choice of $T$ and threshold for $E_t - E_{t+T}$ in Corollary \ref{cor:localball}, by the Cauchy-Swartz inequality we can restrict iterates $\{\x_\tau\}_{\tau = t}^{t+T}$ to all lie within a local ball around $\x_t$ with radius $\sqrt{\epsilon/\rho}$, where both the gradient and Hessian of $f(\cdot)$ and its quadratic approximation
$\tilde{f}_t(\x) = f(\x_t) + \la \grad f(\x_t), \x - \x_t\ra + \frac{1}{2} (\x - \x_t)\trans \hess f(\x_t) (\x - \x_t)$ are close:
\begin{fact}
Assume $f(\cdot)$ is $\rho$-Hessian Lipschitz, then for all $\x$ so that $\norm{\x - \x_t} \le \sqrt{\epsilon/\rho}$, 
we have $\|\grad f(\x) - \grad \tilde{f}_t(\x)\| \le \epsilon$ and $\| \hess f(\x) - \hess \tilde{f}_t(\x)\|
= \| \hess f(\x) - \hess f(\x_t) \|\le \sqrt{\rho\epsilon}$.
\end{fact}

%!TEX root = main.tex

\subsection{Main Framework} \label{sec:framework}
For simplicity of presentation, recall $\utime \defeq \sqrt{\cn} \cdot \chi c = \tilde{\Theta}(\sqrt{\cn})$ and denote $\ufun \defeq \sqrt{\epsilon^3/\rho}\cdot \chi^{-5}c^{-7} = \tilde{\Theta}(\sqrt{\epsilon^3/\rho})$, where $c$ is sufficiently large constant as in Theorem \ref{thm:main}. Our overall proof strategy will be to show the following ``average descent claim'':
\emph{Algorithm~\ref{algo:PAGD} decreases the Hamiltonian by $\ufun$ in every set of $\utime$ iterations as long as it does not reach an~\ESSP}.
Since the Hamiltonian cannot decrease more than $E_0 - E^\star = f(\x_0) - f^\star$, this immediately shows that it has to reach an~\ESSP~in $O((f(\x_0) - f^\star)\utime/\ufun)$ steps, proving Theorem~\ref{thm:main}.

It can be verified by the choice of parameters~\eqref{eq:parameter} and Lemma~\ref{lem:energy_nonconvex} that whenever \eqref{eq:certificate} holds so that NCE is triggered, the Hamiltonian decreases by at least $\ufun$ in one step.
So, if~\nce~step is performed even once in each round of $\utime$ steps, we achieve enough average decrease. The troublesome case is when in some time interval of $\utime$ steps starting with $\x_t$, only AGD steps are performed without NCE.
If $\x_t$ is not an $\epsilon$-second order stationary point, either the gradient is large or the Hessian has a large negative direction. We prove the average decrease claim by considering these two cases.

\begin{lemma}[Large gradient]\label{lem:largeGrad}
Consider the setting of Theorem~\ref{thm:main}.
%If in~\pagd~(Algorithm~\ref{algo:PAGD}),
If $\norm{\grad f(\x_\tau)} \ge \epsilon$ for all $ \tau \in [t, t+\utime]$, then by running Algorithm \ref{algo:PAGD} we have $E_{t+\utime} - E_t \le -\ufun$.
\end{lemma}
% \noindent
% The following lemma captures the case where Hessian has a large negative direction.
\begin{lemma}[Negative curvature]\label{lem:negHess}
Consider the setting of Theorem~\ref{thm:main}. 
If $\norm{\grad f(\x_t)} \le \epsilon$, $\lambda_{\min} (\hess f(\x_t)) < -\sqrt{\rho\epsilon}$, 
and perturbation has not been added in iterations $\tau \in [t-\utime, t)$, then by running Algorithm \ref{algo:PAGD}, we have $E_{t+\utime} - E_t \le -\ufun$ with high probability.
\end{lemma}

We note that an important aspect of these two lemmas is that the Hamiltonian decreases by $\Omega(\ufun)$ in $\utime = \tilde{\Theta}(\sqrt{\cn})$ steps, which is faster compared to~\pgd~which decreases the function value by $\Omega(\ufun)$ in $\utime^2 = \tilde{\Theta}(\cn)$ steps~\citep{jin2017escape}.  That is, the acceleration phenomenon in~\pagd~happens in both cases.
We also stress that under both of these settings,~\pagd~cannot achieve $\Omega(\ufun/\utime)$ decrease in each step---it has to accumulate momentum over time to achieve $\Omega(\ufun/\utime)$ amortized decrease.

% requiring an understanding of multiple-step behavior of~\nag. The \emph{\iol} framework (Corollary \ref{cor:localball}) is crucial here, telling us that either the Hamiltonion decreases by $\ufun$, or $\{\x_\tau\}_{\tau = t}^{t+\utime}$ is localized in a {ball} of radius $\uspace = \tilde{\Theta}(\sqrt{\epsilon/\rho})$. In this ball, we can approximate $f(\cdot)$ by a quadratic taking advantage of the Lipschitz property, and analyze the remainder as approximation error.

% Our
% analysis heavily  
% In the case of exact quadratic function, we can decouple the function according to each eigendirection of Hessian. Suppose for along one eigendirection, function can be written as $\lambda x^2 /2$ (with eigenvalue lambda) gradient 

%  In following, we will show the reason seperately.

% The proof of Lemma~\ref{lem:negHess} follows by combining estimates of eigenvalues of~\nag~operator matrix with coupling techniques introduced in~\cite{jin2017escape}, in order to show escape from saddle points. \cnote{after calculating the eigenvalues it is true}.
% The proof of Lemma~\ref{lem:largeGrad} on the other hand, is technically quite challenging. We now give a brief sketch of the proofs of Lemmas~\ref{lem:largeGrad} and~\ref{lem:negHess}.

% \subsubsection{Proof overview of Lemma~\ref{lem:largeGrad}}
\subsubsection{Large Gradient Scenario}
For AGD, gradient and momentum interact, and both play important roles in the dynamics. Fortunately, according to Lemma~\ref{lem:energy_nonconvex}, the Hamiltonian decreases sufficiently whenever the momentum $\v_t$ is large; so it is sufficient to discuss the case where the momentum is small.
%According to Hamiltonian monotonically decreasing lemma (Lemma \ref{lem:energy_nonconvex}), whenever momentum $\v_t$ is large, Hamiltonian will have sufficiently fast decrease, we focus on cases where momentum is small.

One difficulty in proving Lemma \ref{lem:largeGrad} lies in the difficulty of enforcing the precondition that gradients of all iterates are large even with quadratic approximation. Intuitively we hope that the large initial gradient $\norm{\grad f(\x_t)}\ge \epsilon$ suffices to give a sufficient decrease of the Hamiltonian. Unfortunately, this is not true. Let $\S$ be the subspace of eigenvectors of $\nabla^2 f (\x_t)$ with eigenvalues in $[\sqrt{\rho\epsilon}, \ell]$, consisting of all the strongly convex directions, and let $\S^c$ be the orthogonal subspace. It turns out that the initial gradient component in $\S$ is not very helpful in decreasing the Hamiltonian since~\nag~rapidly decreases the gradient in these directions. %due to large positive curvature, which 
We instead prove Lemma~\ref{lem:largeGrad} in two steps.

\begin{lemma}(informal)\label{lem:1}
    If $\v_t$ is small, $\norm{\grad f(\x_t)}$ not too large and $E_{t+\utime/2} - E_t \ge - \ufun$, then for all $\tau\in [t+\utime/4, t+\utime/2]$ we have $\norm{\proj_\S \grad f(\x_\tau)}\le \epsilon/2$.
    % \begin{equation*}
    % \norm{\proj_\S\grad f(\x_{t})} \le \frac{\epsilon}{10}
    % \text{~~and~~}
    % % \norm{\proj_\S(\x_t - \x_{t-1})} \le \frac{\epsilon}{\ell}
    % \v_t\trans [\proj_\S\trans \hess f(\x_0) \proj_\S] \v_t \le \frac{\epsilon^2}{10\ell}.
    % \end{equation*}
\end{lemma}

\begin{lemma}(informal)\label{lem:2}
% Under the setting of Theorem \ref{thm:main}, if $\norm{\proj_{\S^c}\grad f(\x_{0})} \ge \frac{\epsilon}{2}$, $\norm{\v_0} \le \umom$, $\v_0\trans [\proj_{\S}\trans\hess f(\x_0) \proj_{\S}] \v_0 \le  2 \sqrt{\rho\epsilon}\umom^2 $,
% and in $t\in [0, \utime/4]$ only AGD steps are used without NCE or perturbation,
% then:
% \begin{equation*}
% E_{\utime/4} - E_0 \le - \ufun
% \end{equation*}
If $\v_t$ is small and $\norm{\proj_{\S^c}\grad f(\x_{t})} \ge \epsilon/2$,
then we have $E_{t+\utime/4} - E_t \le - \ufun.$
\end{lemma}
\noindent See the formal versions, Lemma \ref{lem:largegrad_nonconvex} and Lemma \ref{lem:largegrad_convex}, 
for more details.  We see that if the Hamiltonian does not decrease much (and so is localized in a small ball), 
the gradient in the strongly convex subspace $\norm{\proj_\S \grad f(\x_\tau)}$ vanishes in $\utime/4$ steps by Lemma~\ref{lem:1}. Since the hypothesis of Lemma~\ref{lem:largeGrad} guarantees a large gradient for all of the $\utime$ steps, this means that $\norm{\proj_{\S^c}\grad f(\x_{t})}$ is large after $\utime/4$ steps, thereby decreasing the Hamiltonian in the next $\utime/4$ steps (by Lemma~\ref{lem:2}).

\subsubsection{Negative Curvature Scenario}
% Denote $\uspace \defeq \sqrt{\frac{2\eta \utime\ufun}{\theta}} = \tilde{\Theta}(\sqrt{\frac{\epsilon}{\rho}})$
In this section, we will show that the volume of the set around a strict saddle point from which AGD does not escape quickly is very small (Lemma~\ref{lem:negHess}).
%In this section, we will show that the volume of points where~\pagd~does not decrease Hamiltonian, and hence might not escape the saddle point, is tiny.
We do this using the coupling mechanism introduced in~\cite{jin2017escape}, which gives a fine-grained understanding of the geometry around saddle points.
More concretely, letting the perturbation radius $r = \tilde{\Theta}(\epsilon/\ell)$ as specified in \eqref{eq:parameter}, we show the following lemma.
\begin{lemma} (informal) \label{lem:informal_neg_curve}
Suppose $\norm{\nabla f(\tilde{\x})} \le \epsilon$ and $\lambda_{\min}(\hess f(\tilde{\x})) \le - \sqrt{\rho\epsilon}$. Let $\x_0,  \modify{\x}_0$ be at distance at most $r$ from $\tilde{\x}$, and $\x_0 -  \modify{\x}_0 = r_0 \e_1$ where $\e_1$ is the minimum eigen-direction of $\hess f(\tilde{\x})$ and $r_0 \ge \delta r /\sqrt{d}$. Then for~\nag~starting at $(\x_0, \v)$ and $(\x_0', \v)$, we have:
\begin{align*}
\min\{E_{\utime} - \widetilde{E}, \modify{E}_{\utime} - \widetilde{E}\} \le - \ufun,
\end{align*}
where $\widetilde{E},E_{\utime}$ and $\modify{E}_{\utime}$ are the Hamiltonians at $(\tilde{\x}, \v), (\x_{\utime}, \v_{\utime})$ and $(\modify{\x}_{\utime}, \modify{\v}_{\utime})$ respectively.
\end{lemma}
\noindent See the formal version in Lemma \ref{lem:2nd_seq}. We note $\delta$ in above Lemma is a small number characterize the failure probability of the algorithm (as defined in Theorem \ref{thm:main}), and $\utime$ has logarithmic dependence on $\delta$ according to \eqref{eq:parameter}.
Lemma~\ref{lem:informal_neg_curve} says that around any strict saddle, for any two points that are separated along the smallest eigen-direction by at least $\delta r /\sqrt{d}$,~\pagd, starting from at least one of those points, decreases the Hamiltonian, and hence escapes the strict saddle. This implies that the width of the region starting from where~\nag~is stuck has width at most $\delta r /\sqrt{d}$, and thus has small volume.

% The proof of this lemma is by tracking the difference between the evolutions of~\pagd~at both $\x_{0}$ and $\x_{0}'$, and showing that this difference i.e., $\norm{\x_{\utime} - \x_{\utime}'}$ is large. This means that either $\norm{\x_{\utime}-\x_{0}}$ is large or $\norm{\x_{\utime}'-\x_{0}'}$ is large. This then implies that $\min\{E_{\utime} - E_0, \modify{E}_{\utime} - \modify{E}_0\} \le - \Omega(\ufun)$.
% \cnote{Talk about two sequence techniques}

%!TEX root = main.tex

\section{Conclusions}
In this paper, we show that a variant of~\nag~can escape saddle points faster than~\gd, demonstrating that momentum techniques can indeed accelerate convergence even for nonconvex optimization. Our algorithm finds an $\epsilon$-second order stationary point in $\otilde{1/\epsilon^{7/4}}$ iterations, faster than the $\otilde{1/\epsilon^2}$ iterations taken by~\gd. This is the first algorithm that is both Hessian-free and single-loop that achieves this rate. Our analysis relies on novel techniques that lead to a better understanding of momentum techniques as well as nonconvex optimization.

The results here also give rise to several questions. The first concerns lower bounds;
is the rate of $\otilde{1/\epsilon^{7/4}}$ that we have established here optimal for 
gradient-based methods under the setting of gradient and Hessian-Lipschitz? 
We believe this upper bound is very likely sharp up to log factors, and developing 
a tight algorithm-independent lower bound will be necessary to settle this question.
The second is whether the negative-curvature-exploitation component of our algorithm 
is actually necessary for the fast rate. To attempt to answer this question, we may 
either explore other ways to track the progress of standard AGD (other than the 
particular Hamiltonian that we have presented here), or consider other discretizations
of the ODE \eqref{eq:ODE} so that the property \eqref{eq:energy_ODE} is preserved 
even for the most nonconvex region.  A final direction for future research is the 
extension of our results to the finite-sum setting and the stochastic setting.

%%-------------------------
%%
%%Discuss 
%%
%%1. the case without Hessian Lipschitz, no acceleration
%%\pn{We say first order stationary points not important. So may be let's not talk about this?}
%%
%%2. Why we believe $\epsilon^{7/4}$ is probably the best we can do.
%%
%%3. NCE might be related to why people need to set large momentum parameter in practice.\pn{For large negative curvature, momentum will be away from $1$ for Hamiltonian right?}
%%
%%Future direction: is NCE necessary? only reset moments?
%
%-----------------------------------
%

% \newpage

\bibliographystyle{plainnat}
\bibliography{saddle}

\newpage

\appendix
%!TEX root = main.tex

\section{Proof of Hamiltonian Lemmas}
In this section, we prove Lemma \ref{lem:energy_nonconvex}, Lemma \ref{lem:energy_NCE} and Corollary \ref{cor:localball}, which are presented in Section \ref{sec:hamiltonian} and 
Section \ref{sec:imp_local}.  In section \ref{sec:counterex} we also give an example 
where standard AGD with negative curvature exploitation can increase the Hamiltonian.

Recall that we define the Hamiltonian as $E_t \defeq f(\x_{t}) + \frac{1}{2\eta}\norm{\v_{t}}^2$, where, for AGD, we define $\v_t = \x_t - \x_{t-1}$.
The first lemma shows that this Hamiltonian decreases in every step of~\nag~for mildly nonconvex functions.

\begingroup
\def\thetheorem{\ref{lem:energy_nonconvex}}
\begin{lemma}[Hamiltonian decreases monotonically]
  Assume that the function $f(\cdot)$ is $\ell$-smooth and set the learning rate to
be $\eta \le \frac{1}{2\ell}$, $\theta\in [2\eta \gamma,\frac{1}{2}]$ in ~\nag~(Algorithm \ref{algo:AGD}). Then, for every iteration $t$ where~\eqref{eq:certificate} does not hold, we have:
  \begin{equation*}
  E_{t+1} \le E_t - \frac{\theta}{2\eta}\norm{\v_t}^2 - \frac{\eta}{4}\norm{\grad f(\y_{t})}^2.
  \end{equation*}
\end{lemma}
\addtocounter{theorem}{-1}
\endgroup

\begin{proof}
Recall that the update equation of accelerated gradient descent has following form:
\begin{align*}
\x_{t+1} &\leftarrow \y_t - \eta \grad f (\y_t) \\
\y_{t+1} &\leftarrow \x_{t+1} + (1-\theta) (\x_{t+1} - \x_t).
\end{align*}
By smoothness, with $\eta \le \frac{1}{2\ell}$:
\begin{align}
f(\x_{t+1}) \le f(\y_t) - \eta\norm{\grad f(\y_t)}^2 + \frac{\ell\eta^2}{2}\norm{\grad f(\y_t)}^2
\le f(\y_t) - \frac{3\eta}{4}\norm{\grad f(\y_t)}^2, \label{eq:energy_smooth}
\end{align}
assuming that the precondition \eqref{eq:certificate} does not hold:
\begin{align}
f(\x_t) \ge f(\y_t) + \la \grad f(\y_t), \x_t - \y_t\ra -
\frac{\gamma}{2}\norm{\y_t - \x_t}^2,
\label{eq:energy_almost_convex}
\end{align}
and given the following update equation:
\begin{align}
\norm{\x_{t+1} - \x_t}^2
 =& \norm{\y_t - \x_t  - \eta \grad f (\y_t)}^2  \nn\\
 =& \left[(1-\theta)^2\norm{\x_t - \x_{t-1}}^2
 - 2\eta\la \grad f(\y_t), \y_t - \x_t \ra
 + \eta^2 \norm{\grad f(\y_t)}^2 \right], \label{eq:energy_momentum}
\end{align}
% Adding up Eq.\eqref{eq:energy_smooth} and \eqref{eq:energy_convex}, we have:
% \begin{equation}
% f(\x_{t+1}) \le  f(\x_t) + \la \grad f(\y_t), \y_t - \x_t \ra + \frac{\ell}{2} \norm{\y_t - \x_t}^2 - \frac{\eta}{2}\norm{\grad f(\y_t)}^2 \label{eq:momentum} 
% \end{equation}
we have:
\begin{align*}
f(\x_{t+1})
+ \frac{1}{2\eta}\norm{\x_{t+1} - \x_{t}}^2
\le& f(\x_t) + \la \grad f(\y_t), \y_t - \x_t \ra - \frac{3\eta}{4}\norm{\grad f(\y_t)}^2 \\
&+\frac{1+ \eta\gamma}{2\eta}(1-\theta)^2 \norm{\x_t - \x_{t-1}}^2 - \la \grad f(\y_t), \y_t - \x_t \ra
 + \frac{\eta}{2} \norm{\grad f(\y_t)}^2\\
\le& f(\x_{t}) + \frac{1}{2\eta}\norm{\x_{t} - \x_{t-1}}^2
 - \frac{2\theta-\theta^2 - \eta\gamma(1-\theta)^2}{2\eta}\norm{\v_{t}}^2 - \frac{\eta}{4}\norm{\grad f(\y_{t})}^2 \\
\le& f(\x_{t}) + \frac{1}{2\eta}\norm{\x_{t} - \x_{t-1}}^2 - \frac{\theta}{2\eta}\norm{\v_t}^2 - \frac{\eta}{4}\norm{\grad f(\y_{t})}^2.
\end{align*}
The last inequality uses the fact that $\theta \in [2\eta \gamma, \frac{1}{2}]$ 
so that $\theta^2 \le \frac{\theta}{2}$ and $\eta\gamma \le \frac{\theta}{2}$. 
We substitute in the definition of $\v_t$ and $E_t$ to finish the proof.
% \begin{align*}
% E_{t+1} \le& E_t
% - \frac{2\theta-\theta^2 - \eta\gamma(1-\theta)^2}{2\eta}\norm{\v_{t}}^2 - \frac{\eta}{4}\norm{\grad f(\y_{t})}^2 \\
% \le &  E_t
% - \frac{\theta}{2\eta}\norm{\v_t}^2 - \frac{\eta}{4}\norm{\grad f(\y_{t})}^2
% \end{align*}
\end{proof}
We see from this proof that~\eqref{eq:energy_almost_convex} relies on approximate convexity of $f(\cdot)$, which explains why in all existing proofs, the convexity between $\x_t$ and $\y_t$ is so important. A perhaps surprising fact to note is that the above proof can in fact go through even with mild nonconvexity (captured in line $8$ of Algorithm~\ref{algo:PAGD}).
Thus, high nonconvexity is the problematic situation.
%Nestrov's AGD is designed to work for convex case. In the nonconvex setting,
To overcome this, we need to slightly modify AGD so that the Hamiltonian is
decreasing. This is formalized in the following lemma.

\begingroup
\def\thetheorem{\ref{lem:energy_NCE}}
\begin{lemma}
Assume that $f(\cdot)$ is $\ell$-smooth and $\rho$-Hessian Lipschitz. For every iteration $t$ of Algorithm~\ref{algo:PAGD} where~\eqref{eq:certificate} holds (thus running NCE), we have:
\begin{equation*}
E_{t+1}\le E_t -\min\{\frac{s^2}{2\eta},  \frac{1}{2}(\gamma - 2\rho s) s^2\}.
\end{equation*}
\end{lemma}
\addtocounter{theorem}{-1}
\endgroup

\begin{proof}
When we perform an NCE step, we know that \eqref{eq:certificate} holds. In the first case ($\norm{\v_t} \ge s$), we set $\x_{t+1}  = \x_t$ and set the momentum $\v_{t+1}$ to zero, which gives:
\begin{align*}
E_{t+1} = f(\x_{t+1}) = f(\x_t) = E_{t} - \frac{1}{2\eta}\norm{\v_t}^2
\le E_{t} - \frac{s^2}{2\eta}.
\end{align*}
In the second case ($\norm{\v_t} \le s$), expanding in a Taylor series with Lagrange remainder, we have:
\begin{equation*}
f(\x_t) =  f(\y_t) + \la \grad f(\y_t), \x_t - \y_t \ra + \frac{1}{2} (\x_t - \y_t)\trans 
\hess f(\zeta_t) (\x_t - \y_t),
\end{equation*}
where $\zeta_t = \phi\x_t + (1-\phi)\y_t$ and $\phi \in [0, 1]$. Due to the certificate \eqref{eq:certificate} we have
\begin{equation*}
\frac{1}{2} (\x_t - \y_t)\trans 
\hess f(\zeta_t) (\x_t - \y_t) \le - \frac{\gamma}{2} \norm{\x_t - \y_t}^2.
\end{equation*}
On the other hand, clearly $\min\{\la \grad f(\x_t), \delta\ra, \la \grad f(\x_t), -\delta\ra  \} \le 0$. WLOG, suppose $\la \grad f(\x_t), \delta\ra \le 0$,
then, by definition of $\x_{t+1}$, we have:
\begin{equation*}
f(\x_{t+1}) \le f(\x_{t} + \delta) 
= f(\x_t) + \la \grad f(\x_t), \delta\ra + \frac{1}{2} \delta\trans \hess f(\zeta'_t) \delta
\le  f(\x_t) + \frac{1}{2} \delta\trans 
\hess f(\zeta'_t) \delta,
\end{equation*}
where $\zeta'_t = \x_t + \phi'\delta$ and $\phi' \in [0, 1]$. 
Since $\norm{\zeta_t - \zeta'_t} \le 2s$, $\delta$ also lines up
with $\y_t - \x_t$:
\begin{equation*}
\delta\trans 
\hess f(\zeta'_t) \delta
\le \delta\trans 
\hess f(\zeta_t) \delta
+ \norm{\hess f(\zeta'_t) - \hess f(\zeta_t) }\norm{\delta}^2
\le - \gamma \norm{\delta}^2 + 2\rho s\norm{\delta}^2.
\end{equation*}
Therefore, this gives
\begin{align*}
E_{t+1} = f(\x_{t+1}) \le f(\x_t) 
- \frac{1}{2}(\gamma - \rho s) s^2
\le E_{t} - \frac{1}{2}(\gamma - 2\rho s) s^2,
\end{align*}
which finishes the proof.
\end{proof}

The Hamiltonian decrease has an important consequence: if the Hamiltonian does not decrease much, then all the iterates are localized in a small ball around the starting point. Moreover, the iterates do not oscillate much in this ball. We called this the improve-or-localize phenomenon.

\begingroup
\def\thetheorem{\ref{cor:localball}}
\begin{corollary}[Improve or localize]
Under the same setting as in Lemma \ref{lem:energy_nonconvex}, if \eqref{eq:certificate} does not hold for all steps in $[t, t+T]$, we have:
\begin{equation*}
\sum_{\tau = t+1}^{t+T}\norm{\x_\tau - \x_{\tau-1}}^2
\le \frac{2\eta}{\theta} (E_t - E_{t+T}).
% \text{~~~and~~~} \norm{\x_{t+T} -\x_t}^2 \le \frac{2\eta T}{\theta}(E_t - E_{t+T}).
\end{equation*}
\end{corollary}
% \begin{corollary}
% Under the same setting as in Lemma \ref{lem:energy_nonconvex}, if NCE is not triggered in step $[t, t+T]$, then:
% \begin{equation*}
% \sum_{\tau = t+1}^{t+T}\norm{\x_\tau - \x_{\tau-1}}^2
% \le \frac{2\eta}{\theta} (E_t - E_{t+T})
% \text{~~~and~~~} \norm{\x_{t+T} -\x_t}^2 \le \frac{2\eta T}{\theta}(E_t - E_{t+T})
% \end{equation*}
% \end{corollary}
\addtocounter{theorem}{-1}
\endgroup

\begin{proof}
The proof follows immediately from telescoping the argument of Lemma \ref{lem:energy_nonconvex}.
% The proof of second inequality follows from triangle inequality and Cauchy-Swartz to first inequality as follows:
% \begin{equation*}
% \norm{\x_{t+T} -\x_t}^2 \le \left(\sum_{\tau = t+1}^{t+T}\norm{\x_\tau - \x_{\tau-1}}\right)^2
% \le T \left(\sum_{\tau = t+1}^{t+T}\norm{\x_\tau - \x_{\tau-1}}^2\right)
% \le \frac{2\eta T}{\theta}(E_t - E_{t+T})
% \end{equation*}
\end{proof}

% \begin{lemma}\label{lem:localball}
% For function $f$ satisfy assumption 1, 2, for $T \le O(\frac{1}{\theta})$, assume
% \begin{equation*}
% E_{t+T} - E_t \ge - O(c\sqrt{\frac{\epsilon^3}{\rho}})
% \end{equation*}
% then, we have:
% \begin{equation*}
% \norm{\x_{t+T} -\x_t} \le O(\sqrt{\frac{c\epsilon}{\rho}})
% \quad \text{~and~} \quad
% \sum_{\tau=t+1}^{t+T} \norm{\x_\tau - \x_{\tau-1}}^2
% \le \frac{\eta}{\theta}O(c\sqrt{\frac{\epsilon^3}{\rho}})
% \end{equation*}
% \end{lemma}
% \begin{proof}
% Cleary NEC must have not been triggered (otherwise, decrease enough energy).
% Therefore, by theorem \ref{thm:energy_nonconvex}, we know the cumulative energy dissipation:
% \begin{align*}
% \sum_{\tau=t+1}^{t+T} \frac{\theta}{2\eta}\norm{\v_\tau}^2 
% =\sum_{\tau=t+1}^{t+T} \frac{\theta}{2\eta}\norm{\x_\tau - \x_{\tau-1}}^2
% \le O(c\sqrt{\frac{\epsilon^3}{\rho}})
% \end{align*}
% By Cauchy-Swartz, we have:
% \begin{align*}
% \sum_{\tau=t+1}^{t+T} \norm{\x_\tau - \x_{\tau+1}} 
% \le \sqrt{T}\sqrt{\sum_{\tau=t+1}^{t+T}\norm{\x_\tau - \x_{\tau-1}}^2}
% \le O(\left(\frac{\eta}{\theta^2} c\sqrt{\frac{\epsilon^3}{\rho}}\right)^{\frac{1}{2}})
% = O(\sqrt{\frac{c\epsilon}{\rho}})
% \end{align*}
% \end{proof}

\subsection{AGD can increase the Hamiltonian under nonconvexity}
\label{sec:counterex}
In the previous section, we proved Lemma \ref{lem:energy_nonconvex} which requires $\theta \ge 2\eta\gamma$, that is, $\gamma \le \theta/(2\eta)$.
In this section, we show Lemma \ref{lem:energy_nonconvex} is almost tight in the sense 
that when $\gamma \ge 4\theta/\eta$ in \eqref{eq:certificate}, we have:
\begin{equation*}
f(\x_t) \le  f(\y_t) + \la \grad f(\y_t), \x_t - \y_t \ra - \frac{\gamma}{2} \norm{\x_t - \y_t}^2.
\end{equation*}
Monotonic decrease of the Hamiltonian may no longer hold, indeed, AGD can increase the Hamiltonian for those steps.

Consider a simple one-dimensional example, $f(x) = -\frac{1}{2}\gamma x^2$, where \eqref{eq:certificate} always holds. Define the initial condition $x_0 = -1, v_0 = 1/(1-\theta)$. By update equation in Algorithm \ref{algo:AGD}, the next iterate will be $x_1 = y_0 = 0$, and $v_1 = x_1 - x_0 = 1$. By the definition of Hamiltonian, we have：
\begin{align*}
E_0 =& f(x_0) + \frac{1}{2\eta}|v_0|^2 = -\frac{\gamma}{2} + \frac{1}{2\eta (1-\theta)^2}\\
E_1 =& f(x_1) + \frac{1}{2\eta}|v_1|^2 = \frac{1}{2\eta},
\end{align*}
since $\theta \le 1/4$. It is not hard to verify that whenever $\gamma \ge 4 \theta/\eta$, we will have $E_1 \ge E_0$; that is, the Hamiltonian increases in this step.

This fact implies that when we pick a large learning rate $\eta$ and small momentum parameter $\theta$ (both are essential for acceleration), standard AGD does not decrease the Hamiltonian in a very nonconvex region. We need another mechanism such as NCE to fix the monotonically decreasing property.

%!TEX root = main.tex

\section{Proof of Main Result}
In this section, we set up the machinery needed to prove our main result, Theorem \ref{thm:main}. We first present the generic setup, then, as in Section \ref{sec:framework}, we split the proof into two cases, one where gradient is large and the other where the Hessian has negative curvature. In the end, we put everything together and prove Theorem \ref{thm:main}.

To simplify the proof, we introduce some notation for this section, and state a convention regarding absolute constants. Recall the choice of parameters in Eq.\eqref{eq:parameter}:
\begin{equation*}
\eta = \frac{1}{4\ell}, \quad
\theta = \frac{1}{4\sqrt{\cn}},
\quad \gamma = \frac{\theta^2}{\eta} = \frac{\sqrt{\rho\epsilon}}{4} ,
\quad s = \frac{\gamma}{4\rho} = \frac{1}{16}\sqrt{\frac{\epsilon}{\rho}}, 
\quad r = \eta\epsilon\cdot \chi^{-5}c^{-8},
\end{equation*}
where $\cn = \frac{\ell}{\sqrt{\rho\epsilon}}, \chi = \max\{1,  \log \frac{d \ell\Delta_f}{\rho \epsilon\delta}\}$, and $c$ is a sufficiently large constant as stated in the precondition of Theorem \ref{thm:main}.
Throughout this section, we also always denote
\begin{equation*}
\utime \defeq  \sqrt{\cn} \cdot \chi c,  \quad
\ufun \defeq \sqrt{\frac{\epsilon^3}{\rho}}\cdot \chi^{-5}c^{-7}, \quad \uspace \defeq \sqrt{\frac{2\eta \utime\ufun}{\theta}} = \sqrt{\frac{2\epsilon}{\rho}} \cdot \chi^{-2}c^{-3}, \quad
\umom \defeq \frac{\epsilon \sqrt{\cn}}{\ell} c^{-1},
\end{equation*}
which represent the special units for time, the Hamiltonian, the parameter space and the momentum.
All the lemmas in this section hold when the constant $c$ is picked to be sufficiently large. To avoid ambiguity, throughout this section $O(\cdot), \Omega(\cdot), \Theta(\cdot)$ notation \textbf{only hides an absolute constant which is independent of the choice of sufficiently large constant $c$}, which is defined in the precondition of Theorem \ref{thm:main}. That is, we will always make $c$ dependence explicit in $O(\cdot), \Omega(\cdot), \Theta(\cdot)$ notation. Therefore, for a quantity like $O(c^{-1})$, we can always pick $c$ large enough so that it cancels out the absolute constant in the $O(\cdot)$ notation, and make $O(c^{-1})$ smaller than any fixed required constant.

% in this section, 

% We choose our parameter:
% \begin{equation*}
% \eta = \frac{1}{2\ell}, \quad
% \theta = \frac{1}{4\sqrt{\cn}},
% \quad \gamma = \frac{\theta^2}{\eta} = \Theta(\sqrt{\rho\epsilon}),
% \quad s = \frac{\gamma}{4\rho} = \Theta(\sqrt{\frac{\epsilon}{\rho}})
% \end{equation*}

% We aim to prove 
% which then finishes the proof. Above argument shows whenever NCE is triggered, we can allow in $\utime$ no progress been made, but Hamiltonian still has enough decrease. For the remaining part we focus on the case NCE is not triggered, which is the behavoir of AGD.

\subsection{Common setup}
Our general strategy in the proof is to show that if none of the iterates $\x_t$ is a SOSP, then in all $\utime$ steps, the Hamiltonian always decreases by at least $\ufun$. This gives an average decrease of $\ufun/\utime$. In this section, we establish some facts which will be used throughout the entire proof, including the decrease of the Hamiltonian in NCE step, the update of AGD in matrix form, and upper bounds on approximation error for a local quadratic approximation.

The first lemma shows if negative curvature exploitation is used, then in a single step, the Hamiltonian will decrease by $\ufun$.
\begin{lemma} \label{lem:NCE_decrease}
Under the same setting as Theorem \ref{thm:main}, for every iteration $t$ of Algorithm~\ref{algo:PAGD} where~\eqref{eq:certificate} holds (thus running NCE), we have:
\begin{equation*}
E_{t+1} - E_t\le  - 2\ufun.
\end{equation*}
\end{lemma}
\begin{proof}
It is also easy to check that the precondition of Lemma \ref{lem:energy_NCE} holds, and by the particular choice of parameters in Theorem \ref{thm:main}, we have:
\begin{equation*}
\min\{\frac{s^2}{2\eta},  \frac{1}{2}(\gamma - 2\rho s) s^2\} \ge \Omega(\ufun c^{7})\ge 2\ufun,
\end{equation*}
where the last inequality is by picking $c$ in Theorem \ref{thm:main} large enough, which finishes the proof.
\end{proof}

Therefore, whenever NCE is called, the decrease of the Hamiltonian is already sufficient. We thus only need to focus on AGD steps. The next lemma derives a general expression for $\x_t$ after an AGD update, which is very useful in multiple-step analysis. 
The general form is expressed with respect to a reference point $\zero$, which can be any arbitrary point (in many cases we choose it to be $\x_0$).

\begin{lemma}\label{lem:AGD_update_general} 
Let $\zero$ be an origin (which can be fixed at an arbitrary point).  Let $\H = \hess f(\zero)$.  Then an AGD (Algorithm \ref{algo:AGD}) update can be written as:
\begin{equation}
\begin{pmatrix}
\x_{t+1} \\ \x_t
\end{pmatrix}
= \A^{t}\begin{pmatrix}
\x_1 \\ \x_0
\end{pmatrix} 
-\eta\sum_{\tau = 1}^t \A^{t-\tau}
\begin{pmatrix}
\grad f(\zero) + \delta_\tau \\ 0
\end{pmatrix},
\label{eq:update_AGD_matrix}
\end{equation}
where $\delta_{\tau} = \grad f(\y_{\tau}) - \grad f(\zero) - \H \y_{\tau}$, and 
$$\A = \begin{pmatrix}
(2-\theta) (\I - \eta \H)&  -(1-\theta) (\I - \eta \H) \\
\I& 0
\end{pmatrix}.$$
\end{lemma}

\begin{proof}
Substituting for $(\y_t, \v_t)$ in Algorithm \ref{algo:AGD}, we have a recursive equation for $\x_t$:
\begin{equation} \label{eq:update_AGD}
\x_{t+1}  = (2-\theta) \x_{t}
- (1-\theta)  \x_{t-1}
 -\eta \grad f((2-\theta) \x_{t}
- (1-\theta)  \x_{t-1}).
\end{equation}
By definition of $\delta_\tau$, we also have:
\begin{equation*}
\grad f(\y_\tau) = \grad f(\zero) + \H \y_\tau + \delta_\tau.
\end{equation*}
Therefore, in matrix form, we have:
\begin{align*}
\begin{pmatrix}
\x_{t+1} \\ \x_t
\end{pmatrix}
=& 
\begin{pmatrix}
(2-\theta) (\I - \eta \H)&  -(1-\theta) (\I - \eta \H) \\
\I& 0
\end{pmatrix}
\begin{pmatrix}
\x_{t} \\ \x_{t-1}
\end{pmatrix}  - \eta
\begin{pmatrix}
\grad f(0) + \delta_t \\ 0
\end{pmatrix} \\
=& \A^{t}\begin{pmatrix}
\x_1 \\ \x_0
\end{pmatrix} 
-\eta\sum_{\tau = 1}^t \A^{t-\tau}
\begin{pmatrix}
\grad f(0) + \delta_\tau \\ 0
\end{pmatrix},
% \label{eq:update_AGD_matrix}
\end{align*}
which finishes the proof.
\end{proof}
Clearly $\A$ in Lemma \ref{lem:AGD_update_general} is a $2d \times 2d$ matrix, and if we expand $\A$ according to the eigenvector directions of $\begin{pmatrix}
\H& 0 \\
0 & \H
\end{pmatrix}$, $\A$ can be reorganized as a block-diagonal matrix consisting of $d$ $2\times 2$ matrices. Let the $j$th eigenvalue of $\H$ be denoted $\lambda_j$, and denote $\A_j$ as the $j$th $2\times 2$ matrix with corresponding eigendirections:
\begin{equation}\label{eq:definition_Aj}
\A_j = \begin{pmatrix}
(2-\theta) (1 - \eta \lambda_j)&  -(1-\theta) (1 - \eta \lambda_j) \\
1 & 0\end{pmatrix}.
\end{equation}
We note that the choice of reference point $\zero$ is mainly to simplify
mathmatical expressions involving $\x_t - \zero$. 

Lemma \ref{lem:AGD_update_general} can be viewed as update from a quadratic expansion around origin $\zero$, and $\delta_\tau$ is the approximation error which marks the difference between true function and its quadratic approximation.
The next lemma shows that when sequence $\x_0, \cdots, \x_t$ are all close to $\zero$, then the approximation error is under control:

% Then we establish useful bounds on approximation error for AGD when treated function as quadratic in local region. All the formula established in this section will be used for both large gradient case and negative curvature case.

% Fix a origin point $0$ (it can be $\x_0$ or not), and denote $\H = \hess f(0)$. This gives
% \begin{equation*}
% \grad f(\y_t) = \grad f(0) + (\H + \Delta_t) \y_t
% \end{equation*}
% where $\Delta_t = \int_0^1 (\hess f(\phi \y_t) - \H) \mathrm{d} \phi$.
% Therefore, in matrix form, we can have:
% \begin{align}
% \begin{pmatrix}
% \x_{t+1} \\ \x_t
% \end{pmatrix}
% =& 
% \begin{pmatrix}
% (2-\theta) (\I - \eta \H)&  -(1-\theta) (\I - \eta \H) \\
% \I& 0
% \end{pmatrix}
% \begin{pmatrix}
% \x_{t} \\ \x_{t-1}
% \end{pmatrix}  - \eta
% \begin{pmatrix}
% \grad f(0) + \delta_t \\ 0
% \end{pmatrix} \nn \\
% =& \A^{t}\begin{pmatrix}
% \x_1 \\ \x_0
% \end{pmatrix} 
% -\eta\sum_{\tau = 1}^t \A^{t-\tau}
% \begin{pmatrix}
% \grad f(0) + \delta_\tau \\ 0
% \end{pmatrix} 
% \label{eq:update_AGD_matrix}
% \end{align}

% Finally, we need to setup some property about $\delta_t$. It turns out if all iterates stay in some ball with radius $\uspace$ around origin, by Hessian smoothness, we can easily get upper bound of $\delta_t$:

\begin{proposition} \label{prop:delta} Using the notation of Lemma \ref{lem:AGD_update_general},
if for any $\tau\le t$, we have $\norm{\x_\tau} \le R$, then for any 
$\tau\le t$, we also have 
\begin{enumerate}
\item $\norm{\delta_\tau} \le O(\rho R^2)$; 
\item $\norm{\delta_\tau - \delta_{\tau -1}}
\le O(\rho R) (\norm{\x_t - \x_{\tau-1}} +\norm{\x_{\tau-1} - \x_{\tau-2}})$;
\item $\sum_{\tau = 1}^t \norm{\delta_\tau - \delta_{\tau -1}}^2 \le O(\rho^2 R^2)\sum_{\tau = 1}^t \norm{\x_\tau -\x_{\tau - 1}}^2$.
\end{enumerate}
\end{proposition}
\begin{proof}
Let $\Delta_\tau = \int_0^1 (\hess f(\phi \y_\tau) - \H) \mathrm{d} \phi$.
The first inequality is true because $\delta_\tau = \Delta_\tau \y_\tau$, thus:
\begin{align*}
\norm{\delta_\tau} = &\norm{\Delta_\tau\y_\tau} \le \norm{\Delta_\tau}\norm{\y_\tau}
= \norm{\int_0^1 (\hess f(\phi \y_\tau) - \H) \mathrm{d} \phi}\norm{\y_\tau} \\
\le & \int_0^1\norm{ (\hess f(\phi \y_\tau) - \H)} \mathrm{d} \phi \cdot \norm{\y_\tau}
\le \rho \norm{\y_\tau}^2
\le \rho \norm{(2-\theta)\x_\tau  - (1-\theta)\x_{\tau-1} }^2 \le O(\rho R^2).
\end{align*}
For the second inequality, we have:
\begin{align*}
\delta_\tau - \delta_{\tau-1} =\grad f(\y_\tau) - \grad f(\y_{\tau-1})  - \H(\y_\tau - \y_{\tau-1})
=\Delta'_\tau (\y_\tau - \y_{\tau-1}),
\end{align*}
where $\Delta'_\tau = \int_0^1 (\hess f(\y_{\tau-1} + \phi (\y_\tau - \y_{\tau-1})) - \H) \mathrm{d} \phi$.
As in the proof of the first inequality, we have:
\begin{align*}
\norm{\delta_\tau - \delta_{\tau-1}} \le& \norm{\Delta'_\tau}\norm{\y_\tau - \y_{\tau-1}}
= \norm{\int_0^1 (\hess f(\y_{\tau-1} + \phi (\y_\tau - \y_{\tau-1})) - \H) \mathrm{d} \phi}\norm{\y_\tau - \y_{\tau-1}} \\
\le & \rho \max\{\norm{\y_\tau}, \norm{\y_{\tau-1}}\}\norm{\y_\tau - \y_{\tau-1}} \le O(\rho R) (\norm{\x_\tau - \x_{\tau-1}} +\norm{\x_{\tau-1} - \x_{\tau-2}}).
\end{align*}
Finally, since $(\norm{\x_\tau - \x_{\tau-1}} +\norm{\x_{\tau-1} - \x_{\tau-2}})^2 \le 2(\norm{\x_\tau - \x_{\tau-1}}^2 + \norm{\x_{\tau-1} - \x_{\tau-2}}^2)$, the third inequality is immediately implied by the second inequality.
\end{proof}

%!TEX root = main.tex

\subsection{Proof for large-gradient scenario}
We prove Lemma \ref{lem:largeGrad} in this subsection. 
Throughout this subsection, we let $\S$ be the subspace with eigenvalues in $(\theta^2/[\eta(2-\theta)^2], \ell]$, and let $\S^c$ be the complementary subspace. Also let $\proj_{\S}$ and $\proj_{\S^c}$ be the corresponding projections.
We note $\theta^2/[\eta(2-\theta)^2] = \Theta(\sqrt{\rho\epsilon})$, and this particular choice lies at the boundary between the real eigenvalues and complex eigenvalues of the matrix $\A_j$, as shown in Lemma \ref{lem:aux_eigenvalues}.

% \cnote{Prove theorem given three lemmas here Deal with NCE here}
The first lemma shows that if momentum or gradient is very large, then the Hamiltonian already has sufficient decrease on average.

\begin{lemma}\label{lem:largegrad_momentum}
Under the setting of Theorem \ref{thm:main}, if $\norm{\v_t}\ge \umom$ or $\norm{\grad f(\x_t)} \ge 2\ell\umom$, and at time step $t$ only AGD is used without NCE or perturbation, then:
\begin{equation*}
 E_{t+1} - E_t \le - 4\ufun/\utime.
\end{equation*} 
\end{lemma}
\begin{proof}
When $\norm{\v_t} \ge \frac{\epsilon \sqrt{\cn}}{10\ell}$, by Lemma \ref{lem:energy_nonconvex}, we have:
\begin{align*}
E_{t+1} -E_t \le -\frac{\theta}{2\eta}\norm{\v_t}^2
\le -\Omega\left(\frac{\ell}{\sqrt{\cn}}\frac{\epsilon^2 \cn}{\ell^2}c^{-2}\right)
= -\Omega\left(\frac{\epsilon^2 \sqrt{\cn}}{2\ell}c^{-2}\right)
\le -\Omega(\frac{\ufun}{\utime} c^6) \le -\frac{4\ufun}{\utime}.
\end{align*}
The last step is by picking $c$ to be a large enough constant.
When $\norm{\v_t} \le \umom$ but $\norm{\grad f(\x_t)} \ge 2\ell\umom$, 
by the gradient Lipschitz assumption, we have:
\begin{equation*}
\norm{\grad f(\y_t)} \ge \norm{\grad f(\x_t)} - (1-\theta) \ell \norm{\v_t} \ge \ell \umom.
\end{equation*}
Similarly, by Lemma \ref{lem:energy_nonconvex}, we have:
\begin{align*}
E_{t+1} -E_t \le -\frac{\eta}{4}\norm{\grad f(\y_t)}^2 \le -\Omega(\frac{\epsilon^2 \cn}{\ell}c^{-2})
\le -\Omega(\frac{\ufun}{\utime} c^6) \le -\frac{4\ufun}{\utime}.
\end{align*}
Again the last step is by picking $c$ to be a large enough constant, which finishes the proof.
\end{proof}

Next, we show that if the initial momentum is small, but the initial gradient on the nonconvex subspace $\S^c$ is large enough, then within $O(\utime)$ steps, the Hamiltonian will decrease by at least $\ufun$.

\begin{lemma}[Formal Version of Lemma \ref{lem:2}]\label{lem:largegrad_nonconvex}
Under the setting of Theorem \ref{thm:main}, if $\norm{\proj_{\S^c}\grad f(\x_{0})} \ge \frac{\epsilon}{2}$, $\norm{\v_0} \le \umom$, $\v_0\trans [\proj_{\S}\trans\hess f(\x_0) \proj_{\S}] \v_0 \le  2 \sqrt{\rho\epsilon}\umom^2 $,
and for $t\in [0, \utime/4]$ only AGD steps are used without NCE or perturbation,
then:
\begin{equation*}
E_{\utime/4} - E_0 \le - \ufun.
\end{equation*}
\end{lemma}

\begin{proof}
The high-level plan is a proof by contradiction. We first assume that the
energy doesn't decrease very much; that is, $E_{\utime/4} - E_0 \ge - \ufun$ for a small enough constant $\mu$. By Corollary \ref{cor:localball} and the Cauchy-Swartz inequality, this immediately implies that for all $t \le \utime$, we have
$\norm{\x_t - \x_0} \le \sqrt{2\eta \utime \ufun/(4\theta)} = \uspace/2$. In
the rest of the proof we will show that this leads to a contradiction.

%  that $\norm{\x_{\utime} - \x_0} \ge \Omega(\uspace)$ is true,
% then by Corollary \ref{cor:localball} it immediately implies $E_{t+\utime} - E_t \le - \Omega(\ufun)$.

% For the remaining of the proof, we assume $\forall t\le \utime, \norm{\x_{t} - \x_0} \le c \uspace$, we will show this can not be true for sufficiently small constant $c$.

Given initial $\x_0$ and $\v_0$, we define $\x_{-1} = \x_0 - \v_0$.
Without loss of generality, set $\x_{0}$ as the origin $\zero$.  
Using the notation and results of Lemma \ref{lem:AGD_update_general}, 
we have the following update equation:
\begin{align*}
\pmat{\x_{t} \\ \x_{t-1}} =& \A^{t}\pmat{0 \\ -\v_0} 
-\eta\sum_{\tau = 0}^{t-1} \A^{t-1-\tau}\pmat{\grad f(0) + \delta_{\tau} \\ 0}.
\end{align*}
Consider the $j$-th eigen-direction of $\H = \hess f(\zero)$, 
recall the definition of the $2\times 2$ block matrix $\A_j$ as 
in \eqref{eq:definition_Aj}, and denote
\begin{equation*}
(a^{(j)}_{t}, ~-b^{(j)}_t) =\pmat{1 & 0 }\A_j^{t}.
\end{equation*}
Then we have for the $j$-th eigen-direction:
% \begin{align*}
% x_t^{(j)}
% = \frac{ \mu_1\mu_2(\mu_1^{t} - \mu_2^{t})}{\mu_1 - \mu_2}
% v_0^{(j)} 
% -\eta\sum_{\tau = 0}^{t-1}\frac{\mu_1^{t-\tau} - \mu_2^{t-\tau}}{\mu_1 - \mu_2}
% (\grad f(0)^{(j)} + \delta_\tau^{(j)})
% \end{align*}
\begin{align*}
x_t^{(j)}
= &b_t^{(j)} v_0^{(j)} 
-\eta\sum_{\tau = 0}^{t-1} a_{t-1-\tau}^{(j)}
(\grad f(0)^{(j)} + \delta_\tau^{(j)})\\
= &-\eta\left[\sum_{\tau = 0}^{t-1} a_\tau^{(j)}\right]
\left(\grad f(0)^{(j)} + \sum_{\tau = 0}^{t-1} p^{(j)}_\tau \delta_\tau^{(j)}
+ q^{(j)}_t  v_0^{(j)} \right),
\end{align*}
where 
\begin{align*}
p^{(j)}_\tau 
= \frac{a_{t-1-\tau}^{(j)}}{\sum_{\tau = 0}^{t-1} a_\tau^{(j)}}
\quad \text{~and~} \quad
q^{(j)}_t = - \frac{b_t^{(j)}}{\eta\sum_{\tau = 0}^{t-1} a_\tau^{(j)}}.
\end{align*}
Clearly $\sum_{\tau=0}^{t-1} p^{(j)}_\tau  = 1$. For $j \in \S^c$, by Lemma \ref{lem:aux_nonconvex_inequal}, we know
$\sum_{\tau = 0}^{t-1} a_\tau^{(j)} \ge \Omega(\frac{1}{\theta^2})$. 
We can thus further write the above equation as:
\begin{equation*}
x_t^{(j)} = -\eta\left[\sum_{\tau = 0}^{t-1} a_\tau^{(j)}\right]
\left(\grad f(0)^{(j)} + \tilde{\delta}^{(j)} + \tilde{v}^{(j)} \right),
\end{equation*}
where $\tilde{\delta}^{(j)} = \sum_{\tau = 0}^{t-1} p^{(j)}_\tau \delta_\tau^{(j)}$
and $\tilde{v}^{(j)} = q^{(j)}_t  v_0^{(j)}$, coming from the Hessian Lipschitz 
assumption and the initial momentum respectively.
For the remaining part, we would like to bound $\|\proj_{\S^c} \tilde{\delta}\|$ and $\norm{\proj_{\S^c} \tilde{\v}}$, and show that both of them are small compared to $\norm{\proj_{\S^c}\grad f(\x_{0})}$.

~

First, for the $\|\proj_{\S^c} \tilde{\delta}\|$ term, we know by definition of the subspace $\S^c$, and given that both eigenvalues of $\A_j$ are real and positive according to Lemma \ref{lem:aux_eigenvalues}, such that $p^{(j)}_\tau$ is positive by Lemma \ref{lem:aux_matrix_form}, we have for any $j\in \S^c$:
\begin{align*}
|\tilde{\delta}^{(j)}| = &|\sum_{\tau = 0}^{t-1} p^{(j)}_\tau \delta_\tau^{(j)}|
\le \sum_{\tau = 0}^{t-1} p^{(j)}_\tau (|\delta_0^{(j)}| + |\delta_\tau^{(j)} - \delta_0^{(j)} |)\\
\le &\left[\sum_{\tau = 0}^{t-1} p^{(j)}_\tau \right]\left(|\delta_0^{(j)}| + \sum_{\tau = 1}^{t-1}|\delta_\tau^{(j)} - \delta_{\tau-1}^{(j)} |\right)
\le |\delta_0^{(j)}| + \sum_{\tau = 1}^{t-1}|\delta_\tau^{(j)} - \delta_{\tau-1}^{(j)} |.
\end{align*}
By the Cauchy-Swartz inequality, this gives:
\begin{align*}
\norm{\proj_{\S^c} \tilde{\delta}}^2
=& \sum_{j\in \S^c} |\tilde{\delta}^{(j)}|^2 \le \sum_{j\in \S^c} (|\delta_0^{(j)}| + \sum_{\tau = 1}^{t-1}|\delta_\tau^{(j)} - \delta_{\tau-1}^{(j)} |)^2
\le 2\left[\sum_{j\in \S^c}|\delta_0^{(j)}|^2 + \sum_{j\in \S^c} (\sum_{\tau = 1}^{t-1}|\delta_\tau^{(j)} - \delta_{\tau-1}^{(j)} |)^2\right] \\
\le& 2\left[\sum_{j\in \S^c}|\delta_0^{(j)}|^2 + t\sum_{j\in \S^c} \sum_{\tau = 1}^{t-1}|\delta_\tau^{(j)} - \delta_{\tau-1}^{(j)} |^2 \right] 
\le 2\norm{\delta_0}^2 + 2t\sum_{\tau = 1}^{t-1}\norm{\delta_\tau- \delta_{\tau-1}}^2.
\end{align*}
Recall that for $t \le \utime$, we have $\norm{\x_t} \le \uspace/2$.
By Proposition \ref{prop:delta}, we know:
$\norm{\delta_0} \le O( \rho \uspace^2)$, 
and by Corollary \ref{cor:localball} and Proposition \ref{prop:delta}:
\begin{align*}
t\sum_{\tau = 1}^{t-1}\norm{\delta_\tau- \delta_{\tau-1}}^2
\le O(\rho^2 \uspace^2) t\sum_{\tau = 1}^{t-1}\norm{\x_\tau- \x_{\tau-1}}^2
\le O(\rho^2 \uspace^4).
\end{align*}
This gives $\|\proj_{\S^c} \tilde{\delta}\| \le O( \rho \uspace^2) \le O(\epsilon \cdot c^{-6}) \le \epsilon/10$.

~

Next we consider the $\norm{\proj_{\S^c} \tilde{\v}}$ term. 
By Lemma \ref{lem:aux_nonconvex_inequal}, we have
\begin{align*}
 - \eta q_t^{(j)}  = \frac{b_t}{\sum_{\tau = 0}^{t-1} a_\tau} \le O(1)\max\{\theta, \sqrt{\eta|\lambda_j|}\}.
\end{align*}
This gives:
\begin{equation}\label{eq:proof_large_grad_mom}
\norm{\proj_{\S^c} \tilde{\v}}^2  = \sum_{j\in\S^c} [q^{(j)}_t  v_0^{(j)}]^2 
\le O(1)\sum_{j\in\S^c}  \frac{\max\{ \eta|\lambda_j|, \theta^2\}}{\eta^2}  [v_0^{(j)}]^2.
\end{equation}
Recall that we have assumed by way of contradiction that $E_{\utime/4} - E_0 \le - \ufun$.
By the precondition that NCE is not used at $t=0$, due to the certificate \eqref{eq:certificate}, we have:
\begin{equation*}
\frac{1}{2} \v_0\trans 
\hess f(\zeta_0) \v_0 \ge - \frac{\gamma}{2} \norm{\v_0}^2 = -\frac{\sqrt{\rho\epsilon}}{8} \norm{\v_0}^2,
\end{equation*}
where $\zeta_0 = \phi\x_0 + (1-\phi)\y_0$ and $\phi \in [0, 1]$. Noting that we fix $\x_0$ as the origin $\zero$, by the Hessian Lipschitz property, it is easy to show that
$\norm{\hess f(\zeta_0) - \H} \le \rho\norm{\y_0} \le \rho \norm{\v_0} \le \rho \umom \le \sqrt{\rho\epsilon}$. This gives:
\begin{equation*}
\v_0 \H \v_0 \ge -2\sqrt{\rho\epsilon} \norm{\v_0}^2.
\end{equation*}
Again letting $\lambda_j$ denote the eigenvalues of $\H$, rearranging the above sum give:
\begin{align*}
\sum_{j:\lambda_j\le 0} |\lambda_j|[v_0^{(j)}]^2
 \le& O(\sqrt{\rho\epsilon})\norm{\v_0}^2 + \sum_{j:\lambda_j> 0} \lambda_j[v_0^{(j)}]^2 \\
 \le&  O(\sqrt{\rho\epsilon})\norm{\v_0}^2 + \sum_{j:\lambda_j> \theta^2/\eta(2-\theta)^2} \lambda_j[v_0^{(j)}]^2
 \le O(\sqrt{\rho\epsilon})\norm{\v_0}^2 + \v_0\trans [\proj_\S\trans \H \proj_\S] \v_0.
\end{align*}
The second inequality uses the fact that $\theta^2/\eta(2-\theta)^2 \le O(\sqrt{\rho\epsilon})$.
Substituting into \eqref{eq:proof_large_grad_mom} gives:
\begin{equation*}
\norm{\proj_{\S^c} \tilde{\v}}^2  \le
O(\frac{1}{\eta})\left[\sqrt{\rho\epsilon}\norm{\v_0}^2 + \v_0\trans [\proj_\S\trans \H \proj_\S] \v_0 \right]
\le O(\ell\sqrt{\rho\epsilon}\umom^2) = O(\epsilon^2 c^{-2}) \le \epsilon^2/100.
\end{equation*}
% \cnote{Using momentum condition for strongly convex part here}
Finally, putting all pieces together, we have:
\begin{align*}
\norm{\x_t} \ge& \norm{\proj_{\S^c}\x_t}  \ge \eta \left[\min_{j\in\S^c}\sum_{\tau = 0}^{t-1} a_\tau^{(j)}\right]
\norm{\proj_{\S^c} (\grad f(0) + \tilde{\delta} + \tilde{\v})}\\
\ge& \Omega(\frac{\eta}{\theta^2})\left[\norm{\proj_{\S^c} \grad f(0)} - \norm{\proj_{\S^c}\tilde{\delta}} - \norm{\proj_{\S^c}\tilde{\v})}\right]
\ge \Omega(\frac{\eta\epsilon}{\theta^2}) \ge \Omega(\uspace c^3) \ge \uspace
\end{align*}
which contradicts the fact $\norm{\x_t}$ that remains inside the ball around $\zero$ with radius $\uspace/2$.
\end{proof}

% \begin{lemma}
% If $\norm{\v_0}\le \frac{\epsilon \sqrt{\cn}}{\ell}$ and $\norm{\grad f(\x_0)} \le 2\epsilon\sqrt{\cn}$. suppose
% \begin{equation*}
% E_{2\utime} - E_0 \ge - \mu \ufun
% \end{equation*}
% for small enough $\mu$, then, we have for all $t\in [\utime, 2\utime]$ that:
% \begin{equation*}
% \norm{\proj_\S\grad f(\x_{t})} \le \frac{\epsilon}{10}
% \text{~~and~~}
% % \norm{\proj_\S(\x_t - \x_{t-1})} \le \frac{\epsilon}{\ell}
% \v_t\trans [\proj_\S\trans\H \proj_\S] \v_t \le c \eta \epsilon^2
% \end{equation*}
% \end{lemma}
The next lemma shows that if the initial momentum and gradient are 
reasonably small, and the Hamitonian does not have sufficient decrease 
over the next $\utime$ iterations, then both the gradient and momentum 
of the strongly convex component $\S$ will vanish in $\utime/4$ iterations.

\begin{lemma}[Formal Version of Lemma \ref{lem:1}]\label{lem:largegrad_convex}
Under the setting of Theorem \ref{thm:main}, suppose $\norm{\v_0}\le \umom$ and $\norm{\grad f(\x_0)} \le 2\ell\umom$, $E_{\utime/2} - E_0 \ge - \ufun$,
and for $t\in [0, \utime/2]$ only AGD steps are used, without NCE or perturbation.
Then $\forall \; t\in [\utime/4, \utime/2]$:
    \begin{equation*}
    \norm{\proj_\S\grad f(\x_{t})} \le \frac{\epsilon}{2}
    \text{~~and~~}
    \v_t\trans [\proj_\S\trans \hess f(\x_0) \proj_\S] \v_t \le \sqrt{\rho\epsilon}\umom^2.
    \end{equation*}
\end{lemma}

\begin{proof}
Since $E_{\utime} - E_0 \ge - \ufun$, by Corollary \ref{cor:localball} and 
the Cauchy-Swartz inequality, we see that for all $t \le \utime$ we have 
$\norm{\x_t - \x_0} \le \sqrt{2\eta \utime \ufun/\theta} = \uspace$.

Given initial $\x_0$ and $\v_0$, we define $\x_{-1} = \x_0 - \v_0$.
Without loss of generality, setting $\x_{0}$ as the origin $\zero$, 
by the notation and results of Lemma \ref{lem:AGD_update_general}, 
we have the update equation:
% Given initial $\x_0$ and $\v_0$, we can construct an equivalent $\x_{-1} = \x_0 - \v_0$.
% WLOG, set $\x_{0}$ as origin $0$, and denote $\H = \hess f(0)$. Recall the update equation:
\begin{align} \label{eq:update_AGD_matrix_strongly_convex}
\pmat{\x_{t} \\ \x_{t-1}}
% =&  \A^{t}\pmat{\x_0 \\ \x_{-1}} 
% -\eta\sum_{\tau = 0}^{t-1} \A^{t-1-\tau}\pmat{\grad f(\x_{0}) + \delta_{\tau} \\ 0}  \nn \\
=& \A^{t}\pmat{0 \\ -\v_0} 
-\eta\sum_{\tau = 0}^{t-1} \A^{t-1-\tau}\pmat{\grad f(0) + \delta_{\tau} \\ 0}.
\end{align}
% where $\norm{\delta_\tau} \le \sqrt{\rho\epsilon}(\norm{\x_\tau} + \norm{\x_{\tau-1}})$. 

First we prove the upper bound on the gradient: $\forall \; t\in [\utime/4, \utime]$, we have $\norm{\proj_\S\grad f(\x_{t})} \le \frac{\epsilon}{2}$. 
Let $\Delta_t = \int_{0}^1 (\hess f(\phi \x_t) - \H)\mathrm{d}\phi$. According to \eqref{eq:update_AGD_matrix_strongly_convex}, we have:
\begin{align*}
\grad f(\x_t) = &
\grad f(0) + (\H + \Delta_t)\x_t \\
=& \underbrace{\left(\I - \eta\H \pmat{\I & 0}\sum_{\tau = 0}^{t-1} \A^{t-1-\tau}\pmat{\I \\ 0}\right) \grad f(0)}_{\g_1}
+ \underbrace{\H \pmat{\I & 0}\A^t\pmat{0 \\ -\v_0}}_{\g_2} \\
&- \underbrace{\eta\H \pmat{\I & 0}\sum_{\tau = 0}^{t-1} \A^{t-1-\tau}\pmat{\delta_t \\ 0}}_{\g_3}
+ \underbrace{\Delta_t \x_t}_{\g_4}.
\end{align*}
We will upper bound four terms $\g_1, \g_2, \g_3, \g_4$ separately. 
Clearly, for the last term $\g_4$, we have:
$$\norm{\g_4} \le \rho\norm{\x_t}^2 \le O(\rho \uspace^2) = O(\epsilon c^{-6}) \le \epsilon/8.$$
Next, we show that the first two terms $\g_1, \g_2$ become very small for $t\in [\utime/4, \utime]$.
Consider coordinate $j \in \S$ and the $2\times 2$ block matrix $\A_j$.  By Lemma \ref{lem:aux_matrix_equality} we have:
\begin{align*}
1 - \eta\lambda_j \pmat{1 & 0}\sum_{\tau = 0}^{t-1} \A_j^{t-1-\tau}\pmat{1 \\ 0}
 = \pmat{1 & 0}  \A_j^t \pmat{1 \\ 1}.
\end{align*}
Denote:
\begin{equation*}
(a^{(j)}_{t}, ~-b^{(j)}_t) =\pmat{1 & 0 }\A_j^{t}.
\end{equation*}
By Lemma \ref{lem:aux_convex_entry}, we know:
\begin{equation*}
\max_{j\in \S} \left\{ |a^{(j)}_t|
, |b^{(j)}_t| \right\} \le (t+1) (1-\theta)^{\frac{t}{2}} .
\end{equation*}
This immediately gives when $t\ge \utime/4 = \Omega(\frac{c}{\theta}\log\frac{1}{\theta})$ for $c$ sufficiently large:
\begin{align*}
\norm{\proj_\S \g_1}^2 =& \sum_{j\in \S} |(a^{(j)}_{t} - b^{(j)}_t) \grad f(0)^{(j)}|^2
\le (t+1)^2(1-\theta)^t \norm{\grad f(0)}^2 \le \epsilon^2/64 \\
\norm{\proj_\S \g_2}^2 = & \sum_{j\in \S} |\lambda_j  b^{(j)}_t \v_0^{(j)}|^2
\le \ell^2 (t+1)^2 (1-\theta)^t \norm{\v_0}^2 \le \epsilon^2/64.
\end{align*}
Finally, for $\g_3$, by Lemma \ref{lem:aux_convex_inequal}, for all $j\in \S$, we have
\begin{equation*}
|\g_3^{(j)}| = \left|\eta \lambda_j \sum_{\tau = 0}^{t-1}a^{(j)}_\tau \delta_{t-1-\tau}\right|
\le |\delta^{(j)}_{t-1}|+  \sum_{\tau=1}^{t-1} |\delta^{(j)}_\tau - \delta^{(j)}_{\tau-1}|.
\end{equation*}
By Proposition \ref{prop:delta}, this gives:
\begin{equation*}
\norm{\proj_\S \g_3}^2 \le 
2\norm{\delta_{t-1}}^2 + 2t\sum_{\tau = 1}^{t-1}\norm{\delta_\tau- \delta_{\tau-1}}^2 
\le O(\rho^2\uspace^4) \le O(\epsilon^2 \cdot c^{-12}) \le \epsilon^2/64.
\end{equation*}
In sum, this gives for any fixed $t \in [\utime/4, \utime]$:
\begin{equation*}
\norm{\proj_\S \grad f(\x_t)}\le \norm{\proj_\S \g_1} + \norm{\proj_\S \g_2} + \norm{\proj_\S \g_3} + \norm{\g_4} \le \frac{\epsilon}{2}.
\end{equation*}

We now provide a similar argument to prove the upper bound for the momentum. 
That is, $\forall \; t\in [\utime/4, \utime]$, we show
$\v_t\trans [\proj_\S\trans \hess f(\x_0) \proj_\S] \v_t \le \sqrt{\rho\epsilon}\umom^2$. According to \eqref{eq:update_AGD_matrix_strongly_convex}, 
we have:
\begin{align*}
\v_t = \pmat{1 & -1}\pmat{\x_{t} \\ \x_{t-1}}
=& \underbrace{\pmat{1 & -1}\A^{t}\pmat{0 \\ -\v_0} }_{\m_1}
-\underbrace{\eta\pmat{1 & -1}\sum_{\tau = 0}^{t-1} \A^{t-1-\tau}\pmat{\grad f(0)  \\ 0}}_{\m_2} \\
&- \underbrace{\eta\pmat{1 & -1}\sum_{\tau = 0}^{t-1} \A^{t-1-\tau}\pmat{\delta_{\tau} \\ 0}}_{\m_3}.
\end{align*}
Consider the $j$-th eigendirection, so that $j \in \S$, and recall the $2\times 2$ block matrix $\A_j$. Denoting
\begin{equation*}
(a^{(j)}_{t}, ~-b^{(j)}_t) =\pmat{1 & 0 }\A_j^{t},
\end{equation*}
by Lemma \ref{lem:aux_matrix_form} and \ref{lem:aux_convex_entry}, we have for $t\ge \utime/4 = \Omega(\frac{c}{\theta}\log\frac{1}{\theta})$ with $c$ sufficiently large:
\begin{equation*}
\norm{[\proj_\S\trans \hess f(\x_0) \proj_\S]^{\frac{1}{2}}\m_1}^2 = \sum_{j \in \S} |\lambda_j^{\frac{1}{2}} (b^{(j)}_t -b^{(j)}_{t-1})  \v_0^{(j)}|^2
\le \ell(t+1)^2 (1-\theta)^t \norm{\v_0}^2 \le O(\frac{\epsilon^2}{\ell}c^{-3}) \le \frac{1}{3}\sqrt{\rho\epsilon}\umom^2.
\end{equation*}
On the other hand, by Lemma \ref{lem:aux_matrix_equality}, we have:
\begin{align*}
\abs{\eta \lambda_j \pmat{1 & -1}\sum_{\tau = 0}^{t-1} \A_j^{t-1-\tau}\pmat{1 \\0}}
= \abs{\eta \lambda_j \pmat{1 & 0}\sum_{\tau = 0}^{t-1} (\A_j^{t-1-\tau}- \A_j^{t-2-\tau})\pmat{1 \\0}}
= \abs{\pmat{1 & 0}  (\A_j^t -\A_j^{t-1})  \pmat{1 \\ 1}}.
\end{align*}
This gives, for $t\ge \utime/4 = \Omega(\frac{c}{\theta}\log\frac{1}{\theta})$, and
for $c$ sufficiently large:
\begin{align*}
\norm{[\proj_\S\trans \hess f(\x_0) \proj_\S]^{\frac{1}{2}}\m_2}^2 
=& \sum_{j \in \S} |\lambda_j^{-\frac{1}{2}} (a^{(j)}_t -a^{(j)}_{t-1} - b^{(j)}_t +b^{(j)}_{t-1} )  \grad f(0)^{(j)}|^2 \\
\le& O(\frac{1}{\sqrt{\rho\epsilon}})(t+1)^2(1-\theta)^t \norm{\grad f(0)}^2 \le O(\frac{\epsilon^2}{\ell}c^{-3}) \le \frac{1}{3}\sqrt{\rho\epsilon}\umom^2.
\end{align*}
Finally, for any $j \in \S$, by Lemma \ref{lem:aux_convex_inequal}, we have:
\begin{equation*}
|(\H^{\frac{1}{2}}\m_3)^{(j)}| = |\eta\lambda_j^{\frac{1}{2}}\sum_{\tau = 0}^{t-1}(a_\tau - a_{\tau -1}) \delta_{t-1-\tau}|
\le \sqrt{\eta} \left[\sum|\delta^{(j)}_{t-1}|+  \sum_{\tau=1}^{t-1} |\delta^{(j)}_\tau - \delta^{(j)}_{\tau-1}|\right].
\end{equation*}
Again by Proposition \ref{prop:delta}:
\begin{align*}
\norm{[\proj_\S\trans \hess f(\x_0) \proj_\S]^{\frac{1}{2}}\m_3}^2 
= \eta \left[2\norm{\delta_{t-1}}^2 + 2t\sum_{\tau = 1}^{t-1}\norm{\delta_\tau- \delta_{\tau-1}}^2 \right]
\le O(\eta\rho^2\uspace^4) \le O(\frac{\epsilon^2}{\ell}c^{-6}) \le \frac{1}{3}\sqrt{\rho\epsilon}\umom^2.
\end{align*}
Putting everything together, we have:
\begin{align*}
\v_t\trans [\proj_\S\trans \hess f(\x_0) \proj_\S] \v_t \le& 
\norm{[\proj_\S\trans \hess f(\x_0) \proj_\S]^{\frac{1}{2}}\m_1}^2
+ \norm{[\proj_\S\trans \hess f(\x_0) \proj_\S]^{\frac{1}{2}}\m_2}^2\\
&+ \norm{[\proj_\S\trans \hess f(\x_0) \proj_\S]^{\frac{1}{2}}\m_3}^2
\le \sqrt{\rho\epsilon}\umom^2.
\end{align*}
This finishes the proof.
\end{proof}

Finally, we are ready to prove the main lemma of this subsection (Lemma \ref{lem:largeGrad}), which claims that if gradients in $\utime$ iterations are always large, then the Hamiltonian will decrease sufficiently within a small number of steps.
\begingroup
\def\thetheorem{\ref{lem:largeGrad}}
\begin{lemma}[Large gradient]
Consider the setting of Theorem~\ref{thm:main}.
%If in~\pagd~(Algorithm~\ref{algo:PAGD}),
If $\norm{\grad f(\x_\tau)} \ge \epsilon$ for all $ \tau \in [0, \utime]$, then by running Algorithm \ref{algo:PAGD} we have $E_{\utime} - E_0 \le -\ufun$.
\end{lemma}
\addtocounter{theorem}{-1}
\endgroup

\begin{proof}
Since $\norm{\grad f(\x_\tau)} \ge \epsilon$ for all $ \tau \in [0, \utime]$, according to Algorithm \ref{algo:PAGD}, the precondition to add perturbation never holds, so Algorithm will not add any perturbation in these $\utime$ iterations.

Next, suppose there is at least one iteration where NCE is used. Then by Lemma \ref{lem:NCE_decrease}, we know that that step alone gives $\ufun$ decrease in the Hamiltonian. According to Lemma \ref{lem:energy_nonconvex} and Lemma \ref{lem:NCE_decrease} we know that without perturbation, the Hamiltonian decreases monotonically in the remaining steps. This means whenever at least one NCE step is performed, Lemma \ref{lem:largeGrad} immediately holds.

For the remainder of the proof, we can restrict the discussion to the case where NCE is never performed in steps $\tau \in [0, \utime]$.
Letting
\begin{equation*}
\tau_1 = \arg\min_{t \in [0, \utime]}\left\{t \left| \norm{\v_t}\le \umom
\text{~and~} \norm{\grad f(\x_t)} \le 2\ell\umom \right.\right\},
\end{equation*}
we know in case $\tau_1 \ge \frac{\utime}{4}$, that Lemma \ref{lem:largegrad_momentum} 
ensures 
$E_{\utime} - E_0 \le E_{\frac{\utime}{4}} - E_0 \le -\ufun$.
Thus, we only need to discuss the case $\tau_1 \le \frac{\utime}{4}$.
Again, if $E_{\tau_1 + \utime/2} - E_{\tau_1}\le -\ufun$, Lemma \ref{lem:largeGrad} immediately holds. For the remaining case, $E_{\tau_1 + \utime/2} - E_{\tau_1}\le -\ufun$, we apply Lemma \ref{lem:largegrad_convex} starting at $\tau_1$, and obtain
    \begin{equation*}
    \norm{\proj_\S\grad f(\x_{t})} \le \frac{\epsilon}{2}
    \text{~~and~~}
    % \norm{\proj_\S(\x_t - \x_{t-1})} \le \frac{\epsilon}{\ell}
    \v_t\trans [\proj_\S\trans \hess f(\x_{\tau_1}) \proj_\S] \v_t \le \sqrt{\rho\epsilon}\umom^2.
    ~~\quad \forall t \in [\tau_1 + \frac{\utime}{4}, \tau_1 + \frac{\utime}{2}].
    \end{equation*}
Letting:
\begin{equation*}
\tau_2 = \arg\min_{t \in [\tau_1 + \frac{\utime}{4}, \utime]}\left\{t \left| \norm{\v_t}\le \umom \right.\right\},
\end{equation*}
by Lemma \ref{lem:largegrad_momentum} we again know we only need to discuss the case where $\tau_2 \le \tau_1 + \frac{\utime}{2}$; otherwise, we already guarantee sufficient decrease in the Hamiltonian.
Then, we clearly have $\norm{\proj_\S\grad f(\x_{\tau_2})} \le \frac{\epsilon}{2}$, also by the precondition of Lemma \ref{lem:largeGrad}, we know $\norm{\grad f(\x_{\tau_2})} \ge \epsilon$, 
thus $\norm{\proj_{\S^c}\grad f(\x_{\tau_2})} \ge \frac{\epsilon}{2}$.
On the other hand, since if the Hamiltonian does not decrease enough, $E_{\tau_2} - E_0 \ge -\ufun$, 
by Lemma \ref{cor:localball}, we have $\norm{\x_{\tau_1} - \x_{\tau_2}} \le 2\uspace$, by the Hessian Lipschitz property, which gives:
\begin{equation*}
\v_{\tau_2}\trans [\proj_\S\trans \hess f(\x_{\tau_2}) \proj_\S] \v_{\tau_2} \le 
\v_{\tau_2}\trans [\proj_\S\trans \hess f(\x_{\tau_1}) \proj_\S] \v_{\tau_2}
+ \norm{\hess f(\x_{\tau_1}) - \hess f(\x_{\tau_2})}\norm{\v_{\tau_2}}^2
\le 2\sqrt{\rho\epsilon}\umom^2.
\end{equation*}
Now $\x_{\tau_2}$ satisfies all the preconditions of Lemma \ref{lem:largegrad_nonconvex}, and by applying Lemma \ref{lem:largegrad_nonconvex} we finish the proof.
\end{proof}

%!TEX root = main.tex

\subsection{Proof for negative-curvature scenario}
We prove Lemma \ref{lem:negHess} in this section. 
We consider two trajectories, starting at $\x_0$
and $\modify{\x}_0$, with $\v_0=\modify{\v}_0$, where $\w_0 =
\x_0 -  \modify{\x}_0 = r_0\e_1$, where $\e_1$ is the minimum 
eigenvector direction of $\H$, and where $r_0$ is not too small. 
We show that at least one of the trajectories will escape 
saddle points efficiently.

\begin{lemma}[Formal Version of Lemma \ref{lem:informal_neg_curve}]\label{lem:2nd_seq}
% Let $r$ be choosen according to Theorem \ref{thm:main}, $\tilde{\x}$ satisfies 
% \begin{equation*}
% \norm{\nabla f(\tilde{\x})} \le \epsilon \quad \quad \text{~and~} \quad \quad \lambda_{\min}(\hess f(\tilde{\x})) \le - \sqrt{\rho\epsilon}
% \end{equation*}
% And $\x_0,  \modify{\x}_0$ are at most $r$ distance away from $\tilde{\x}$, $\v_0 = \modify{\v}_0$.
% $\w_0 =\x_0 -  \modify{\x}_0 = \Omega(\delta r/\sqrt{d}) \cdot \e_1$ where $\e_1$ is the minimum eigenvector direction of $\H$. Then we have:
% \begin{align*}
% \min\{E_\utime - E_0, \modify{E}_\utime - \modify{E}_0\} \le - 2\ufun
% \end{align*}

Under the same setting as Theorem \ref{thm:main}, suppose $\norm{\nabla f(\tilde{\x})} \le \epsilon$ and $\lambda_{\min}(\hess f(\tilde{\x})) \le - \sqrt{\rho\epsilon}$. 
Let $\x_0$ and $\modify{\x}_0$ be at distance at most $r$ from $\tilde{\x}$.
%, $\v_0 = \modify{\v}_0$.
Let $\x_0 -  \modify{\x}_0 = r_0 \cdot \e_1$ and let $\v_0 = \modify{\v}_0 = \tilde{\v}$ where $\e_1$ is the minimum eigen-direction of $\hess f(\tilde{\x})$. 
Let $r_0 \ge \frac{\delta\ufun }{2\Delta_f}\cdot\frac{r}{\sqrt{d}}$. 
Then, running~\nag~starting at $(\x_0, \v_0)$ and $(\x'_0, \v'_0)$ respectively, we have:
\begin{align*}
\min\{E_{\utime} - \widetilde{E}, \modify{E}_{\utime} - \widetilde{E}\} \le - \ufun,
\end{align*}
where $\widetilde{E},E_{\utime}$ and $\modify{E}_{\utime}$ are the Hamiltonians at $(\tilde{\x}, \tilde{\v}), (\x_{\utime}, \v_{\utime})$ and $(\modify{\x}_{\utime}, \modify{\v}_{\utime})$ respectively.

\end{lemma}

\begin{proof}
Assume none of the two sequences decrease the Hamiltonian fast enough; that is,
\begin{align*}
\min\{E_{\utime} -E_0, \modify{E}_{\utime} - \modify{E}_0\} \ge - 2\ufun,
\end{align*}
where $E_0$ and $\modify{E}_0$ are the Hamiltonians at $(\x_0, \v_0)$ and $(\modify{\x}_0, \modify{\v}_0)$.
Then, by Corollary \ref{cor:localball} and the Cauchy-Swartz inequality, we have
for any $t \le \utime$:
\begin{equation*}
\max\{\norm{\x_t - \tilde{\x}},  \norm{\modify{\x}_t - \tilde{\x}}\} \le 
r + \max\{\norm{\x_t - \x_0},  \norm{\modify{\x}_t - \modify{\x}_0}\}
\le r + \sqrt{4\eta \utime \ufun/\theta} \le 2\uspace.
\end{equation*}
Fix the origin $\zero$ at $\tilde{\x}$ and let $\H$ be the Hessian at $\tilde{\x}$. Recall that the update equation of AGD (Algorithm \ref{algo:AGD}) can be re-written as:
\begin{align*}
\x_{t+1}  =& (2-\theta) \x_{t} - (1-\theta)  \x_{t-1}
 -\eta \grad f((2-\theta) \x_{t} - (1-\theta)  \x_{t-1}) 
\end{align*}
Taking the difference of two AGD sequences starting from $\x_0,  \modify{\x}_0$, and let $\w_t = \x_t - \modify{\x}_t$, we have:
\begin{align*}
\w_{t+1} =& (2-\theta) \w_{t} - (1-\theta)  \w_{t-1}
-\eta \grad f( \y_{t}) 
+ \eta \grad f(\modify{\y}_{t})\\ 
=& (2-\theta)(I - \eta\H - \eta\Delta_t) \w_{t} - (1-\theta)(I - \eta\H  -\eta\Delta_t) \w_{t-1},
\end{align*}
where $\Delta_t = \int_{0}^1 (\hess f(\phi\y_{t} + (1-\phi)\modify{\y}_{t}) - \H) \mathrm{d}\phi$. In the last step, we used 
\begin{align*}
\grad f( \y_{t}) - \grad f(\modify{\y}_{t})
= (\H + \Delta_t)(\y_{t} - \modify{\y}_{t})
= (\H + \Delta_t)[(2-\theta) \w_{t} - (1-\theta)  \w_{t-1}].
\end{align*}
We thus obtain the update of the $\w_t$ sequence in matrix form:
\begin{align}
\pmat{\w_{t+1} \\ \w_t}=& \pmat{(2-\theta) (\I - \eta \H)&  -(1-\theta) (\I - \eta \H) \\\I& 0}
\pmat{\w_{t} \\ \w_{t-1}} \nn \\
&- \eta\pmat{(2-\theta) \Delta_t \w_t - (1 - \theta) \Delta_{t} \w_{t-1} \\ 0} \nn \\
=& \A\pmat{\w_t \\ \w_{t-1}} -\eta\pmat{ \delta_t \\ 0} = \A^{t+1}\pmat{\w_0 \\ \w_{-1}} 
-\eta\sum_{\tau = 0}^t \A^{t-\tau}\pmat{ \delta_\tau \\ 0},
\label{eq:update_w}
\end{align}
where $\delta_t = (2-\theta) \Delta_t \w_t - (1 - \theta) \Delta_{t} \w_{t-1}$. Since $\v_0 = \v'_0$, we have $\w_{-1} = \w_0$, and $\norm{\Delta_t} \le \rho\max\{\norm{\x_t - \tilde{\x}},  \norm{\modify{\x}_t- \tilde{\x}}\}
\le 2\rho\uspace$, as well as $\norm{\delta_\tau} \le  6\rho\uspace(\norm{\w_\tau} + \norm{\w_{\tau-1}})$. According to \eqref{eq:update_w}:
\begin{align*}
\w_t
=& \pmat{\I & 0}\A^{t}\pmat{\w_0 \\ \w_{0}} -\eta\pmat{\I & 0}\sum_{\tau = 0}^{t-1} \A^{t-1-\tau}\pmat{ \delta_\tau \\ 0}.
\end{align*}
Intuitively, we want to say that the first term dominates. Technically, we will 
set up an induction based on the following fact:
\begin{align*}
\norm{\eta\pmat{\I, 0}\sum_{\tau = 0}^{t-1} \A^{t-1-\tau}\pmat{\delta_\tau \\ 0} }
\le \frac{1}{2}\norm{\pmat{\I, 0}\A^{t}\pmat{\w_0 \\ \w_{0}}}.
\end{align*}
% \cnote{No need to do $t/T$, $1/2$ suffices.}

It is easy to check the base case holds for $t=0$. Then, assume that for all time steps less than or equal to $t$, the induction assumption hold. We have:
% \cnote{modify this to hold for $t$, use new notation}
\begin{align*}
\norm{\w_t} 
\le& \norm{\pmat{\I & 0}\A^{t}\pmat{\w_0 \\ \w_{0}}} +\norm{\eta\pmat{\I & 0}\sum_{\tau = 0}^{t-1} \A^{t-1-\tau}\pmat{ \delta_\tau \\ 0} } \\
\le& 2\norm{\pmat{\I & 0}\A^{t}\pmat{\w_0 \\ \w_{0}}},
\end{align*}
which gives:
\begin{align*}
\norm{\delta_t} \le& O(\rho\uspace)(\norm{\w_t} + \norm{\w_{t-1}})
\le O(\rho\uspace) \left[\norm{\pmat{\I & 0}\A^{t}\pmat{\w_0 \\ \w_{0}}}
+ \norm{\pmat{\I & 0}\A^{t-1}\pmat{\w_0 \\ \w_{0}}}\right] \\
\le& O(\rho\uspace) \norm{\pmat{\I & 0}\A^{t}\pmat{\w_0 \\ \w_{0}}},
\end{align*}
where in the last inequality, we used Lemma \ref{lem:aux_increase_t} 
for monotonicity in $t$.

To prove that the induction assumption holds for $t+1$ we compute:
\begin{align}
\norm{\eta\pmat{\I, 0}\sum_{\tau = 0}^{t} \A^{t-\tau}\pmat{\delta_\tau \\ 0} }
\le& \eta \sum_{\tau = 0}^{t} \norm{\pmat{\I, 0}\A^{t-\tau}\pmat{\I \\ 0}}
\norm{\delta_\tau}  \nn \\
\le& O(\eta\rho\uspace) \sum_{\tau = 0}^{t} \norm{\pmat{\I, 0}\A^{t-\tau}\pmat{\I \\ 0}}
\norm{\pmat{\I & 0}\A^{\tau}\pmat{\w_0 \\ \w_{0}}}. \label{eq:saddle_app}
\end{align}
By the precondition we have $\lambda_{\min}(\H) \le -\sqrt{\rho\epsilon}$. 
Without loss of generality, assume that the minimum eigenvector direction 
of $\H$ is along he first coordinate $\e_1$, and denote the corresponding 
$2\times 2$ matrix as $\A_1$ (as in the convention of \eqref{eq:definition_Aj}. 
Let:
\begin{equation*}
(a^{(1)}_{t}, ~-b^{(1)}_t) =\pmat{1 & 0 }\A_1^{t}.
\end{equation*} 
We then see that (1) $\w_0$ is along the $\e_1$ direction, 
and (2) according to Lemma \ref{lem:aux_increase_x}, the matrix 
$\pmat{\I, 0}\A^{t-\tau}\pmat{\I \\ 0}$ is a diagonal matrix, where the spectral norm is achieved along the first coordinate which corresponds to the eigenvalue 
$\lambda_{\min}(\H)$. Therefore, using Equation \eqref{eq:saddle_app}, we have:
\begin{align*}
\norm{\eta\pmat{\I, 0}\sum_{\tau = 0}^{t} \A^{t-\tau}\pmat{\delta_\tau \\ 0} }
\le& O(\eta\rho\uspace) \sum_{\tau = 0}^t a^{(1)}_{t-\tau}(a^{(1)}_{\tau} - b^{(1)}_{\tau}) \norm{\w_0}\\
\le& O(\eta\rho\uspace) \sum_{\tau =0}^t [\frac{2}{\theta} + (t+1)] |a^{(1)}_{t+1} - b^{(1)}_{t+1}|\norm{\w_0}\\
\le& O(\eta\rho\uspace\utime^2)\norm{\pmat{\I, 0}\A^{t+1}\pmat{\w_0 \\ \w_{0}}},
\end{align*}
where, in the second to last step, we used Lemma \ref{lem:aux_eigen_combo_inequal}, and in the last step we used $1/\theta \le \utime$. Finally,
$O(\eta\rho\uspace\utime^2) \le O(c^{-1})\le 1/2$ by choosing a sufficiently large constant $c$. Therefore, we have proved the induction, which gives us:
\begin{align*}
\norm{\w_t} =& \norm{\pmat{\I & 0}\A^{t}\pmat{\w_0 \\ \w_{0}} }
-\norm{\eta\pmat{\I & 0}\sum_{\tau = 0}^{t-1} \A^{t-1-\tau}\pmat{ \delta_\tau \\ 0} }
\ge \frac{1}{2}\norm{\pmat{\I & 0}\A^{t}\pmat{\w_0 \\ \w_{0}} }.
\end{align*}
Noting that $\lambda_{\min}(\H) \le -\sqrt{\rho\epsilon}$, by applying Lemma \ref{lem:aux_increase_t} we have 
\begin{equation*}
\frac{1}{2}\norm{\pmat{\I & 0}\A^{t}\pmat{\w_0 \\ \w_{0}}}
\ge \frac{\theta}{4}(1+ \Omega(\theta))^t r_0,
\end{equation*}
which grows exponentially. Therefore, for  $r_0 \ge \frac{\delta\ufun }{2\Delta_f}\cdot\frac{r}{\sqrt{d}}$,
and $\utime = \Omega(\frac{1}{\theta}\cdot\chi c)$ where $\chi =\max\{1, \log \frac{d \ell\Delta_f}{\rho \epsilon\delta}\}$, where the constant $c$ is sufficiently large, 
we have
$$\norm{\x_\utime - \modify{\x}_\utime} = \norm{\w_\utime} \ge \frac{\theta}{4}(1+ \Omega(\theta))^\utime r_0 \ge 4\uspace,$$
which contradicts the fact that:
\begin{equation*}
\forall t \le \utime, \max\{\norm{\x_t - \tilde{\x}},  \norm{\modify{\x}_t - \tilde{\x}}\} \le O(\uspace).
\end{equation*}
This means our assumption is wrong, and we can therefore conclude:
\begin{align*}
\min\{E_{\utime} -E_0, \modify{E}_{\utime} - \modify{E}_0\} \le - 2\ufun.
\end{align*}
On the other hand, by the precondition on $\tilde{x}$ and the gradient Lipschitz
property, we have:
\begin{align*}
\max\{E_0 - \tilde{E}, \modify{E}_0- \tilde{E}\}
\le \epsilon r + \frac{\ell r^2}{2} \le \ufun,
\end{align*}
where the last step is due to our choice of $r= \eta\epsilon\cdot \chi^{-5}c^{-8}$ 
in \eqref{eq:parameter}.  Combining these two facts:
\begin{align*}
\min\{E_{\utime} -\tilde{E}, \modify{E}_{\utime} - \tilde{E}\}
\le \min\{E_{\utime} -E_0, \modify{E}_{\utime} - \modify{E}_0\} +
\max\{E_0 - \tilde{E}, \modify{E}_0- \tilde{E}\} \le -\ufun,
\end{align*}
which finishes the proof.
\end{proof}

We are now ready to prove the main lemma in this subsection, which states 
with that random perturbation, PAGD will escape saddle points efficiently 
with high probability.
\begingroup
\def\thetheorem{\ref{lem:negHess}}
\begin{lemma}[Negative curvature]
Consider the setting of Theorem~\ref{thm:main}. 
If $\norm{\grad f(\x_0)} \le \epsilon$, $\lambda_{\min} (\hess f(\x_0)) < -\sqrt{\rho\epsilon}$, 
and a perturbation has not been added in iterations $\tau \in [-\utime, 0)$, 
then, by running Algorithm \ref{algo:PAGD}, we have $E_{\utime} - E_0 \le -\ufun$ 
with probability at least $1-\frac{\delta \ufun}{2\Delta_f}$.
\end{lemma}
% \begin{lemma}[Negative curvature]
% Consider the setting of Theorem~\ref{thm:main}. 
% If $\norm{\grad f(\x_0)} \le \epsilon$ and $\lambda_{\min} (\hess f(\x_0)) < -\sqrt{\rho\epsilon}$, then by running Algorithm \ref{algo:PAGD}, we have $E_{\utime} - E_0 \le -\Omega (\ufun)$ with high probability.
% \end{lemma}
\addtocounter{theorem}{-1}
\endgroup

\begin{proof}
Since a perturbation has not been added in iterations $\tau \in [-\utime, 
0)$, according to PAGD (Algorithm \ref{algo:PAGD}), we add perturbation at 
$t=0$, the Hamiltonian will increase by at most:
\begin{align*}
\Delta E
\le \epsilon r + \frac{\ell r^2}{2} \le \ufun,
\end{align*}
where the last step is due to our choice of $r= \eta\epsilon\cdot \chi^{-5}c^{-8}$ in \eqref{eq:parameter} with constant $c$ sufficiently large.
Again by Algorithm \ref{algo:PAGD}, a perturbation will never be added in 
the remaining iterations, and by Lemma \ref{lem:energy_nonconvex} and 
Lemma \ref{lem:NCE_decrease} we know the Hamiltonian always decreases 
for the remaining steps. Therefore, if at least one NCE step is performed 
in iteration $\tau \in [0, \utime]$, by Lemma \ref{lem:NCE_decrease} we 
will decrease $2\ufun$ in that NCE step, and at most increase by $\ufun$ 
due to the perturbation. This immediately gives $E_\utime -E_0 \le -\ufun$.

Therefore, we only need to focus on the case where NCE is never used in 
iterations $\tau \in [0, \utime]$.
% Since we adding perturbation in first iteration, let $\tilde{\x_0}$ be the iterate after adding perturbation, and $\tilde{E}_0$ be its corresponding Hamiltonion. The actual decrease in Hamiltonian can be decompose into two parts:
% \begin{equation*}
% E_{\utime} - E_0
% = E_{\utime} - \tilde{E}_0 + \tilde{E}_0 - E_0
% \end{equation*}
% since we only add perturbation to $\x$ not momentum, we have:
% \begin{equation*}
% \tilde{E}_0 - E_0
% = f(\tilde{\x}_0) - f(\x_0)  \le \epsilon r + \frac{\ell r^2}{2} \le \ufun
% \end{equation*}
Let $\mathbb{B}_{\x_0}(r)$ denote the ball with radius $r$ around $\x_0$. According to algorithm \ref{algo:PAGD}, we know the iterate after adding perturbation to $\x_0$ is uniformly sampled from the ball $\mathbb{B}_{\x_0}(r)$. Let $\mathcal{X}_{\text{stuck}} \subset \mathbb{B}_{\x_0}(r)$ be the region where AGD is stuck (does not decrease the
Hamiltonian $\ufun$ in $\utime$ steps).
Formally, for any point $\x \in \mathcal{X}_{\text{stuck}}$, let $\x_1, \cdots, \x_\utime$ be the AGD sequence starting at $(\x, \v_0)$, then $E_\utime - E_0 \ge -\ufun$. By Lemma \ref{lem:2nd_seq}, $\mathcal{X}_{\text{stuck}}$ can have at most width 
$r_0 = \frac{\delta\ufun }{2\Delta_f}\cdot\frac{r}{\sqrt{d}}$ along the
minimum eigenvalue direction. Therefore,
\begin{align*}
\frac{\text{Vol}(\cXs)}{\text{Vol}(\ball^{(d)}_{\x_0}(r))}
\le \frac{r_0 \times \text{Vol}(\ball^{(d-1)}_0(r))}{\text{Vo{}l} (\ball^{(d)}_0(r))}
= \frac{r_0}{r\sqrt{\pi}}\frac{\Gamma(\frac{d}{2}+1)}{\Gamma(\frac{d}{2}+\frac{1}{2})}
\le \frac{r_0}{r\sqrt{\pi}} \cdot \sqrt{\frac{d}{2}+\frac{1}{2}} \le \frac{\delta\ufun }{2\Delta_f}.
\end{align*}
% Since $\delta$ only comes in dependence on logarithmic factor of $\utime$, we say it is high probability.
Thus, with probability at least $1-\frac{\delta\ufun }{\Delta_f}$, the
perturbation will end up outside of $\cXs$, which give $E_{\utime} - E_0\le -\ufun$. 
This finishes the proof.

\end{proof}

\subsection{Proof of Theorem \ref{thm:main}}

Our main result is now easily obtained from Lemma \ref{lem:largeGrad} and Lemma \ref{lem:negHess}.

\begin{proof}[Proof of Theorem \ref{thm:main}]
Suppose we never encounter any \ESSP. Consider the set $\mathfrak{T} = \{\tau | \tau \in [0, \utime] \text{~and~}
\norm{\grad f(\x_\tau)} \le \epsilon\}$, and two cases: (1) $\mathfrak{T} = \varnothing$, in which case we know all gradients are large and by Lemma \ref{lem:largeGrad} we have $E_{\utime} - E_0 \le -\ufun$;
(2) $\mathfrak{T} \neq \varnothing$.  In this case, define $\tau' = \min \mathfrak{T}$; i.e., the earliest iteration where the gradient is small. Since by assumption, $\x_\tau'$ is not an \ESSP, this gives $\hess f(\x_{\tau'}) \le - \sqrt{\rho\epsilon}$, and by Lemma \ref{lem:negHess}, we can conclude $E_{\tau'+\utime} - E_0 \le
E_{\tau'+\utime} - E_{\tau'} \le -\ufun$. Clearly $\tau'+\utime \le 2\utime$. That is, in either case, we will decrease the Hamiltonian by $\ufun$ in at most $2\utime$ steps.

Then, for the the first case, we can repeat this argument starting at iteration $\utime$, and for the second case, we can repeat the argument starting at iteration $\tau'+\utime$. Therefore, we will continue to obtain a decrease of the Hamiltonian by an average of $\ufun/(2\utime)$ per step. Since the function $f$ is lower bounded, we know the Hamiltonian can not decrease beyond $E_0 - E^\star = f(\x_0) - f^\star$, which means that in $\frac{2(f(\x_0) - f^\star)\utime}{\ufun}$ steps, we must encounter an \ESSP~at least once.

Finally, in $\frac{2(f(\x_0) - f^\star)\utime}{\ufun}$ steps, we will call Lemma \ref{lem:negHess} at most $\frac{2\Delta_f}{\ufun}$ times, and since Lemma \ref{lem:negHess} holds with probability $1-\frac{\delta \ufun}{2\Delta_f}$, by a union bound, we know that the argument above is true with probability at least:
$$1-\frac{\delta \ufun}{2\Delta_f}\cdot \frac{2\Delta_f}{\ufun} = 1-\delta,$$
which finishes the proof.
% two cases: 1) If $\norm{\grad f(\x_\tau)} \ge \epsilon$ for all $ \tau \in [0, \utime]$ then by Lemma \ref{lem:largeGrad}, Hamiltonian decrease by $\Omega(\ufun)$ in $\utime$ steps.
% 2) If there is $\tau'\in [0, \utime]$ such that $\norm{\grad f(\x_{\tau'})} \le \epsilon$. Let $\tau'$ be the earlest such time where since gradient small $\x_{\tau'}$ must has negative curvature in Hessian, by Lemma \ref{lem:negHess}, we know Hamiltonian decrease $\Omega(\ufun)$ in next $\utime$ iterations.
% Overall, as long as all $\x_t$ are not \ESSP, we always guarantees decrease in Hamiltonian by $\Omega(\ufun)$
% in at most $2\utime$ iterations. Since function $f$ is lower bounded, we know Hamiltonian can not decrease beyond $E_0 - E^\star = f(\x_0) - f^\star$, this means in $O(\frac{(f(\x_0) - f^\star)\utime}{\ufun})$ we must encounter \ESSP at least once, which finishes the proof.
\end{proof}

\section{Auxiliary Lemma}
In this section, we present some auxiliary lemmas which are used 
in proving Lemma \ref{lem:largegrad_nonconvex}, Lemma \ref{lem:largegrad_convex} and Lemma \ref{lem:2nd_seq}.  These deal with the large-gradient scenario (nonconvex component), the large-gradient scenario (strongly convex component), and the negative curvature scenario, respectively.

The first two lemmas establish some facts about powers of the structured 
matrices arising in~\nag.
\begin{lemma}\label{lem:aux_matrix_form}
Let the $2\times 2$ matrix $\A$ have following form, for arbitrary $a, b\in \R$:
\begin{equation*}
\A = \pmat{ a &  b \\1 & 0}.
\end{equation*}
Letting $\mu_1, \mu_2$ denote the two eigenvalues of $\A$ (can be repeated or complex eigenvalues), then, for any $t\in \N$:
\begin{align*}
\pmat{1 & 0 } \A^t =&
\left(\sum_{i=0}^t \mu_1^i \mu_2^{t-i}, \quad - \mu_1\mu_2\sum_{i=0}^{t-1} \mu_1^{i} \mu_2^{t-1-i}\right)\\
\pmat{0 & 1 } \A^t =& \pmat{1 & 0 } \A^{t-1}.
\end{align*}
% \begin{equation*}
% \pmat{1 & 0 }\A^t\pmat{1 \\ 0 }
% = \sum_{i=0}^t \mu_1^i \mu_2^{t-i}
% \quad\text{~and~}\quad
% \pmat{1 & 0 }\A^t\pmat{0 \\ -1 }
% = \mu_1\mu_2\sum_{i=0}^{t-1} \mu_1^{i} \mu_2^{t-1-i}
% \end{equation*}
\end{lemma}
\begin{proof}
When the eigenvalues $\mu_1$ and $\mu_2$ are distinct, the matrix $\A$ 
can be rewritten as $\pmat{\mu_1+\mu_2 & -\mu_1\mu_2 \\ 1 & 0 }$, and it 
is easy to check  that the two eigenvectors have the form 
$\pmat{\mu_1 \\ 1}$ and $\pmat{\mu_2 \\ 1}$. Therefore, we can 
write the eigen-decomposition as:
\begin{equation*}
\A = \frac{1}{\mu_1 - \mu_2} \pmat{\mu_1 & \mu_2 \\1 & 1}
\pmat{\mu_1 & 0 \\ 0 & \mu_2}
\pmat{1 & -\mu_2 \\ -1 & \mu_1 },
\end{equation*}
and the $t$th power has the general form:
\begin{equation*}
\A^t = \frac{1}{\mu_1 - \mu_2} \pmat{\mu_1 & \mu_2 \\1 & 1}
\pmat{\mu_1^t & 0 \\ 0 & \mu_2^t}
\pmat{1 & -\mu_2 \\ -1 & \mu_1 }
\end{equation*}

When there are two repeated eigenvalue $\mu_1$, the matrix 
$\pmat{a & b \\ 1 & 0}$ can be rewritten as $\pmat{
2\mu_1 & -\mu_1^2 \\ 1 & 0 
}$. It is easy to check that $\A$ has the following Jordan normal form:
\begin{equation*}
\A = - \pmat{\mu_1 & \mu_1+1 \\1 & 1}\pmat{\mu_1 & 1 \\0 & \mu_1}
\pmat{1 & -(\mu_1 + 1) \\-1 & \mu_1},
\end{equation*}
which yields:
\begin{equation*}
\A^t = - \pmat{\mu_1 & \mu_1+1 \\1 & 1}
\pmat{\mu^t_1 & t\mu_1^{t-1} \\0 & \mu^t_1}
\pmat{1 & -(\mu_1 + 1) \\-1 & \mu_1}.
\end{equation*}

The remainder of the proof follows from simple linear algebra 
calculations for both cases.
\end{proof}

\begin{lemma}\label{lem:aux_matrix_equality}
Under the same setting as Lemma \ref{lem:aux_matrix_form}, for any $t\in \N$: 
\begin{equation*}
(\mu_1 - 1)(\mu_2 - 1)
\pmat{1 & 0} \sum_{\tau = 0}^{t-1} \A^\tau \pmat{1 \\ 0}
 = 1 - \pmat{1 & 0}  \A^t \pmat{1 \\ 1}.
\end{equation*}
\end{lemma}

\begin{proof}
When $\mu_1$ and $\mu_2$ are distinct, we have:
\begin{equation*}
\pmat{1 & 0 } \A^t =
\left(\frac{\mu_1^{t+1} - \mu_2^{t+1}}{\mu_1 - \mu_2}, \quad - \frac{\mu_1\mu_2(\mu_1^t - \mu_2^t)}{\mu_1 - \mu_2}\right).
\end{equation*}
When $\mu_1, \mu_2$ are repeated, we have:
\begin{equation*}
\pmat{1 & 0 } \A^t =
\left((t+1)\mu_1^t, \quad -t \mu_1^{t+1}\right).
\end{equation*}
The remainder of the proof follows from Lemma \ref{lem:aux_geometric_power} 
and linear algebra.
\end{proof}

\noindent
The next lemma tells us when the eigenvalues of the~\nag~matrix are real and when they are complex.
\begin{lemma}\label{lem:aux_eigenvalues}
Let $\theta \in (0, \frac{1}{4}]$, $\x \in [-\frac{1}{4}, \frac{1}{4}]$ and
define the $2\times 2$ matrix $\A$ as follows:
\begin{equation*}
\A = \pmat{(2-\theta) (1 - x)&  -(1-\theta) (1 - x) \\ 1 & 0}
\end{equation*}
Then the two eigenvalues $\mu_1$ and $\mu_2$ of $\A$ are solutions of the
following equation:
\begin{equation*}
\mu^2 - (2-\theta)(1-x)\mu + (1-\theta)(1-x) = 0.
\end{equation*}
Moreover, when $x \in [-\frac{1}{4}, \frac{\theta^2}{(2-\theta)^2}]$, 
$\mu_1$ and $\mu_2$ are real numbers, and when
$x \in (\frac{\theta^2}{(2-\theta)^2}, \frac{1}{4}]$, 
$\mu_1$ and $\mu_2$ are conjugate complex numbers.
\end{lemma}

\begin{proof}
An eigenvalue $\mu$ of the matrix $\A$ must satisfy the following equation:
\begin{align*}
\det (\A - \mu \I) = \mu^2 - (2-\theta)(1-x)\mu + (1-\theta)(1-x) = 0.
\end{align*}
The discriminant is equal to
\begin{align*}
\Delta = &(2-\theta)^2(1-x)^2 - 4(1-\theta)(1-x) \\
= &(1-x)(\theta^2 - (2-\theta^2)x).
\end{align*}
Then $\mu_1$ and $\mu_2$ are real if and only if $\Delta \ge 0$, 
which finishes the proof.
\end{proof}

Finally, we need a simple lemma for geometric sums.
\begin{lemma}\label{lem:aux_geometric_power}
For any $\lambda >0$ and fixed $t$, we have:
\begin{equation*}
 \sum_{\tau = 0}^{t-1} (\tau +1) \lambda^\tau =
 \frac{1-\lambda^t}{(1-\lambda)^2} - \frac{t\lambda^t}{1-\lambda}.
 \end{equation*} 
\end{lemma}
\begin{proof}
Consider the truncated geometric series:
\begin{equation*}
 \sum_{\tau = 0}^{t-1} \lambda^\tau =
 \frac{1-\lambda^t}{1-\lambda}.
\end{equation*}
Taking derivatives, we have:
\begin{equation*}
 \sum_{\tau = 0}^{t-1} (\tau +1) \lambda^\tau = \frac{\mathrm{d}}{\mathrm{d}\lambda}\sum_{\tau = 0}^{t-1} \lambda^{\tau+1} =
 \frac{\mathrm{d}}{\mathrm{d}\lambda}\left[\lambda \cdot\frac{1-\lambda^t}{1-\lambda}\right]
 =\frac{1-\lambda^t}{(1-\lambda)^2} - \frac{t\lambda^t}{1-\lambda}.
\end{equation*}
\end{proof}

\subsection{Large-gradient scenario (nonconvex component)}
All the lemmas in this section are concerned with the behavior 
of the~\nag~matrix for eigen-directions of the Hessian with 
eigenvalues being negative or small and positive, as used 
in proving Lemma \ref{lem:largegrad_nonconvex}.
The following lemma bounds the smallest eigenvalue of 
the~\nag~matrix for those directions.
\begin{lemma}\label{lem:aux_nonconvex_mu2}
Under the same setting as Lemma \ref{lem:aux_eigenvalues}, 
and for $x \in [-\frac{1}{4}, \frac{\theta^2}{(2-\theta)^2}]$, 
where $\mu_1 \ge \mu_2$, we have:
\begin{equation*}
\mu_2 \le 1 - \frac{1}{2}\max\{\theta, \sqrt{|x|}\}.
\end{equation*}
\end{lemma}

\begin{proof}
The eigenvalues satisfy:
\begin{align*}
\det (\A - \mu \I) = \mu^2 - (2-\theta)(1-x)\mu + (1-\theta)(1-x) = 0.
\end{align*}
Let $\mu = 1+u$.  We have 
\begin{align*}
&& (1+u)^2 - (2-\theta)(1-x)(1+u) + (1-\theta)(1-x) &= 0 \\
\Rightarrow  && u^2 +  ((1-x)\theta + 2x) u + x&= 0.
\end{align*}
Let $f(u) = u^2 + \theta u + 2xu - x\theta u + x$.  
To prove $\mu_2(\A) \le 1 - \frac{\sqrt{|x|}}{2}$ 
when $x\in [-\frac{1}{4}, -\theta^2]$, we only need to verify $f(-\frac{\sqrt{|x|}}{2}) \le 0$:
\begin{align*}
f(-\frac{\sqrt{|x|}}{2}) = &\frac{|x|}{4} - \frac{\theta\sqrt{|x|}}{2}
+ |x|\sqrt{|x|} - \frac{|x|\sqrt{|x|}\theta}{2} - |x| \\
\le & |x|\sqrt{|x|}(1-\frac{\theta}{2}) - \frac{3|x|}{4} \le 0
\end{align*}
The last inequality follows because $|x| \le \frac{1}{4}$ by assumption.

For $x\in [-\theta^2, 0]$, we have:
\begin{align*}
f(-\frac{\theta}{2}) = \frac{\theta^2}{4} -\frac{\theta^2}{2} - x\theta
+ \frac{x\theta^2}{2} + x
= -\frac{\theta^2}{4} + x(1-\theta) + \frac{x\theta^2}{2} \le 0.
\end{align*}
On the other hand, when $x \in [0, \theta^2/(2-\theta)^2]$, both 
eigenvalues are still real, and the midpoint of the two roots is:
\begin{align*}
\frac{u_1 + u_2}{2} = -\frac{(1-x)\theta + 2x}{2}
=-\frac{\theta + (2-\theta)x}{2}
\le -\frac{\theta}{2}.
\end{align*}
Combining the two cases, we have shown that when 
$x \in [-\theta^2, \theta^2/(2-\theta)^2]$ we have 
$\mu_2 (\A) \le 1-\frac{\theta}{2}$.

In summary, we have proved that
\begin{equation*}
\mu_2(\A) \le 
\begin{cases}
1 - \frac{\sqrt{|x|}}{2}, & x \in  [-\frac{1}{4}, -\theta^2]\\
1-\frac{\theta}{2}.& x \in  [-\theta^2, \theta^2/(2-\theta)^2],
\end{cases}
\end{equation*}
which finishes the proof.
\end{proof}
\noindent
In the same setting as above, the following lemma bounds the largest eigenvalue.
\begin{lemma}\label{lem:aux_nonconvex_mu1}
Under the same setting as Lemma \ref{lem:aux_eigenvalues}, and with
$x \in [-\frac{1}{4}, \frac{\theta^2}{(2-\theta)^2}]$, and letting
$\mu_1 \ge \mu_2$, we have:
\begin{equation*}
\mu_1 \le 1 + 2 \min\{\frac{|x|}{\theta}, \sqrt{|x|}\}.
\end{equation*}
\end{lemma}
\begin{proof}
By Lemma \ref{lem:aux_eigenvalues} and Vieta's formula， we have:
$$(\mu_1 - 1)(\mu_2-1) = \mu_1\mu_2 - (\mu_1 + \mu_2) + 1 = x.$$
An application of Lemma \ref{lem:aux_nonconvex_mu2} finishes the proof.
\end{proof}
The following lemma establishes some properties of the powers of the~\nag~matrix.
\begin{lemma}\label{lem:aux_nonconvex_inequal}
Consider the same setting as Lemma \ref{lem:aux_eigenvalues}, and let
$x \in [-\frac{1}{4}, \frac{\theta^2}{(2-\theta)^2}]$.
Denote:
\begin{equation*}
(a_t, ~-b_t) = \pmat{1 & 0 } \A^t.
\end{equation*}
Then, for any $t \ge \frac{2}{\theta} + 1$, we have:
\begin{align*}
\sum_{\tau = 0}^{t-1} a_\tau \ge& \Omega (\frac{1}{\theta^2}) \\
\frac{1}{b_t}\left(\sum_{\tau = 0}^{t-1} a_\tau \right) \ge& \Omega(1)\min\left\{\frac{1}{\theta}, \frac{1}{\sqrt{|x|}}\right\}.
\end{align*}
\end{lemma}

\begin{proof}
We prove the two inequalities seperately.

\noindent \textbf{First Inequality:}
By Lemma \ref{lem:aux_matrix_form}:
\begin{align*}
\sum_{\tau = 0}^t \pmat{ 1 & 0 } \A^\tau \pmat{ 1 \\ 0 }
% = &\sum_{\tau = 0}^t\frac{\mu_1^{\tau+1} - \mu_2^{\tau+1}}{\mu_1 - \mu_2}
= &\sum_{\tau=0}^t \sum_{i=0}^\tau \mu_1^{\tau-i}\mu_2^{i} 
= \sum_{\tau=0}^t (\mu_1\mu_2)^{\frac{\tau}{2}}\sum_{i=0}^\tau (\frac{\mu_1}{\mu_2})^{\frac{\tau}{2} - i} \\
\ge& \sum_{\tau=0}^t [(1-\theta)(1-x)]^{\frac{\tau}{2}} \cdot \frac{\tau}{2}
\end{align*}
The last inequality holds because in $\sum_{i=0}^\tau (\frac{\mu_1}{\mu_2})^{\frac{\tau}{2}-i}$ at least $\frac{\tau}{2}$ terms are greater than one. 
Finally, since $x \le \theta^2/(2-\theta)^2 \le \theta^2\le \theta$, we have $1-x \ge 1-\theta$, thus:
\begin{align*}
\sum_{\tau=0}^t [(1-\theta)(1-x)]^{\frac{\tau}{2}} \cdot \frac{\tau}{2}
\ge & \sum_{\tau=0}^t (1-\theta)^{\tau} \cdot \frac{\tau}{2}
\ge \sum_{\tau=0}^{1/\theta} (1-\theta)^{\tau} \cdot \frac{\tau}{2}\\
\ge & (1-\theta)^{\frac{1}{\theta}}\sum_{\tau=0}^{1/\theta}  \frac{\tau}{2}
\ge \Omega(\frac{1}{\theta^2}),
\end{align*}
which finishes the proof.

\noindent \textbf{Second Inequality:}
Without loss of generality, assume $\mu_1 \ge \mu_2$. 
Again by Lemma \ref{lem:aux_matrix_form}:
\begin{align*}
\frac{\sum_{\tau = 0}^{t-1} a_\tau}{b_t}
=& \frac{\sum_{\tau = 0}^{t-1} \sum_{i=0}^{\tau} \mu_1^i \mu_2^{\tau-i}}
{\mu_1\mu_2\sum_{i=0}^{t-1} \mu_1^{i} \mu_2^{t-1-i}} 
= \frac{1}{\mu_1\mu_2}
\sum_{\tau = 0}^{t-1} \frac{\sum_{i=0}^{\tau} \mu_1^i \mu_2^{\tau-i}}
{\sum_{i=0}^{t-1} \mu_1^{i} \mu_2^{t-1-i}} \\
\ge & \frac{1}{\mu_1\mu_2}
\sum_{\tau = (t-1)/2}^{t-1} \frac{\sum_{i=0}^{\tau} \mu_1^i \mu_2^{\tau-i}}
{\sum_{i=0}^{t-1} \mu_1^{i} \mu_2^{t-1-i}} 
\ge  \frac{1}{\mu_1\mu_2} \sum_{\tau = (t-1)/2}^{t-1} \frac{1}{2 \mu_1^{t-1-\tau}} \\
= & \frac{1}{2\mu_1\mu_2} \left[1 + \frac{1}{\mu_1} + \cdots + \frac{1}{\mu_1^{(t-1)/2}}\right] 
\ge  \frac{1}{2\mu_1\mu_2} \left[1 + \frac{1}{\mu_1} + \cdots + \frac{1}{\mu_1^{1/\theta}}\right].
\end{align*}
The second-to-last inequality holds because it is easy to check
\begin{equation*}
2 \mu_1^{t-1-\tau}  \sum_{i=0}^{\tau} \mu_1^i \mu_2^{\tau-i} \ge \sum_{i=0}^{t-1} \mu_1^{i} \mu_2^{t-1-i},
\end{equation*}
for any $\tau \ge (t-1)/2$. Finally, by Lemma \ref{lem:aux_nonconvex_mu1}, we have
\begin{equation*}
\mu_1 \le 1 + 2\min \{\frac{|x|}{\theta}, \sqrt{|x|} \}.
\end{equation*}
Since $\mu_1 = \Theta(1)$, $\mu_2 = \Theta(1)$, we have that when $|x| \le \theta^2$, 
\begin{equation*}
\frac{\sum_{\tau = 0}^{t-1} a_\tau}{b_t} 
\ge \Omega(1) \left[1 + \frac{1}{\mu_1} + \cdots + \frac{1}{\mu_1^{1/\theta}}\right]
\ge \Omega(1) \cdot \frac{1}{\theta} \cdot \frac{1}{(1+\theta)^{\frac{1}{\theta}}}
\ge \Omega(\frac{1}{\theta}).
\end{equation*}
When $|x| > \theta^2$, we have:
\begin{equation*}
\frac{\sum_{\tau = 0}^{t-1} a_\tau}{b_t} 
\ge \Omega(1) \left[1 + \frac{1}{\mu_1} + \cdots + \frac{1}{\mu_1^{1/\theta}}\right]
= \Omega(1)
\frac{1 - \frac{1}{\mu_1^{1/\theta + 1}}}{1- \frac{1}{\mu_1}}
=\Omega(\frac{1}{\mu_1 - 1}) = \Omega(\frac{1}{\sqrt{|x|}}).
\end{equation*}
Combining the two cases finishes the proof.
\end{proof}

\subsection{Large-gradient scenario (strongly convex component)}
All the lemmas in this section are concerned with the behavior 
of the~\nag~matrix for eigen-directions of the Hessian with eigenvalues 
being large and positive, as used in proving Lemma \ref{lem:largegrad_convex}.
The following lemma gives eigenvalues of the~\nag~matrix for those directions.

\begin{lemma}\label{lem:aux_convex_rphi}
Under the same setting as Lemma \ref{lem:aux_eigenvalues}, and 
with $x \in (\frac{\theta^2}{(2-\theta)^2}, \frac{1}{4}]$, we have
$\mu_1 = r e^{i\phi}$ and $\mu_2 = r e^{-i\phi}$, where:
\begin{equation*}
r = \sqrt{(1-\theta)(1-x)}, \quad\quad \sin{\phi} = \sqrt{((2-\theta)^2x - \theta^2)(1-x)}/2r.
\end{equation*}
\end{lemma}
\begin{proof}
By Lemma \ref{lem:aux_eigenvalues}, we know that $\mu_1$ and $\mu_2$ 
are two solutions of 
\begin{equation*}
\mu^2 - (2-\theta)(1-x)\mu + (1-\theta)(1-x) = 0.
\end{equation*}
This gives $r^2 = \mu_1\mu_2 = (1-\theta)(1-x)$. 
On the other hand, discriminant is equal to
\begin{align*}
\Delta =& (2-\theta)^2(1-x)^2 - 4(1-\theta)(1-x) \\
=& (1-x)(\theta^2 - (2-\theta^2)x).
\end{align*}
Since $\Im(\mu_1) = r \sin \phi = \frac{\sqrt{-\Delta}}{2}$, 
the proof is finished.
\end{proof}

Under the same setting as above, the following lemma delineates some 
properties of powers of the~\nag~matrix.
\begin{lemma}\label{lem:aux_convex_entry}
Under the same setting as in Lemma \ref{lem:aux_eigenvalues}, 
and with $x \in (\frac{\theta^2}{(2-\theta)^2}, \frac{1}{4}]$, denote:
\begin{equation*}
(a_t, ~-b_t) = \pmat{1 & 0 } \A^t.
\end{equation*}
Then, for any $t\ge 0$, we have:
\begin{equation*}
\max\{|a_t|, ~|b_t|\} \le (t+1) (1-\theta)^{\frac{t}{2}}.
\end{equation*}
\end{lemma}

\begin{proof}
By Lemma \ref{lem:aux_matrix_form} and Lemma \ref{lem:aux_convex_rphi}, 
using $|\cdot|$ to denote the magnitude of a complex number, we have:
\begin{align*}
|a_t| =& \left|\sum_{i=0}^t \mu_1^i \mu_2^{t-i}\right|
\le \sum_{i=0}^t |\mu_1^i \mu_2^{t-i}| = (t+1)r^{t} \le (t+1)(1-\theta)^{\frac{t}{2}} \\
|b_t| =& \left|\mu_1\mu_2\sum_{i=0}^{t-1} \mu_1^{i} \mu_2^{t-1-i}\right|
\le \sum_{i=0}^{t-1} |\mu_1^{i+1}\mu_2^{t-i}|
\le t r^{t+1} \le t (1-\theta)^{\frac{t+1}{2}}.
\end{align*}
Reorganizing these two equations finishes the proof.
\end{proof}

The following is a technical lemma which is useful in bounding the change in 
the Hessian by the amount of oscillation in the iterates.
\begin{lemma}\label{lem:aux_convex_trigonometry}
Under the same setting as Lemma \ref{lem:aux_convex_rphi}, for any $T\ge 0$, any sequence $\{\epsilon_t\}$, and any $\varphi_0 \in [0, 2\pi]$:
\begin{equation*}
\sum_{t=0}^{T} r^t \sin(\phi t + \varphi_0) \epsilon_t
% \le &\sum_{t=0}^{T} r^t \sin(\phi t) \epsilon_0
% + \frac{1}{\sqrt{x}} \sum_{t=1}^T |\epsilon_t - \epsilon_{t-1}| \\
\le  O(\frac{1}{\sin\phi}) \left(|\epsilon_0|+  \sum_{t=1}^T |\epsilon_t - \epsilon_{t-1}|\right).
\end{equation*}
% we also have \cnote{say something more about second statement in proof}:
% \begin{align*}
% \sum_{t=0}^{T} r^t \cos(\phi t) \epsilon_t
% \le  & O(\frac{1}{\sqrt{x}}) \left(|\epsilon_0|+  \sum_{t=1}^T |\epsilon_t - \epsilon_{t-1}|\right)
% \end{align*}
\end{lemma}

\begin{proof}
Let $\tau = \lfloor 2\pi/\phi\rfloor$ be the approximate period, and $J = \lfloor T/\tau \rfloor$ be the number of periods that exist within time $T$. Then, we can group the summation by each period:
% \begin{align*}
% \sum_{t=0}^{T} r^t \sin(\phi t) \epsilon_t =&\sum_{j=0}^{J} \left[\sum_{t = j\tau}^{\min\{(j+1)\tau-1, T\}} r^t \sin(\phi t)\right] \epsilon_{j\tau}\\
% =&\sum_{j=0}^{J - 1} \left[\sum_{t = j\tau}^{(j+1)\tau-1} r^t \sin(\phi t)\right] \epsilon_{j\tau}
% + \sum_{t = J \tau}^{T}
% \\
% = &\sum_{j=0}^{\lfloor T/\tau \rfloor} \left[\sum_{t = j\tau}^{\min\{(j+1)\tau-1, T\}} r^t \sin(\phi t)\right] (\epsilon_0 + \epsilon_{j\tau} - \epsilon_0)\\
% \le& 
% \sum_{j=0}^{\lfloor T/\tau \rfloor - 1} r^{j\tau} \frac{\theta + x}{\phi^2}(|\epsilon_0| + |\epsilon_{j\tau} - \epsilon_0|)
% + \sum_{t = \lfloor T/\tau \rfloor \tau}^{T} (|\epsilon_0| + |\epsilon_{\lfloor T/\tau \rfloor \tau} - \epsilon_0|)\\
% \le &\left[\frac{1}{1-r^\tau}\frac{\theta + x}{\phi^2} + \tau\right] \cdot
% \left[|\epsilon_0|+\sum_{t=1}^T |\epsilon_t - \epsilon_{t-1}|\right]
% \le \left[\frac{1}{\phi} + \tau\right] \cdot
% \left[|\epsilon_0|+\sum_{t=1}^T |\epsilon_t - \epsilon_{t-1}|\right]
% \end{align*}

\begin{align*}
\sum_{t=0}^{T} r^t \sin(\phi t) \epsilon_t
=& \sum_{j=0}^{J} \left[\sum_{t = j\tau}^{\min\{(j+1)\tau-1, T\}} r^t \sin(\phi t+ \varphi_0) \epsilon_t \right]\\
= &
\sum_{j=0}^{J} \left[\sum_{t = j\tau}^{\min\{(j+1)\tau-1, T\}} r^t \sin(\phi t+ \varphi_0) [\epsilon_{j\tau}
+ (\epsilon_t - \epsilon_{j\tau})] \right]\\
\le & 
\underbrace{\sum_{j=0}^{J} \left[\sum_{t = j\tau}^{\min\{(j+1)\tau-1, T\}} r^t \sin(\phi t+ \varphi_0)\right] \epsilon_{j\tau}}_{\text{Term 1}} + \underbrace{\sum_{j=0}^{J} \left[\sum_{t = j\tau}^{\min\{(j+1)\tau-1, T\}}r^t|\epsilon_t - \epsilon_{j\tau}| \right]}_{\text{Term 2}}.
\end{align*}
We prove the lemma by bounding the first term and the second term 
on the right-hand-side of this equation separately.

\noindent \textbf{Term 2:} Since $r\le 1$, it is not hard to see:
\begin{align*}
\text{Term 2} =& \sum_{j=0}^{J} \left[\sum_{t = j\tau}^{\min\{(j+1)\tau-1, T\}}r^t|\epsilon_t - \epsilon_{j\tau}|\right] \\ 
\le& \sum_{j=0}^{J} \left[\sum_{t = j\tau}^{\min\{(j+1)\tau-1, T\}}r^t\right]
\left[\sum_{t = j\tau+1}^{\min\{(j+1)\tau-1, T\}}|\epsilon_t - \epsilon_{t-1}|\right] \\
\le& 
\tau \sum_{j=0}^{J} \left[\sum_{t = j\tau+1}^{\min\{(j+1)\tau-1, T\}}|\epsilon_t - \epsilon_{t-1}|\right]
\le \tau \sum_{t = 1}^T |\epsilon_t - \epsilon_{t-1}|.
\end{align*}

\noindent \textbf{Term 1:} We first study the inner-loop factor, $\sum_{t = j\tau}^{(j+1)\tau-1} r^t \sin(\phi t)$. Letting $\psi = 2\pi - \tau \phi$ be the offset 
for each approximate period, we have that for any $j < J$:
\begin{align*}
\left|\sum_{t = j\tau}^{(j+1)\tau-1} r^t \sin(\phi t+ \varphi_0)\right|
=& \left|\Im\left[\sum_{t = 0}^{\tau-1} r^{j\tau + t} e^{i\cdot [\phi(j\tau + t)+ \varphi_0]}\right]\right|\\
\le& r^{j\tau}\norm{\sum_{t = 0}^{\tau-1} r^{t} e^{i\cdot \phi t}}
\le r^{j\tau}\norm{\frac{1- r^{\tau} e^{i\cdot (2\pi - \psi)}}{1-r e^{i\cdot \phi}}} \\
=& r^{j\tau} \sqrt{\frac{(1 - r^{\tau}\cos \psi)^2 + (r^{\tau}\sin\psi)^2}
{(1 - r\cos \phi)^2 + (r\sin\phi)^2}}.
% \le r^{j\tau} O(\frac{\theta + x}{\phi^2})
\end{align*}
Combined with the fact that for all $y\in[0, 1]$ we have $e^{-3y} \le 1-y \le e^{-y} $, we obtain the following:
\begin{equation}\label{eq:aux_r_tau}
1 - r^{\tau} = 1 - [(1-\theta)(1-x)]^{\frac{\tau}{2}}
= 1 - e^{-\Theta((\theta+x)\tau)} = \Theta ((\theta + x)\tau)
= \Theta \left(\frac{(\theta + x)}{\phi}\right)
\end{equation}
Also, for any $a, b\in [0, 1]$, we have $
(1-ab)^2 \le (1-\min\{a, b\})^2 \le (1-a^2)^2 + (1-b^2)^2$, and by definition of $\tau$, we immediately have $\psi \le \phi$. This yields:
\begin{align*}
\frac{(1 - r^{\tau}\cos \psi)^2 + (r^{\tau}\sin\psi)^2}
{(1 - r\cos \phi)^2 + (r\sin\phi)^2}
\le& \frac{2(1 - r^{2\tau})^2+  2(1-\cos^2 \psi)^2 + (r^{\tau}\sin\psi)^2}
{(r\sin\phi)^2} \\
\le&O\left(\frac{1}{\sin^2 \phi}\right) \left[\frac{(\theta + x)^2}{\phi^2}  +  \sin^4\phi + \sin^2 \phi\right]
\le O\left(\frac{(\theta + x)^2}{\sin^4 \phi}\right)
% \approx O(1)\frac{(\theta + x)^2}{\phi^4}
\end{align*}
The second last inequality used the fact that $r = \Theta(1)$ (although note $r^{\tau}$ is not $\Theta(1)$).
The last inequality is true since by Lemma \ref{lem:aux_convex_rphi}, we know
$(\theta+x)/ \sin^2 \phi \ge \Omega(1)$. This gives:
$$\left|\sum_{t = j\tau}^{(j+1)\tau-1} r^t \sin(\phi t + \varphi_0)\right|
\le r^{j\tau}  \cdot \frac{\theta + x}{\sin^2\phi},$$
and therefore, we can now bound the first term:
\begin{align*}
\text{Term 1} =& \sum_{j=0}^{J} \sum_{t = j\tau}^{\min\{(j+1)\tau-1, T\}} r^t \sin(\phi t+ \varphi_0) \epsilon_{j\tau} 
= \sum_{j=0}^{J} \left[\sum_{t = j\tau}^{\min\{(j+1)\tau-1, T\}} r^t \sin(\phi t+ \varphi_0)\right] (\epsilon_0 + \epsilon_{j\tau} - \epsilon_0)\\
\le&  O(1)\sum_{j=0}^{J - 1} \left[r^{j\tau} \frac{\theta + x}{\sin^2\phi}\right](|\epsilon_0| + |\epsilon_{j\tau} - \epsilon_0|)
+ \sum_{t = J\tau}^{T} (|\epsilon_0| + |\epsilon_{J \tau} - \epsilon_0|)\\
\le& O(1) \left[\frac{1}{1-r^\tau}\frac{\theta + x}{\sin^2\phi} + \tau\right] \cdot
\left[|\epsilon_0|+\sum_{t=1}^T |\epsilon_t - \epsilon_{t-1}|\right]
\le \left[O(\frac{1}{\sin\phi}) + \tau\right] \cdot
\left[|\epsilon_0|+\sum_{t=1}^T |\epsilon_t - \epsilon_{t-1}|\right].
\end{align*}
The second-to-last inequality used Eq.\eqref{eq:aux_r_tau}. In conclusion, since $\tau \le \frac{2\pi}{\phi} \le \frac{2\pi}{\sin\phi}$, we have:
\begin{align*}
\sum_{t=0}^{T} r^t \sin(\phi t+ \varphi_0) \epsilon_t
\le& \text{Term 1} + \text{Term 2}
\le \left[O(\frac{1}{\sin\phi}) + 2\tau\right] \cdot 
\left[|\epsilon_0|+\sum_{t=1}^T |\epsilon_t - \epsilon_{t-1}|\right] \\
\le& O\left(\frac{1}{\sin\phi}\right)\left[|\epsilon_0|+\sum_{t=1}^T |\epsilon_t - \epsilon_{t-1}|\right].
\end{align*}
\end{proof}

\noindent
The following lemma combines the previous two lemmas to bound the
approximation error in the quadratic.
\begin{lemma}\label{lem:aux_convex_inequal}
Under the same setting as Lemma \ref{lem:aux_eigenvalues}, and 
with $x \in (\frac{\theta^2}{(2-\theta)^2}, \frac{1}{4}]$, denote:
\begin{equation*}
(a_t, ~-b_t) = \pmat{1 & 0 } \A^t.
\end{equation*}
Then, for any sequence $\{\epsilon_\tau\}$, any $t \ge \Omega(\frac{1}{\theta})$, we have:
\begin{align*}
\sum_{\tau = 0}^{t-1} a_\tau  \epsilon_\tau \le& O(\frac{1}{x})\left(|\epsilon_0| + \sum_{\tau = 1}^{t-1}|\epsilon_\tau - \epsilon_{\tau -1}|\right)\\
\sum_{\tau = 0}^{t-1} (a_\tau - a_{\tau-1})  \epsilon_\tau \le& O(\frac{1}{\sqrt{x}})\left(|\epsilon_0| + \sum_{\tau = 1}^{t-1}|\epsilon_\tau - \epsilon_{\tau -1}|\right).
\end{align*}
\end{lemma}

\begin{proof}
We prove the two inequalities separately.

\noindent \textbf{First Inequality:} Since $x \in (\frac{\theta^2}{(2-\theta)^2}, 
\frac{1}{4}]$, we further split the analysis into two cases:

\noindent \textbf{Case $x \in (\frac{\theta^2}{(2-\theta)^2}, \frac{2\theta^2}{(2-\theta)^2}]$:}
By Lemma \ref{lem:aux_matrix_form}, we can expand dthe left-hand-side as:
\begin{equation*}
\sum_{\tau = 0}^{t-1} a_\tau  \epsilon_\tau
\le \sum_{\tau = 0}^{t-1} |a_\tau|  (|\epsilon_0| + |\epsilon_\tau - \epsilon_0|)
\le \left[\sum_{\tau = 0}^{t-1} |a_\tau|\right]  \left(|\epsilon_0| + \sum_{\tau = 1}^{t-1}|\epsilon_\tau - \epsilon_{\tau -1}|\right).
\end{equation*}
Noting that in this case $x = \Theta(\theta^2)$, by Lemma \ref{lem:aux_convex_entry} and Lemma \ref{lem:aux_geometric_power}, we have for $t \ge O(1/\theta)$:
\begin{equation*}
\sum_{\tau = 0}^{t-1} |a_\tau|
 \le \sum_{\tau = 0}^{t-1} (\tau+1)(1-\theta)^{\frac{\tau}{2}}
 \le O(\frac{1}{\theta^2}) = O(\frac{1}{x}).
\end{equation*}

\noindent \textbf{Case $x \in (\frac{2\theta^2}{(2-\theta)^2}, \frac{1}{4}]$:}
Again, we expand the left-hand-side as:
\begin{equation*}
\sum_{\tau = 0}^{t-1} a_\tau  \epsilon_\tau
= \sum_{\tau =0}^{t-1} \frac{\mu_1^{\tau+1} -  \mu_2^{\tau+1}}{\mu_1 -\mu_2}\epsilon_\tau
= \sum_{\tau =0}^{t-1}\frac{r^{\tau+1}\sin[(\tau+1)\phi]}{r\sin[\phi]} \epsilon_\tau.
\end{equation*}
Noting in this case that $x = \Theta(\sin^2\phi)$ by Lemma \ref{lem:aux_convex_rphi}, then by Lemma \ref{lem:aux_convex_trigonometry} we have:
\begin{equation*}
\sum_{\tau = 0}^{t-1} a_\tau  \epsilon_\tau 
\le O(\frac{1}{\sin^2\phi})\left(|\epsilon_0| + \sum_{\tau = 1}^{t-1}|\epsilon_\tau - \epsilon_{\tau -1}|\right)
\le O(\frac{1}{x})\left(|\epsilon_0| + \sum_{\tau = 1}^{t-1}|\epsilon_\tau - \epsilon_{\tau -1}|\right).
\end{equation*}

\noindent\textbf{Second Inequality:} 
Using Lemma \ref{lem:aux_matrix_form}, we know:
\begin{align*}
a_\tau - a_{\tau-1} 
=&  \frac{(\mu_1^{\tau+1} -  \mu_2^{\tau+1}) - (\mu_1^{\tau} -  \mu_2^{\tau}) }{\mu_1 -\mu_2} \\
=& \frac{r^{\tau+1}\sin[(\tau+1)\phi]
- r^{\tau}\sin[\tau\phi]}{r\sin[\phi]} \\
=& \frac{r^{\tau}\sin[\tau\phi](r\cos\phi - 1)
+ r^{\tau+1}\cos[\tau\phi]\sin\phi}{r\sin[\phi]} \\
=& \frac{r\cos\phi - 1}{r\sin\phi} \cdot r^{\tau}\sin[\tau\phi]  + r^{\tau}\cos[\tau\phi],
\end{align*}
where we note $r = \Theta(1)$ and the coefficient of the first 
term is upper bounded by the following:
\begin{equation*}
\left|\frac{r\cos\phi - 1}{r\sin\phi}\right|
\le \frac{(1-\cos^2\phi) + (1-r^2)}{r\sin\phi}
\le O\left(\frac{\theta + x}{\sin\phi}\right).
\end{equation*}
As in the proof of the first inequality, we split the analysis 
into two cases:

\noindent \textbf{Case $x \in (\frac{\theta^2}{(2-\theta)^2}, \frac{2\theta^2}{(2-\theta)^2}]$:} Again, we use
\begin{equation*}
\sum_{\tau = 0}^{t-1} (a_\tau - a_{\tau-1} )   \epsilon_\tau
\le \sum_{\tau = 0}^{t-1} |a_\tau- a_{\tau-1} |  (|\epsilon_0| + |\epsilon_\tau - \epsilon_0|)
\le \left[\sum_{\tau = 0}^{t-1} |a_\tau- a_{\tau-1} |\right]  \left(|\epsilon_0| + \sum_{\tau = 1}^{t-1}|\epsilon_\tau - \epsilon_{\tau -1}|\right).
\end{equation*}
Noting $x = \Theta(\theta^2)$, again by Lemma \ref{lem:aux_geometric_power} and $|\frac{\sin \tau\phi}{\sin \phi}| \le \tau$, we have:
\begin{equation*}
\left[\sum_{\tau = 0}^{t-1} |a_\tau- a_{\tau-1} |\right]
\le O(\theta + x)\sum_{\tau = 0}^{t-1}\tau (1-\theta)^{\frac{\tau}{2}} + \sum_{\tau = 0}^{t-1} (1-\theta)^{\frac{\tau}{2}}
\le O(\frac{1}{\theta}) = O(\frac{1}{\sqrt{x}}).
\end{equation*}

\noindent \textbf{Case $x \in (\frac{2\theta^2}{(2-\theta)^2}, \frac{1}{4}]$:}
From the above derivation, we have:
\begin{align*}
\sum_{\tau = 0}^{t-1} (a_\tau - a_{\tau-1})  \epsilon_\tau
% =& \sum_{\tau =0}^{t-1} \frac{(\mu_1^{\tau+1} -  \mu_2^{\tau+1}) - (\mu_1^{\tau} -  \mu_2^{\tau}) }{\mu_1 -\mu_2}\epsilon_\tau \\
% =& \sum_{\tau =0}^{t-1}\frac{r^{\tau+1}\sin[(\tau+1)\phi]
% - r^{\tau}\sin[\tau\phi]}{r\sin[\phi]} \epsilon_\tau \\
% =& \sum_{\tau =0}^{t-1}\frac{r^{\tau}\sin[\tau\phi](r\cos\phi - 1)
% + r^{\tau+1}\cos[\tau\phi]\sin\phi}{r\sin[\phi]} \epsilon_\tau\\
=& \frac{r\cos\phi - 1}{r\sin\phi} \sum_{\tau =0}^{t-1}r^{\tau}\sin[\tau\phi]\epsilon_\tau  + \sum_{\tau =0}^{t-1}r^{\tau}\cos[\tau\phi]\epsilon_\tau.
\end{align*}
According to Lemma \ref{lem:aux_convex_rphi}, in this case $ x = \Theta(\sin^2\phi)$, $r = \Theta(1)$ and since $\Omega(\theta^2)\le x \le O(1)$, we have:
\begin{equation*}
\left|\frac{r\cos\phi - 1}{r\sin\phi}\right|
\le O\left(\frac{\theta + x}{\sin\phi}\right)
\le O\left(\frac{\theta + x}{\sqrt{x}}\right) \le O(1).
\end{equation*}
Combined with Lemma \ref{lem:aux_convex_trigonometry}, this gives:
\begin{equation*}
\sum_{\tau = 0}^{t-1} (a_\tau - a_{\tau-1})  \epsilon_\tau
\le O(\frac{1}{\sin\phi})\left(|\epsilon_0| + \sum_{\tau = 1}^{t-1}|\epsilon_\tau - \epsilon_{\tau -1}|\right)
\le O(\frac{1}{\sqrt{x}})\left(|\epsilon_0| 
+ \sum_{\tau = 1}^{t-1}|\epsilon_\tau - \epsilon_{\tau -1}|\right).
\end{equation*}

% \begin{align*}
% \sum_{\tau = 0}^{t-1} (a_\tau - a_{\tau-1})  \epsilon_\tau
% =& \sum_{\tau =0}^{t-1} \frac{(\mu_1^{\tau+1} -  \mu_2^{\tau+1}) - (\mu_1^{\tau} -  \mu_2^{\tau}) }{\mu_1 -\mu_2}\epsilon_\tau \\
% =& \sum_{\tau =0}^{t-1}\frac{r^{\tau+1}\sin[(\tau+1)\phi]
% - r^{\tau}\sin[\tau\phi]}{r\sin[\phi]} \epsilon_\tau \\
% =& \sum_{\tau =0}^{t-1}\frac{r^{\tau}\sin[\tau\phi](r\cos\phi - 1)
% + r^{\tau+1}\cos[\tau\phi]\sin\phi}{r\sin[\phi]} \epsilon_\tau\\
% =& \frac{r\cos\phi - 1}{r\sin\phi} \sum_{\tau =0}^{t-1}r^{\tau}\sin[\tau\phi]\epsilon_\tau  + \sum_{\tau =0}^{t-1}r^{\tau}\cos[\tau\phi]\epsilon_\tau
% \end{align*}
% According to Lemma \ref{lem:aux_convex_rphi}, in this case $ x = \Theta(\sin^2\phi)$, $r = \Theta(1)$ and $ \Omega(\theta^2)\le x \le O(1)$, we have:
% \begin{equation*}
% \left|\frac{r\cos\phi - 1}{r\sin\phi}\right|
% \le \frac{(1-\cos^2\phi) + (1-r^2)}{r\sin\phi}
% \le O\left(\frac{\theta + x}{\sqrt{x}}\right) \le O(1)
% \end{equation*}
% Combined with Lemma \ref{lem:aux_convex_trigonometry}, this gives:
% \begin{equation*}
% \sum_{\tau = 0}^{t-1} (a_\tau - a_{\tau-1})  \epsilon_\tau
% \le O(\frac{1}{\sin\phi})\left(|\epsilon_0| + \sum_{\tau = 1}^{t-1}|\epsilon_\tau - \epsilon_{\tau -1}|\right)
% \le O(\frac{1}{\sqrt{x}})\left(|\epsilon_0| 
% + \sum_{\tau = 1}^{t-1}|\epsilon_\tau - \epsilon_{\tau -1}|\right)
% \end{equation*}

\noindent
Putting all the pieces together finishes the proof.
\end{proof}

\subsection{Negative-curvature scenario}
In this section, we will prove the auxiliary lemmas required for 
proving Lemma \ref{lem:2nd_seq}. 

The first lemma lower bounds the largest eigenvalue of 
the~\nag~matrix for eigen-directions whose eigenvalues are negative.
\begin{lemma}\label{lem:aux_negcurve_mu1}
Under the same setting as Lemma \ref{lem:aux_eigenvalues}, and 
with $x \in [-\frac{1}{4}, 0]$, and $\mu_1 \ge \mu_2$, we have:
\begin{equation*}
\mu_1 \ge 1 + \frac{1}{2}\min\{\frac{|x|}{\theta}, \sqrt{|x|}\}.
\end{equation*}
\end{lemma}

\begin{proof}
The eigenvalues satisfy:
\begin{align*}
\det (\A - \mu \I) = \mu^2 - (2-\theta)(1-x)\mu + (1-\theta)(1-x) = 0.
\end{align*}
Let $\mu = 1+u$. We have 
\begin{align*}
&& (1+u)^2 - (2-\theta)(1-x)(1+u) + (1-\theta)(1-x) &= 0 \\
\Rightarrow  && u^2 +  ((1-x)\theta + 2x) u + x&= 0.
\end{align*}
Let $f(u) = u^2 + \theta u + 2xu - x\theta u + x$.  To prove $\mu_1(\A) \ge 1 + \frac{\sqrt{|x|}}{2}$ when $x\in [-\frac{1}{4}, -\theta^2]$, we only need to verify $f(\frac{\sqrt{|x|}}{2}) \le 0$:
\begin{align*}
f(\frac{\sqrt{|x|}}{2}) = &\frac{|x|}{4} + \frac{\theta\sqrt{|x|}}{2}
- |x|\sqrt{|x|} + \frac{|x|\sqrt{|x|}\theta}{2} - |x| \\
\le &  \frac{\theta\sqrt{|x|}}{2}- \frac{3|x|}{4} -|x|\sqrt{|x|}(1-\frac{\theta}{2})\le 0
\end{align*}
The last inequality holds because $\theta \le \sqrt{|x|}$ in this case.

For $x\in [-\theta^2, 0]$, we have:
\begin{align*}
f(\frac{|x|}{2\theta}) = \frac{|x|^2}{4\theta^2} + \frac{|x|}{2} - \frac{|x|^2}{\theta}
+ \frac{|x|^2}{2} - |x|
= \frac{|x|^2}{4\theta^2} - \frac{|x|}{2} -|x|^2(\frac{1}{\theta} - \frac{1}{2})\le 0,
\end{align*}
where the last inequality is due to $\theta^2 \ge |x|$.

In summary, we have proved
\begin{equation*}
\mu_1(\A) \ge 
\begin{cases}
1 + \frac{\sqrt{|x|}}{2}, & x \in  [-\frac{1}{4}, -\theta^2]\\
1 + \frac{|x|}{2\theta}.& x \in  [-\theta^2, 0],
\end{cases}
\end{equation*}
which finishes the proof.
\end{proof}

The next lemma is a technical lemma on large powers.
\begin{lemma} \label{lem:aux_eigen_combo_inequal} 
Under the same setting as Lemma \ref{lem:aux_eigenvalues}, 
and with $x \in [-\frac{1}{4}, 0]$, denote
\begin{equation*}
(a_t, ~-b_t) = \pmat{1 & 0 } \A^t.
\end{equation*}
Then, for any $0\le\tau\le t$, we have 
\begin{equation*}
|a^{(1)}_{t-\tau}||a^{(1)}_{\tau} - b^{(1)}_{\tau}| \\
\le  [\frac{2}{\theta} + (t+1)] |a^{(1)}_{t+1} - b^{(1)}_{t+1}|.
\end{equation*}
\end{lemma}
\begin{proof}
Let $\mu_1$ and $\mu_2$ be the two eigenvalues of the matrix $\A$,
where $\mu_1 \ge \mu_2$.  Since $x\in [-\frac{1}{4}, 0]$, 
according to Lemma \ref{lem:aux_eigenvalues} and 
Lemma \ref{lem:aux_nonconvex_mu2}, we have 
$0\le \mu_2\le 1 - \frac{\theta}{2} \le 1 \le \mu_1$, 
and thus expanding both sides using Lemma \ref{lem:aux_matrix_form} yields:
% \begin{align*}
% \frac{\mu_1^{t+1-\tau} - \mu^{t+1 - \tau}_2}{\mu_1 -\mu_2}\cdot
% \frac{\mu_1^{\tau+1} - \mu_2^{\tau+1} - \mu_1\mu_2 (\mu_1^\tau -\mu_2^\tau)}{\mu_1 - \mu_2}
% \le& [\frac{2}{\theta}+ (t+1)]
% \frac{\mu_1^{t+2} - \mu_2^{t+2} - \mu_1\mu_2 (\mu_1^{t+1} -\mu_2^{t+1})}{\mu_1 - \mu_2}
% \end{align*}
\begin{align*}
\text{LHS} =& 
\left[\sum_{i=0}^{t-\tau}\mu_1^{t-\tau-i}\mu_2^i\right]
\left[(1-\mu_2)\left(\sum_{i=0}^{\tau-1}\mu_1^{\tau-i}\mu_2^i\right) + \mu_2^\tau \right] \\
=& \left[\sum_{i=0}^{t-\tau}\mu_1^{t-\tau-i}\mu_2^i\right](1-\mu_2)\left(\sum_{i=0}^{\tau-1}\mu_1^{\tau-i}\mu_2^i\right)
+ \left[\sum_{i=0}^{t-\tau}\mu_1^{t-\tau-i}\mu_2^i\right] \mu_2^\tau \\
\le& (t-\tau+1)\mu_1^{t-\tau} (1-\mu_2)\left(\sum_{i=0}^{\tau-1}\mu_1^{\tau-i}\mu_2^i\right)
 + \left[\sum_{i=0}^{t-\tau}\mu_1^{t-\tau-i}\mu_2^i\right] \\
\le& (t+1)(1-\mu_2)\left(\sum_{i=0}^{\tau-1}\mu_1^{t+1-i}\mu_2^i\right)
+\frac{2}{\theta}(1-\mu_2)\left[\sum_{i=0}^{t-\tau}\mu_1^{t+1-i}\mu_2^i\right] \\
\le& [\frac{2}{\theta}+ (t+1)]\left[(1-\mu_2)\sum_{i=0}^{t}\mu_1^{t+1-i}\mu_2^i
+ \mu_2^{t+1}\right]
 =  \text{RHS},
\end{align*}
which finishes the proof.
\end{proof}

\noindent
The following lemma gives properties of the $(1,1)$ element of large 
powers of the~\nag~matrix.
\begin{lemma}\label{lem:aux_increase_x}
Let the $2\times 2$ matrix $\A(x)$ be defined as follows and let
$x\in[-\frac{1}{4}, 0]$ and $\theta \in (0, \frac{1}{4}]$.
\begin{equation*}
\A(x) = \pmat{(2-\theta) (1 - x)&  -(1-\theta) (1 - x) \\1 & 0}.
\end{equation*}
For any fixed $t>0$, letting $g(x) = \abs{\pmat{1 & 0 }[\A(x)]^{t}\pmat{1 \\ 0 }}$, 
then we have:
\begin{enumerate}
\item $g(x)$ is a monotonically decreasing function for $x \in [-1, \theta^2/(2-\theta)^2]$.
\item For any $x \in [\theta^2/(2-\theta)^2, 1]$, we have $g(x) \le g(\theta^2/(2-\theta)^2)$.
\end{enumerate}

\end{lemma}
\begin{proof}
For $x \in [-1, \theta^2/(2-\theta)^2]$, we know that $\A(x)$ has two 
real eigenvalues $\mu_1(x)$ and $\mu_2(x)$, Without loss of generality, 
we can assume $\mu_1(x) \ge \mu_2(x)$.
By Lemma \ref{lem:aux_matrix_form}, we know:
\begin{align*}
g(x) = \abs{\pmat{1 & 0 }[\A(x)]^{t}\pmat{1 \\ 0 }}
=\sum_{i=0}^t [\mu_1(x)]^i [\mu_2(x)]^{t-i}
= [\mu_1(x)\mu_2(x)]^{\frac{t}{2}} \sum_{i=0}^t \left[\frac{\mu_1(x)}{\mu_2(x)}\right]^{\frac{t}{2} - i}.
\end{align*}
By Lemma \ref{lem:aux_eigenvalues} and Vieta's formulas, 
we know that $[\mu_1(x)\mu_2(x)]^{\frac{t}{2}} 
= [(1-\theta) (1 - x)]^{\frac{t}{2}}$ is monotonically decreasing in $x$.
On the other hand, we have that:
\begin{align*}
\frac{\mu_1(x)}{\mu_2(x)} + \frac{\mu_2(x)}{\mu_1(x)}
+ 2 = \frac{[\mu_1(x) + \mu_2(x)]^2}{\mu_1(x)\mu_2(x)}
= \frac{(2-\theta)^2(1-x)}{1-\theta}
\end{align*}
is monotonically decreasing in $x$, implying that 
$\sum_{i=0}^t \left[\frac{\mu_1(x)}{\mu_2(x)}\right]^{\frac{t}{2} - i}$ 
is monotonically decreasing in $x$. Since both terms are positive, 
this implies the product is also monotonically decreasing in $x$, 
which finishes the proof of the first part.

For $x \in [\theta^2/(2-\theta)^2, 1]$, the two eigenvalues 
$\mu_1(x)$ and $\mu_2(x)$ are conjugate, and we have:
\begin{equation*}
[\mu_1(x)\mu_2(x)]^{\frac{t}{2}} = [(1-\theta) (1 - x)]^{\frac{t}{2}} \le [\mu_1(\theta^2/(2-\theta)^2)\mu_2(\theta^2/(2-\theta)^2)]^{\frac{t}{2}}
\end{equation*}
which yields:
\begin{equation*}
\sum_{i=0}^t \left[\frac{\mu_1(x)}{\mu_2(x)}\right]^{\frac{t}{2} - i}
\le \norm{\sum_{i=0}^t \left[\frac{\mu_1(x)}{\mu_2(x)}\right]^{\frac{t}{2} - i}}
\le \sum_{i=0}^t \norm{\frac{\mu_1(x)}{\mu_2(x)}}^{\frac{t}{2} - i}
=t+1
= \sum_{i=0}^t \left[\frac{\mu_1(\theta^2/(2-\theta)^2)}{\mu_2(\theta^2/(2-\theta)^2)}\right]^{\frac{t}{2} - i},
\end{equation*}
and this finishes the proof of the second part.
%\cnote{fill in more detail.}
\end{proof}

\noindent
The following lemma gives properties of the sum of the first 
row of large powers of the~\nag~matrix.
\begin{lemma}\label{lem:aux_increase_t}
Under the same setting as Lemma \ref{lem:aux_eigenvalues}, 
and with $x \in [-\frac{1}{4}, 0]$, denote
\begin{equation*}
(a_t, ~-b_t) = \pmat{1 & 0 } \A^t.
\end{equation*}
Then we have
$$|a_{t+1} - b_{t+1}| \ge |a_t - b_t|$$
and
$$|a_t - b_t| \ge \frac{\theta}{2}\left(1 + \frac{1}{2}\min\{\frac{|x|}{\theta}, \sqrt{|x|}\}\right)^t.$$
\end{lemma}
\begin{proof}
Since $x<0$, we know that $\A$ has two distinct real eigenvalues. 
Let $\mu_1$ and $\mu_2$ be the two eigenvalues of $\A$.
For the first inequality, by Lemma \ref{lem:aux_matrix_form}, we only need to prove:
\begin{align*}
\mu_1^{t+1} - \mu_2^{t+1} - \mu_1\mu_2(\mu_1^{t} - \mu_2^{t})
\ge \mu_1^{t} - \mu_2^{t} - \mu_1\mu_2(\mu_1^{t-1} - \mu_2^{t-1}).
\end{align*}
Taking the difference of the LHS and RHS, we have:
\begin{align*}
& \mu_1^{t+1} - \mu_2^{t+1} - \mu_1\mu_2(\mu_1^{t} - \mu_2^{t}) - 
(\mu_1^{t} - \mu_2^{t}) + \mu_1\mu_2(\mu_1^{t-1} - \mu_2^{t-1}) \\
=& \mu_1^{t}(\mu_1 - \mu_1\mu_2 - 1  + \mu_2) - \mu_2^{t}(\mu_2 - \mu_1\mu_2 - 1 +\mu_1)\\
=& (\mu_1^t - \mu_2^t)(\mu_1 - 1)(1-\mu_2).
\end{align*}
According to Lemma \ref{lem:aux_eigenvalues} and Lemma \ref{lem:aux_nonconvex_mu2},
$\mu_1 \ge 1 \ge \mu_2 \ge 0$, which finishes the proof of the first claim.

For the second inequality, again by Lemma \ref{lem:aux_matrix_form}, 
since both $\mu_1$ and $\mu_2$ are positive, we have:
\begin{align*}
a_t - b_t = \sum_{i=0}^t \mu_1^i \mu_2^{t-i} - \mu_1\mu_2\sum_{i=0}^{t-1} \mu_1^{i} \mu_2^{t-1-i}
\ge (1-\mu_2)\sum_{i=0}^t \mu_1^i \mu_2^{t-i} \ge (1-\mu_2)\mu_1^t.
\end{align*}
By Lemma \ref{lem:aux_nonconvex_mu2} we have $1-\mu_2 \ge \frac{\theta}{2}$, 
By Lemma \ref{lem:aux_negcurve_mu1} we know $\mu_1 \ge 1 + \frac{1}{2}\min\{\frac{|x|}{\theta}, \sqrt{|x|}\}$.
Combining these facts finishes the proof.
\end{proof}

\end{document}